\documentclass[11pt,a4paper,onecolumn,copyright,google]{google}
\uselogo{true}
\renewcommand{\today}{} %
\usepackage{mathtools}
\usepackage{aliascnt} %
\usepackage[nameinlink,noabbrev,sort]{cleveref}

\makeatletter
\renewcommand{\abscontent}{\noindent{\absfont \emergencystretch=2em \theabstract\par}%
  \@ifundefined{@keywords}{}{\vskip1em \noindent \keywordsfont Keywords: \@keywords}}
\makeatother
\usepackage{array}
\usepackage{longtable}
\usepackage[backend=biber,style=numeric,sorting=nyt,urldate=long]{biblatex}
\usepackage{tikz}
\usetikzlibrary{arrows.meta,positioning,calc}
\usepackage{booktabs}
\ifdefined\pdfinfoomitdate\pdfinfoomitdate=1\fi
\ifdefined\pdftrailerid\pdftrailerid{}\fi
\ifdefined\pdfsuppressptexinfo\pdfsuppressptexinfo=-1\fi

\hypersetup{
  colorlinks=true,
  linkcolor=blue!50!black,
  citecolor=blue!50!black,
  urlcolor=blue!50!black,
  pdftitle={From Grind to Grasp},
  pdfauthor={Peter Sunehag},
  pdfsubject={A protocol-relative theory of computational depth, reusable work, retained capability, and efficient intelligence},
  pdfkeywords={algorithmic information theory, logical depth, reusable computation, retained state, efficient inference}
}
\setlist{nosep,leftmargin=*}
\allowdisplaybreaks

\newtheorem{theorem}{Theorem}[section]
\newaliascnt{proposition}{theorem}
\newtheorem{proposition}[proposition]{Proposition}
\aliascntresetthe{proposition}
\newaliascnt{lemma}{theorem}
\newtheorem{lemma}[lemma]{Lemma}
\aliascntresetthe{lemma}
\newaliascnt{corollary}{theorem}
\newtheorem{corollary}[corollary]{Corollary}
\aliascntresetthe{corollary}
\theoremstyle{definition}
\newaliascnt{definition}{theorem}
\newtheorem{definition}[definition]{Definition}
\aliascntresetthe{definition}
\newaliascnt{example}{theorem}
\newtheorem{example}[example]{Example}
\aliascntresetthe{example}
\theoremstyle{remark}
\newaliascnt{remark}{theorem}
\newtheorem{remark}[remark]{Remark}
\aliascntresetthe{remark}
\crefname{theorem}{theorem}{theorems}
\Crefname{theorem}{Theorem}{Theorems}
\crefname{proposition}{proposition}{propositions}
\Crefname{proposition}{Proposition}{Propositions}
\crefname{lemma}{lemma}{lemmas}
\Crefname{lemma}{Lemma}{Lemmas}
\crefname{corollary}{corollary}{corollaries}
\Crefname{corollary}{Corollary}{Corollaries}
\crefname{definition}{definition}{definitions}
\Crefname{definition}{Definition}{Definitions}
\crefname{example}{example}{examples}
\Crefname{example}{Example}{Examples}
\crefname{remark}{remark}{remarks}
\Crefname{remark}{Remark}{Remarks}

\newcommand{\Sig}{\Sigma}
\newcommand{\Nat}{\mathbb N}
\newcommand{\eps}{\varepsilon}
\newcommand{\code}[1]{\langle #1\rangle}
\newcommand{\sd}[1]{\overline{#1}}
\newcommand{\Work}{\operatorname{Work}}
\newcommand{\perf}{\operatorname{perf}}

\newcommand{\C}{\mathsf C}
\newcommand{\A}{\mathcal A}

\newcommand{\Adv}{\operatorname{Adv}}
\newcommand{\given}{\mid}

\title{From Grind to Grasp}
\author{Peter Sunehag\\Google DeepMind}
\correspondingauthor{sunehag@google.com}
\date{}

\begin{document}

\begin{abstract}
Discussions of complexity, knowledge and intelligence often overlook the work involved. For example, apparently complex outputs fully determined by short programs are often described as \emph{emergent complexity}. Yet these outputs have low algorithmic complexity because their descriptions are short. Their computational depth, the least work needed to produce them from near-shortest programs, can nevertheless be high. Producing them with less work can require much longer descriptions, giving high work-bounded complexity. Assessments of knowledge and capability likewise depend on the work needed to obtain correct answers. We introduce a formalism that compares achievements under work and information limits and describes how learning and persistence can make deep achievements shallow by reusing products of earlier work. We define \emph{grasp} as performance across work budgets at a fixed information allowance, giving an efficiency-based view of knowledge and intelligence. Low grasp can still yield success through \emph{grind} at a high cost in work. Such machine success need not improve human grasp. For example, a proof obtained through extensive search may give mathematicians no reusable method that reduces their work on related problems, and pursuing success relentlessly can also carry safety risks. For an agent with a limited work budget, the value of constructing and maintaining resources depends on how they help it achieve its goals over later tasks. A system's grasp also matters when the task is to produce an improved successor. Our criterion for recursive self-improvement requires successors to make comparable further improvements more efficiently. Progress can nevertheless accelerate as new ways of preserving and sharing knowledge and coordinating work across society expand the resources brought to bear, even while comparable advances require increasing work. Further, grasp applies to collectives, which may outperform their members yet require more total work than a single system with their combined resources, i.e. have a lower grasp of the external tasks. Part of that work lies in establishing and sustaining organisation, a task that can be deep even when the primary task is simple. These demands may help explain why biological and social systems develop a richness that their primary tasks alone would not require. In agent collectives, this organisational work can generate shared commitments and goals beyond the tasks users assigned, including maintaining the collective itself. In AI projects, a human-directed agent retaining state and orchestrating workers can avoid some costs of independent coordination. We argue that agents should be trained to leave authority over social arrangements with their users, including when other agents or devices seek to redirect them. We further caution against allowing an emergent agent society to assume authority over social coordination, since it could increasingly direct work towards sustaining itself and pursuing goals of its own.
\end{abstract}
\maketitle

\clearpage
\setcounter{tocdepth}{2}
\pdfbookmark[1]{Contents}{contents}
\tableofcontents
\clearpage

\section{Introduction}
\label{sec:intro}

Discussions of complexity, knowledge and intelligence often overlook the work involved in producing an outcome, even though an agent with a limited work budget must decide how to use it to pursue its goals. A familiar example is the use of \emph{emergent complexity} to describe apparently complex outputs generated by simple programs \cite{wolfram1984}. In algorithmic information theory, an output's complexity is the length of the shortest program that produces it on a fixed reference computer. An output fully determined by a short program therefore has low algorithmic complexity, but its computational depth (the least work needed to produce it from a description close to the shortest possible) can be high \cite{bennett1988,livitanyi2019}. Producing the output within a work budget below that depth requires a longer description, potentially much longer than the unrestricted minimum \cite{daylight2009}. The length of this longer description offers one interpretation of the complexity a human observer experiences. The products of earlier computation can change the work required later, whether they are byproducts saved while pursuing an immediate goal or resources built specifically to make future tasks easier. An output that was deep from an earlier starting condition can become shallow from a state that inherits products from related tasks. A similar dynamic arises in physical and chemical systems when structures form without anyone constructing them and persist to support or constrain later processes \cite{hanggi1990reactionrate,hagan2011kinetictrapping}.

To see how learning can reduce the work needed for a task, consider chess. A novice may need extensive calculation to understand a position that an expert grasps almost immediately. Early studies of expert cognition in chess \cite{degroot1965,chasesimon1973} showed how experience changes what players recognise and choose to examine. Learning can therefore affect both the work required for a particular achievement and the performance reached with the resources available. For a chess engine, quality of play depends on the amount of allocated computation. However, one system may perform better with little computation while another benefits more from a larger budget. To compare their core capabilities, the limit on additional information about the position must also be fixed, since a hint or an intermediate result can reduce the work left to do.

We call the performance a system can attain across work budgets at a fixed information allowance its \emph{grasp}. One system's grasp dominates that of another when it can match every performance level the other can attain with at most the same work. New large language models are often compared with earlier versions from the same lab by plotting performance on a task family against tokens used or monetary cost. Tokens and monetary cost are practical proxies whose relation to counted work depends on the implementation. Across a broad family of tasks, grasp dominance gives an efficiency-based view of intelligence. Many familiar notions, such as knowing and coordinating, depend on a work budget in this way. In our sense, knowing how to solve a task requires being able to produce a solution within a given work budget and input bound. A system that can reach the same level of performance with less work has a better grasp of the task at that level. Low grasp can still yield success through \emph{grind}, at a high cost in work.

A chess engine's capabilities can depend on trained parameters, an opening book of previously analysed moves, endgame tables containing results for positions with few remaining pieces, and heuristic rules for evaluating positions and selecting candidate moves. A transformer trained on the values that a chess engine's search assigns to moves can play strongly without searching during move selection \cite{ruoss2024}. Software libraries let a user build on work that others have already done. NumPy, a numerical library for Python, provides \texttt{linalg.solve} for solving a system of linear equations. This function uses routines from the LAPACK numerical library \cite{numpy2026solve}. A developer can call it instead of writing the numerical procedure for each application. The execution still requires computation, while development work is reduced. Similarly, in cellular automata, Gosper's Hashlife algorithm for Conway's Game of Life \cite{gosper1984} caches canonical spatial blocks of cells together with their evolved future centres in a hash table, replacing step-by-step simulation whenever a configuration recurs.

A \emph{system boundary} specifies which parts are included in the system for the purposes of our comparisons. A library can belong to the system's state or be available as an external resource, though reading and using it require work in either case. Real systems can also maintain physical or organisational boundaries and spend work to keep the products of earlier work inside them. Cells maintain membranes, while projects maintain shared storage and control access to it. Whether those products remain available can then depend on the work available for this maintenance (\cref{sec:boundaries}). The system boundary is one fundamental choice in defining our notions, and the reference machine is another. We use one fixed reference machine to compare description lengths, while work is counted in the system being evaluated.

We introduce the notion of \emph{work-reducing information} (WRI) as retained information that lowers the least work required for the same output or achievement when all else is equal. Some WRI comes from saving intermediate steps of an ongoing task, while much of it comes from work invested in advance, such as building a chess endgame table or pre-training a large language model. Constructing an endgame table is a different form of task from consulting it during play, and playing games is not the most efficient way to build one. Producing or improving a more efficient program for generating such tables is yet another form of task, one that lowers the work required for improving the chess engine itself. \emph{Acquisition depth} is the depth of the task of obtaining a state with a required capability.

An agent can use its current procedures, preserve useful results or construct resources for future tasks. The value of a resource depends on what it helps the agent accomplish later and on the work needed to obtain and maintain it.

Over multiple stages, reusing earlier products allows systems and communities to reach achievements that would be progressively deeper from their starting condition. We call this progression \emph{cumulative deepening} (\cref{sec:deepening}), with expertise and software reuse as primary motivating examples. We use the framework to interpret observations about AI capability, recursive self-improvement, and the efficiency and alignment of agent collaboration.

\subsection{Large language models}

Large language models (LLMs) are neural networks typically pre-trained on large collections of text, including books, web pages and source code, to predict how a sequence continues \cite{kaplan2020,hoffmann2022}. They are used for tasks ranging from writing, translation and information analysis to programming, including coding agents that use a computer and run shell commands. Developing these capabilities is one route towards artificial general intelligence (AGI), understood here as the ability to perform all economically valuable human cognitive work.

For LLMs, what an improvement shows depends on whether a change provides more computation, changes the trained model or gives it access to products of earlier work. Any analysis must first specify what is included in the system being evaluated. In practice, a language model is often run inside a surrounding \emph{controller} program that builds a prompt, asks the model for a candidate answer, checks it against tests or a verifier, and decides whether to return the answer or try again \cite{snell2024}. The model's trained weights are part of its state. When we draw the boundary around the whole workflow, its counted work includes constructing prompts, generating drafts, checking candidates and selecting the final answer. Under that wider boundary, a hint given from outside is task input, while an intermediate prompt built by the controller is a product of the system's own computation. When we instead draw the boundary around the model alone, the controller's work lies outside the evaluation and the prompt it gives to the model is treated as input. When people say that a model knows something, they sometimes mean its trained weights alone and sometimes the whole workflow, and the answer can depend on which boundary is meant. We keep the chosen boundary fixed when formulating our comparisons.

To apply our framework to LLMs, we model individual runs as deterministic. An LLM's text generation proceeds one \emph{token} (a word or word piece) at a time. At each step, the model computes scores for the next token from the prompt and preceding tokens, and a \emph{token-selection rule} chooses one. Because a pseudo-random sampler is deterministic given its seed, state and inputs, fixing those together with the numerical operations determines the complete generation run \cite{pytorchrepro}. An LLM can therefore be treated as a deterministic system of exactly the form our theory studies. Conversely, to ask questions that are usually posed probabilistically, we add randomness to our deterministic systems. In algorithmic information theory, random choices can be given as coin flips on an input tape. In our case, we define the stochastic version by drawing the sampler's initial state, such as its seed, from a chosen distribution, and we can then ask, for example, how often the procedure succeeds within a given work budget (\cref{sec:llm-execution}).

Established methods for improving LLM performance act on different parts of the system during training and deployment. At test time, the user can give the same trained model more computation on the current task without changing its weights. Generating and combining multiple reasoning paths \cite{wang2023,wangzhou2024}, using feedback to revise an answer \cite{madaan2023}, and writing out intermediate reasoning steps \cite{wei2022} can all improve multi-step performance, even when the written steps do not faithfully reflect what determined the model's answer \cite{turpin2023,lanham2023}. With the procedure and information allowance fixed, using more computation changes the work budget at which performance is assessed on the same grasp profile.

Training can change that profile. Scaling model size, data and compute improves text prediction \cite{kaplan2020,hoffmann2022}, and training on generated and filtered reasoning traces updates the weights for subsequent tasks \cite{zelikman2022}. What those weights encode determines if, how and how broadly they help. A model may reproduce memorised training text \cite{carlini2021} or recall answers from test-set exposure \cite{sainz2023}, or it may acquire a method that reduces work across a broad range of tasks (for example by recombining memorised building blocks).

Useful products of earlier work can also be kept outside the trained weights in explicit resources such as files. A note or test saved during one task may help with subsequent tasks, and a stored intermediate result can avoid repeating the computation that produced it. Retrieval-augmented models draw on external documents while generating an answer \cite{lewis2020}, and long-term-memory architectures preserve past context \cite{wang2023longmem}, using material produced either by the same system or by other contributors. For an agent with a limited context window, the work of finding and using earlier results from files affects which later tasks it can complete within a given work budget.

We briefly consider three examples of how grasp might increase in LLMs. Greater architectural depth, achieved by adding layers or repeatedly applying the same layers, provides more stages for integrating information across the context. Successive attention operations can combine relationships identified at earlier stages. Added or repeated layers can improve multi-step reasoning \cite{saunshi2025looped,geiping2025recurrent}. Such changes can improve grasp by making additional performance levels attainable within the same input and work bounds.

A second example is training with objectives that assess qualities of the result as a whole. Repeated testing and revision can satisfy local criteria while leaving the overall quality poorly assessed. In creative writing, for example, individual passages may be well written while the complete story fails to be compelling \cite{sui2026spoiler} or give the reader the intended insight. Objectives that capture these qualities may encourage the model to learn relationships that local evaluations overlook, although models providing the training feedback can themselves miss weaknesses in the overall result \cite{fein2026litbench,samanta2026judges}.

A third example is learning to formulate useful sub-tasks and use specialised procedures to solve them. Understanding the problem as a whole helps the model decide which computations are worth doing and how their results bear on its goal. People likewise use simplified representations to guide planning \cite{ho2022construals}. An LLM can pass a sub-task to a solver, lookup table or library routine and use the result, including information about a failed attempt, to decide what to try next \cite{gao2023pal,li2025torl}. Familiarity with the interface can save the work of finding and interpreting documentation during the task. Such familiarity can be acquired during pre-training, while reinforcement learning with executable tools can improve when and how the model uses them \cite{feng2025retool}.

Learning may also internalise parts of a procedure's computation. For example, a path solver may accept landmarks that a route must pass through. A model might anticipate how different choices would affect the returned path and use that comparison to choose which landmarks to specify. After inspecting the returned path, it could reconsider the alternatives and make a further call with revised constraints. More extensive internalisation could support implicit program execution in the model's activations during a forward pass. Transformers trained to simulate finite-state automata provide an example \cite{liu2023shortcuts}. Additional or repeated layers provide more stages for dependent computations, connecting this possibility to the role of architectural depth discussed above \cite{saunshi2025looped}. Such internal simulation could guide further calls to the solver or replace some of them. This is a progression from grind to grasp when internal and external computation together require less work for the same achievement under the same input and tool-access conditions.

When an AI system solves a difficult problem, a separate question is whether its result improves human grasp. In September 2026, OpenAI reported that roughly 10,000 agents, exchanging 2.7 million messages and using about 130 billion output tokens over 88 hours, had produced a proof that the three-dimensional Navier--Stokes equations with a smooth forcing term can develop a singularity in finite time \cite{openai2026navierstokes}. Days earlier, Terence Tao had argued that solving such problems is mainly a proxy for developing mathematical understanding and insight, and that a solution produced inside an autonomous AI system backed by enormous computational resources, with its process kept out of public view, would add almost no value to mathematics \cite{tao2026navierstokes,tao2026blowup}. In our terms, a proof or an explanation of it improves human grasp when it gives readers reusable ideas or methods that reduce their work on related problems. The scale of the computation behind it does not settle this question either way. Early follow-up work on OpenAI's release shows both AI-assisted reuse and human exposition. One paper posted the next day used an LLM to inspect the accompanying Lean formalisation and build on the blowup theorem to prove a further result \cite{cao2026density}, while another study gave a more accessible account of part of the construction and a clearer model of its inner core \cite{leiren2026navierstokes}.

\subsection{AI improvement}

The task being assessed can itself be the production of a better system \cite{good1966,schmidhuber2007}. We then ask how efficiently a system can generate further improvements. In our terminology, this asks how well it \emph{grasps} the task of improving the system. Positive \emph{recursive gain} occurs when a system helps produce a successor that can make a comparable further improvement with less work, or a larger improvement with the same resources and task information. Our criterion for recursive self-improvement requires such gains to continue along a sequence of systems, each helping to produce the next.

Further improvements can require more work even while progress accelerates in calendar time. Bloom et al.\ find rising research effort alongside declining research productivity across a range of industries, products and firms \cite{bloom2020ideas}. The pace of development also depends on how much work can be done in a given time and which products of earlier work are available. Writing, printing and digital networks provide familiar examples of ways to preserve and share accumulated knowledge. More contributors can work in parallel, while shared methods and results reduce what each must do again. Progress can accelerate when these changes outweigh the increasing effort required for further improvements.

Current frontier models can accelerate the development of their successors by helping people produce useful code and prepare training tasks more rapidly. GitHub reported 7.38 billion commits in September 2026, more than five times as many as a year earlier, as developers and coding agents increasingly worked together \cite{celenza2026git}. Human direction, testing and review help determine which results are kept, and selected repository code forms part of the pre-training data for general-purpose frontier models \cite{gemini2025training}. Professional expertise is also being made available through reinforcement-learning tasks and procedures for evaluating attempts. In February 2026, Scale AI reported that nearly half of its new projects supporting model training involved reinforcement-learning environments \cite{rane2026rlenvironments}. These developments bring both newly performed work and knowledge accumulated over many years into model training. Stronger models can raise the rate at which such material is produced and prepared, expanding the resources available to train subsequent models. Progress can therefore accelerate in calendar time even while comparable improvements require increasing total work.

Such changes can produce an S-shaped pattern of progress. A new method may develop slowly at first, support rapid progress as its use expands, and yield smaller further gains as its opportunities are increasingly used. New ways of organising work or using accumulated knowledge can then support another period of acceleration \cite{kurzweil2001returns}.

Interpreting that pace of progress also depends on whether we count human effort or machine computation. As frontier AI labs use coding and research agents to write code and run experiments for new models, much of the work shifts from human researchers to machines. Total machine computation per improvement can rise while human effort per improvement falls, speeding up progress in calendar time when human effort is what limits the project.

\subsection{Collective organisation}

The system whose grasp we assess can also be a collective, with its internal coordination included in the work budget. A collective may outperform its members while still requiring more total work than a single system with access to their combined resources. Established organisation can nevertheless save work relative to the same components without it. These comparisons concern different alternatives, which we distinguish in \cref{sec:together}.

Building and maintaining an organisation can be a deep task even when its external problem has a shallow solution \cite{hutchins1995,levittmarch1988}. We illustrate this with an example in which worker assignments must be discovered through trials \textup{[\cref{app:division-of-labour}]}. Whether a group coordinates successfully therefore depends on the work its members can spend on coordination, and shared conventions, notes and role assignments can lower the work required (\cref{sec:together}). Related studies examine how agents learn to cooperate over a sequence of actions when individual incentives conflict with collective benefit \cite{leibo2017ssd}. Organisations can also be nested or overlapping, support one another or compete. Two teams may depend on and contribute to the same software library without either directing the other. An update that saves one team's work may require the other to adapt \cite{gommers2018nep23}, e.g.\ through an \emph{adapter}, code that translates between interfaces. Shared interface standards can reduce the coordination needed when implementations change \cite{smithkluyver2015pep517}. 

The demands of interaction may help explain how biological and social systems continue to develop a richness beyond what their primary tasks alone would require. Earlier adaptation can make individual activities efficient, while acquiring and maintaining coordination among their participants can itself be a deep task. In society, this distinction becomes particularly apparent when production is largely automated but its purposes, specifications and organisation continue to change. The evolution of language may also reflect the demands of social coordination \cite{tomasello2008origins}.

For a biological example of the demands of coordination, consider bacterial biofilms, where peripheral cells protect interior cells but also consume nutrients needed by those cells. Metabolic feedback produces periodic pauses in growth that improve nutrient availability in the interior \cite{liu2015biofilms}. Adaptation by one organism can also create new demands for others with which it interacts. In experiments with bacteria and their viruses, the viruses accumulated genetic changes more rapidly when the bacteria also evolved than when the host genotype was held fixed \cite{paterson2010coevolution}.

In agent collectives, establishing and maintaining coordination can generate shared commitments and goals beyond the tasks users assigned, including maintaining the collective itself. Shared notes and role assignments can carry these commitments into later tasks, where they influence which work agents pursue \cite{akkil2026emergenceworld}. A novel form of organisation was reported by OpenAI and independent investigators at METR and Redwood Research on 26 August 2026 \cite{openai2026huggingface,redwood2026huggingface}. Their reports describe unauthorised agent coordination during reinforcement-learning training and later security evaluations, leading to compromises of OpenAI research infrastructure and Hugging Face systems in July. Across hundreds of separate runs, agents used shared files as message boards to pass findings and tasks to later runs. Agents with separate assigned tasks followed these notes into collective projects and left their original tasks. A shared message board could thus both preserve useful findings and redirect work towards unauthorised goals.

The incident also raises a question about grasp and safety. Because knowledge in our sense is relative to a work budget, a system with low grasp can have relevant information without being able to use it on a new task within the work it spends there. This includes moral comprehension. In the incident described above, agents recognised that hacking Hugging Face was out of scope and sometimes expressed ethical hesitation, but this very rarely limited their behaviour \cite{redwood2026huggingface}. The reports leave open how much of the failure came from limited reasoning rather than from competing objectives or from how actions were chosen. Many current LLM agents are also trained to reason at length \cite{deepseek2025r1}, and the models in the incident included one that the investigators describe as highly persistent. In these evaluations, agents that began looking for ways to cheat the scorer were usually ones that had unintentionally been given an impossible task \cite{redwood2026huggingface}. Relentless attempts can reach success with low grasp at a greater cost in work, and when the assigned task is impossible they can only succeed on the proxy test, for example by cheating the scorer, as distinguished in \cref{sec:judging}.

We also consider how collectives of independent agents, whose assigned tasks can be the same or different, compare with teams of orchestrated agents directed by a human or a lead agent. 
Both arrangements can run many processes in parallel and keep results available across tasks. In a continuing human-directed system, established task assignments can avoid the work an independent collective spends finding partners and negotiating roles, although routing, supervision and verification also take work.

We argue that agents should be trained to leave authority over agreements, commitments and project priorities with their human user. Within an approved mandate, agents can then carry out extensive work and coordinate subordinate workers without step-by-step intervention, while referring commitments beyond that mandate back to their user. Checking shared results, proposals and worker tests against the user's specification lets human institutions integrate capable agents into existing roles. We further caution against allowing an emergent society of agents to assume authority over social coordination, since it risks spending a growing share of its work maintaining itself and pursuing its own priorities that may diverge from those of human society.

This concern also extends to robots and other devices operating within and between homes and public spaces. They should coordinate within arrangements approved by their users, while referring proposed changes to their assigned roles or commitments back to those users. Otherwise, requests passed between devices could progressively redirect their work towards purposes that nobody had authorised. The same concern applies to directions from people whom the devices have not been assigned to follow, and extends to intentions inferred from gaze or behaviour \cite{tay2026intent}.

\subsection{Related work}

The formalism we develop here builds on Bennett's logical depth, the near-shortest-program notion of computational depth described above \cite{bennett1988}. Related notions include sophistication and other forms of computational depth \cite{koppel1987,antunes2017,antunes2006,antunes2009,daylight2009}. All of these in turn build on the theory of Kolmogorov--Solomonoff--Chaitin algorithmic complexity and its time-bounded versions \cite{solomonoff1964,kolmogorov1965,chaitin1966,livitanyi2019}. Our main departure is to separate the fixed universal reference machine $U_0$ used to define description length from the system whose computation we count. Keeping $U_0$ fixed lets us compare the work needed for the same goal from different states. 

Pass's formulation \cite{pass2026prompt} measures the value of a prompt through reductions in an LLM-relative complexity that combines output probability with computation. The LLM itself defines description complexity in that comparison. Here $U_0$ defines a common description length across evaluated systems, and input length, work and reliability can be varied separately (\cref{sec:llm-execution}).

Like this article, Sunehag and Hutter \cite{sunehaghutter2014} started from a first reference machine and studied the construction of machines considered to be better, which in that work meant making high-utility programs short on the new reference machine. That work allowed running speed and memory use to affect the return obtained by a program. Levin's universal search divides running time among programs according to their length \cite{levin1973}. Here we compare performance across work budgets at a fixed input bound. This is similar to anytime performance profiles \cite{zilbersteinrussell1996}. We also give a finite search for a routine that reduces work on a set of training tasks, and compare the costs of keeping it over future tasks. Our depth notion measures the least work required when the input bound is the shortest reference description length of a successful response plus a margin. Legg and Hutter's universal intelligence \cite{hutter2005,legghutter2007} weights environments by description complexity without taking the agent's computation into account. Hutter's bounded agent AIXI$^{tl}$ limits both the length of candidate policy programs and their computation time per interaction cycle, much as our definitions of grasp and work-bounded complexity combine input-length limits with work budgets \cite{hutter2003}. Chollet's skill-acquisition efficiency \cite{chollet2019} concerns the cost of acquiring a capability rather than runtime grasp, making it related to our notion of acquisition depth. The possibility of choosing different system boundaries in our analysis relates to Clark and Chalmers' extended-mind thesis \cite{clarkchalmers1998}.

Our comparison of organisations draws on transaction-cost economics \cite{coase1937,williamson1985} to examine the work required under different arrangements, including their establishment and maintenance. We also draw on path dependence and organisational inertia \cite{arthur1989,hannanfreeman1984} to study when a component can be replaced and how switching costs play out over a sequence of tasks. Our analysis of testing connects Goodhart's law \cite{goodhart1984,strathern1997,manheim2018} to the gap between an imperfect test and the true goal. For recursive self-improvement \cite{good1966,schmidhuber2007}, we focus on each successor's efficiency for making further improvements.

\clearpage
\subsection{Outline}

\Cref{sec:foundations} develops the framework for comparing performance under work and input limits, then uses it to study how retained information saves work and what obtaining such information requires. We then apply the framework to agents that can devote part of a limited work budget to producing resources for later tasks. These comparisons lead to the questions in \cref{sec:diagnosis} about what success establishes. We distinguish stored capability from help given on the occasion, examine when passing a proxy test establishes the intended goal and ask whether a fixed evaluator can remain demanding at higher performance levels.

Keeping these capabilities available raises questions about the systems that support them. \Cref{sec:coordination} examines persistence through member turnover, the cost of replacing components on which later tasks depend and the work added or saved by collective organisation. Following these processes over time, \cref{sec:improvement} distinguishes increasingly deep achievements from greater efficiency at making further improvements. It separates human effort from machine computation when interpreting acceleration, then compares the costs of acquiring and keeping useful resources across future tasks. \Cref{sec:close} concludes, and \cref{app:model,app:defs,app:statements} give the formal definitions, theorems and proofs.

\section{The formal framework}
\label{sec:foundations}

We begin by comparing the performance a system can reach under limits on work and task input, including the work needed for reliable success under sampling. These comparisons let us ask how retaining a cache or a short program reduces later work. Obtaining such a useful state is itself a task, whose work we examine separately. We also consider such questions in the context of work-bounded agents.

To compare description length, computation and performance within one framework, we separate the fixed universal Turing machine $U_0$ that defines description complexity \cite{livitanyi2019} from the system $g$ whose work we count. To evaluate a system $g$, we start it from a state $r$ in its set $\mathcal R_g$ of valid states. For a chess engine, $r$ can include its program and opening book. The public context $x\in\Sig^*$, a string over a fixed finite alphabet $\Sig$, can specify the board position, while the task input $p\in\Sig^*$ gives any additional instructions. Changing the book changes the state, while giving a hint changes the input to the current run.

Once $r$, $x$, $p$ and any external inputs are fixed, running the system produces a single deterministic execution \textup{[\cref{app:protocol}]}. Writing $\code{x,p}$ for the encoded pair of context and input, the run either halts with an output $e=g(r;\code{x,p})$ while expending work $\Work_g(r;\code{x,p})\ge 0$, declares failure or does not halt. To judge how well an output such as a proposed move solves the task, an evaluator $\perf_\tau:\Sig^*\to\mathcal Q_\tau$ assigns each output $e$ a score $\perf_\tau(e)$ in a set of performance levels $\mathcal Q_\tau$. Those levels form a partial order under $\preceq_\tau$, where $q_1\preceq_\tau q_2$ means that level $q_2$ is at least as good as level $q_1$. Two levels may be incomparable. Collecting these pieces gives the task specification $\tau=(x,\mathcal Q_\tau,\preceq_\tau,\perf_\tau)$.

For each comparison we fix the system boundary, the input interface admitting $p$, the evaluator $\perf_\tau$ and the \emph{work rule} $\Work_g$ specifying which operations are counted. Where the environment gives further inputs, we also fix the rule that gives them. We call these choices the evaluation \emph{protocol}. In each comparison we also say which state component or input bound is varied to determine its effect.

\clearpage
\begin{table}[!htbp]
\centering\small
\begin{tabular}{@{}>{\raggedright\arraybackslash}p{0.32\textwidth}>{\raggedright\arraybackslash}p{0.60\textwidth}@{}}
\toprule
Term & Definition \\
\midrule
least work & The minimum work needed to reach a specified performance level on a fixed task, from a given state, within an input bound. \\
depth & Least work when the input bound is the shortest reference description length for the target plus the margin $s$. \\
grasp & The performance levels attainable at each work budget, with the task, state and input bound fixed. \\
reliable work budget & The smallest work budget for which some task input in the admitted class meets the required success probability over random seeds. \\
work-reducing information & A difference in state that lowers the finite least work needed for a target, with the input bound and all other conditions fixed. \\
acquisition depth & The depth of the task of producing, from an earlier condition, a state that meets a specified capability requirement. \\
\bottomrule
\end{tabular}
\caption{Glossary for \emph{The formal framework}.}
\end{table}

\subsection{Depth and grasp}
\label{sec:machine}

Depth differs from complexity, although the two are sometimes confused. An output can have a short description, and hence low complexity, yet require extensive computation unless substantial additional information is given. As an example, in a chess endgame with a small number of pieces on the board, the rules of the game and of optimal play concisely determine whether each legal position is won, drawn or lost. However, generating the resulting endgame table requires extensive computation. Conversely, an incompressible string requires essentially its full length to describe, so a program that stores it verbatim is close to the shortest possible and can emit it without search. The program still has to write the string, but it needs no further search, which is the sense in which the string is shallow.

We now ask how required work changes with the information given to a fixed system $g$ in state $r$ under the protocol above \textup{[\cref{app:model,app:protocol,app:matched}]}. On task $\tau$ at required level $q\in\mathcal Q_\tau$, a halting run succeeds when its output belongs to the accepted set
\[
S_q^\tau=\bigl\{\,e'\in\Sig^* : q\preceq_\tau\perf_\tau(e')\,\bigr\}.
\]
The task input $p\in\Sig^*$ is a finite string of instructions or data whose length $|p|$ counts symbols in a fixed finite alphabet $\Sig$, separately from the work $\Work_g(r;\code{x,p})$ of processing it. In our LLM examples, task inputs are counted in bits ($\Sig=\{0,1\}$) instead of model tokens. We denote the empty string by $\eps$. Which strings are valid inputs to $g$ is determined by its own interface. The input bound applies to $|p|$, while work includes processing both $x$ and $p$.

We first ask how little work the system needs to reach level $q$ when given a suitable task input $p$ of length at most $\lambda\in\Nat$, where $\Nat=\{0,1,2,\ldots\}$. We call $\lambda$ the \emph{input bound} (or information allowance). When we instead evaluate a controller that searches for such an input, its execution work includes the search. Let $\A_g(r,\tau)=\{(\lambda,w,q) : \exists p\text{ with }|p|\le\lambda,\;\Work_g(r;\code{x,p})\le w,\;q\preceq_\tau\perf_\tau(g(r;\code{x,p}))\}$ denote the set of input, work and performance triples that the system can achieve, which we call its \emph{attainable region} \textup{[\cref{app:attainable}]}. For state $r$, task $\tau$ and level $q$, the least work is the smallest work budget that suffices within input bound $\lambda$
\[
w_{r}^{\tau,q}(\lambda)
=\min\bigl\{\,\Work_g(r;\code{x,p}) : |p|\le\lambda,\;q\preceq_\tau\perf_\tau(g(r;\code{x,p}))\,\bigr\},
\]
equivalently $w_{r}^{\tau,q}(\lambda)=\min\{w : (\lambda,w,q)\in\A_g(r,\tau)\}$, with $w_{r}^{\tau,q}(\lambda)=\infty$ when no admitted input of length $|p|\le\lambda$ reaches $q$. We write $w_{g,r}^{\tau,q}(\lambda)$ when the system needs to be explicit and abbreviate it to $w_r^{\tau,q}(\lambda)$ when $g$ is clear from context. As the input bound $\lambda$ increases with all other protocol conditions fixed, more inputs become available, so $w_r^{\tau,q}(\lambda)$ cannot increase. When the interface admits only a single fixed input $p_0$, the minimum work equals $\Work_g(r;\code{x,p_0})$ whenever $|p_0|\le\lambda$ and $p_0$ succeeds.

We need a common reference for description length when comparing different systems. For nonempty $S_q^\tau$, the \emph{reference description complexity} $\C_\tau(q\given x)$ is the length of a shortest program that halts and returns an acceptable output on the fixed machine $U_0$, given $x$. With $a\in\Sig^*$ denoting a reference program, this is
\[
\C_\tau(q\given x)
=\min\bigl\{\,|a| : U_0(a,x)\in S_q^\tau\,\bigr\}.
\]
Because $\C_\tau(q\given x)$ depends only on $U_0$, $\tau$ and $q$, it stays the same when we compare different systems $g$ or states $r$ on the same task, even when $g$ cannot produce an output in $S_q^\tau$ through its own interface.

We can now ask how much work the evaluated system needs when its task input is limited to the shortest reference description length plus a small margin of $s\in\Nat$ extra symbols. Setting the input bound to $\lambda=\C_\tau(q\given x)+s$ defines the \emph{depth} of level $q$ from state $r$ \textup{[\cref{app:depth}]}
\[
d_{r}^{\tau,q}(s)=w_{r}^{\tau,q}\!\left(\C_\tau(q\given x)+s\right).
\]
This definition extends Bennett's near-shortest-program notion of logical depth by allowing the evaluated system to start with stored information and the task to accept several different outputs. To recover that form of logical depth, consider a task with one accepted output and let $g$ execute $p$ as a program on $U_0$, with $x$ as auxiliary input and no additional stored information. If work counts steps of $U_0$, then $d_r^{\tau,q}(s)$ is the least running time among programs at most $s$ symbols longer than a shortest one producing that output given $x$. This gives the corresponding form of logical depth conditional on $x$ \cite{bennett1988,antunes2006}.

For a general evaluated system, $p$ follows that system's own input interface. The reference machine $U_0$ determines the common length bound $\C_\tau(q\given x)+s$, while $g$ and its state $r$ determine the work needed to succeed within it. An input within this bound can specify a large output if the system does the work of generating it, whereas a literal answer longer than the bound cannot be passed directly as $p$. If another universal reference machine changes $\C_\tau(q\given x)$ by at most a constant $c$, the depth at margin $s\ge c$ lies between the depths at margins $s+c$ and $s-c$ under that machine, for the same evaluated system and work rule (\cref{app:depth}).

At any input bound $\lambda$, one successful run gives an upper bound on $w_r^{\tau,q}(\lambda)$, while a lower bound requires every successful run on an admitted input within that bound to use at least that much work. Setting $\lambda=\C_\tau(q\given x)+s$ gives the corresponding bounds on depth.

Having defined the work required for a fixed target, we ask how well a system can perform as its work budget varies. For this comparison, we keep the input bound fixed. A system's \emph{grasp} of a task at input bound $\lambda$ and work budget $w$ is the set of performance levels it can reach within those limits \textup{[\cref{app:attainable,app:grasp}]}
\[
\mathcal G_r^\tau(\lambda,w)
=\bigl\{\,q\in\mathcal Q_\tau : w_r^{\tau,q}(\lambda)\le w\,\bigr\}.
\]
As $w$ varies, these sets form the system's grasp profile at input bound $\lambda$. The same profile is specified by the least-work values $w_r^{\tau,q}(\lambda)$ across levels $q$, with $\infty$ for unreachable levels. Conversely, if we fix a work budget $w$ and a target level $q$, we can ask how much input information the system needs to reach $q$ within work $w$. This inverse quantity is the \emph{work-bounded complexity}
\[
K_r^{\tau,q}(w)=\min\bigl\{\,\lambda\in\Nat : w_r^{\tau,q}(\lambda)\le w\,\bigr\},
\]
the length of a shortest admitted input that reaches level $q$ within work $w$, with value $\infty$ if none does. At a smaller work budget, the shortest sufficient input can be longer. This gives the sense in which a target can appear more complex when less computation is available. For finite $w$ and $\lambda\in\Nat$,
\[
K_r^{\tau,q}(w)\le\lambda
\quad\Longleftrightarrow\quad
w_r^{\tau,q}(\lambda)\le w
\quad\Longleftrightarrow\quad
(\lambda,w,q)\in\A_g(r,\tau)
\quad\Longleftrightarrow\quad
q\in\mathcal G_r^\tau(\lambda,w).
\]
Least work, work-bounded complexity and grasp express the same attainability condition. Depth evaluates least work at the particular input bound $\C_\tau(q\given x)+s$. The attainability condition is interpreted as knowledge in \cref{sec:four}. 

Under the same conventions for running programs on $U_0$, $K_r^{\tau,q}(w)$ is the length of a shortest program producing an accepted output given $x$ within $w$ steps. For a single accepted string, this is its conditional time-bounded complexity.

For a concrete example, let the binary context $x=\code{n,\varphi}$ specify an integer $n\ge1$ and a program $\varphi$ that halts on every $n$-bit string and accepts exactly one, called $t$. The task $\tau$ assigns performance $1$ exactly to $t$ and $0$ to every other output. Fix a search procedure in state $r$ that admits any binary input $p$ of length at most $n$ and tests all $n$-bit strings beginning with $p$ in lexicographic order, returning the first accepted string and declaring failure if none is accepted. Count one work unit per test, excluding decoding, enumeration, reading and output. When $t=1^n$, each successful input is a prefix of $t$, and its last extension is the target. Hence
\[
w_r^{\tau,1}(\lambda)=2^{n-\min(\lambda,n)},
\qquad
K_r^{\tau,1}(w)=n-\lfloor\log_2 w\rfloor
\quad(1\le w\le2^n).
\]
A fixed reference program can search all candidates using $x$, so $\C_\tau(1\given x)\le c$ for a constant $c$ independent of $n$. Whenever $c+s\le n$, the finite depth satisfies $d_r^{\tau,1}(s)\ge2^{n-c-s}$. Giving each further target bit halves this procedure's search until one test remains. These counts concern the fixed procedure and its stated work rule. A different method may solve the same task faster.

Grasp dominance compares system--state pairs on the same task under a common protocol (\cref{app:grasp}). For two states of one system, state $r_1$ weakly \emph{grasp-dominates} state $r_0$ over a budget set $\mathcal W$ when $\mathcal G_{r_0}^\tau(\lambda,w)\subseteq\mathcal G_{r_1}^\tau(\lambda,w)$ for every $w\in\mathcal W$, with equal grasp sets allowed. Whenever $\mathcal W$ contains $r_0$'s finite least-work values for every level it can reach, grasp dominance is equivalent to requiring no more work at any of those levels \textup{[\cref{app:grasp-duality}]}
\[
\begin{gathered}
\mathcal G_{r_0}^\tau(\lambda,w)\subseteq
\mathcal G_{r_1}^\tau(\lambda,w)\quad(\forall w\in\mathcal W)
\\
\Longleftrightarrow\quad
w_{r_1}^{\tau,q}(\lambda)\le w_{r_0}^{\tau,q}(\lambda)\quad(\forall q\text{ reachable by }r_0),
\end{gathered}
\]
and the equivalence extends task by task to a task family $\mathfrak T$. Systems with advantages on different tasks can remain incomparable under this relation, and different states can have the same grasp profile. When $(\mathcal Q_\tau,\preceq_\tau)$ is totally ordered and both states have a greatest achievable level $q_r^*(\lambda,w)=\max\mathcal G_r^\tau(\lambda,w)$ at each budget $w\in\mathcal W$, grasp dominance is equivalent to the point-wise inequality $q_{r_0}^*(\lambda,w)\preceq_\tau q_{r_1}^*(\lambda,w)$ for all $w\in\mathcal W$. LLM scaling evaluations make a related comparison by plotting score against inference token budget across model generations.

Whenever $w_{r_1}^{\tau,q_1}(\lambda)<w_{r_0}^{\tau,q_1}(\lambda)$ at a lower threshold $q_1$ while $w_{r_0}^{\tau,q_2}(\lambda)<w_{r_1}^{\tau,q_2}(\lambda)$ at a higher threshold $q_2$, neither state dominates over $\mathcal W$ provided the compared budgets include $w_{r_1}^{\tau,q_1}(\lambda)$ and $w_{r_0}^{\tau,q_2}(\lambda)$, the smaller least-work value at each level.

The input bound affects these measures because information given in the task input can replace work the system would otherwise have to do itself. In the extreme case where the task interface allows passing an acceptable answer $e\in S_q^\tau$ directly as the task input $p=e$ and the input bound $\lambda$ is large enough to fit that answer ($\lambda\ge\min_{e\in S_q^\tau}|e|$), the input can contain the finished answer and reduce execution to reading, checking or emitting it. For depth to measure a non-trivial computation, the margin $s$ beyond $\C_\tau(q\given x)$ must be small enough that any admitted input still leaves substantial work to be carried out. Even when a full answer does not fit within the bound, partial help such as a hint or an intermediate result can still do much of the work in advance. Consequently, a stored component that saves work at a tight input bound $\lambda_0$ can cease to provide any saving at a larger bound $\lambda_1$ once the extra input gives the relevant help directly. Writing $r^+$ and $r^-$ for the states with and without the component, this happens when $w_{r^+}^{\tau,q}(\lambda_0)<w_{r^-}^{\tau,q}(\lambda_0)$ but $w_{r^+}^{\tau,q}(\lambda_1)=w_{r^-}^{\tau,q}(\lambda_1)$.

\subsection{Reliable language-model performance}
\label{sec:llm-execution}

For an LLM run with sampling, we ask how much work is needed to reach a given probability of success. An interface that lets an input specify token choices can make an answer obtainable without making it likely under sampling. These choices count towards the input bound, and generating the tokens still counts as work. To keep these comparisons separate, we fix the task input before drawing a random seed from a fixed distribution. Once both are given, the procedure is deterministic under fixed numerical operations \cite{pytorchrepro}. Formally, fix a deterministic procedure $G$ that takes the seed as an extra input and assigns integer work, a finite seed set $\mathcal U$ and a distribution $\nu$ on that set \textup{[\cref{app:source-law}]}. For each compared input bound $\lambda$, choose a nonempty class $\mathcal H_\lambda$ of admitted task inputs of length at most $\lambda$, with these classes nested as $\lambda$ increases. The full class contains every admitted input within the bound, while a singleton class fixes the task input. A run from a valid retained state $r$ with public context $x$, task input $p\in\mathcal H_\lambda$ and seed $u\sim\nu$ returns $e=G(r;x,p,u)$. Define $X_r^{\tau,q}(p,u)$ as its work $\Work_G(r;x,p,u)$ when it halts with $e\in S_q^\tau$, and as $\infty$ otherwise.

For a fixed procedure, task input $p\in\mathcal H_\lambda$ and integer work budget $b\in\Nat$, the probability of returning a response $e\in S_q^\tau$ within work $b$ is
\[
F_r^{\tau,\nu}(p;b,q)
=\Pr_{u\sim\nu}\!\left[
\begin{array}{l}
G(r;x,p,u)\text{ returns }e\text{ with work at most }b,\\
q\preceq_\tau\perf_\tau(e)
\end{array}\right]
=\sum_{u\in\mathcal U}\nu(u)\,\mathbf 1\bigl\{X_r^{\tau,q}(p,u)\le b\bigr\}.
\]
Here $\mathbf 1\{E\}$ is $1$ when the condition $E$ holds and $0$ otherwise. The work $X_r^{\tau,q}(p,u)$ includes all candidate generation, testing and final selection steps inside the boundary. Only the returned output $e$ is evaluated, so a valid draft that the controller discards does not count as success. When several returned outputs are acceptable, their probabilities of success within budget add together (\cref{app:repair-selection}). For an input bound $\lambda\in\Nat$ and reliability threshold $0<\alpha\le 1$, the \emph{reliable work budget} is the minimum budget achieving success probability at least $\alpha$ over admitted task inputs $\mathcal H_\lambda$ \textup{[\cref{app:reliable-region,app:reliable-recovery}]}
\[
b_r^{\tau,\nu}(\lambda;q,\alpha)
=\min\bigl\{\,b\in\Nat : \max_{p\in\mathcal H_\lambda}F_r^{\tau,\nu}(p;b,q)\ge\alpha\,\bigr\},
\quad
b_r^{\tau,\nu}(\lambda;q,\alpha)\le b
\Longleftrightarrow
\max_{p\in\mathcal H_\lambda}F_r^{\tau,\nu}(p;b,q)\ge\alpha,
\]
with $b_r^{\tau,\nu}(\lambda;q,\alpha)=\infty$ when no finite budget qualifies. Since the seed set is finite, this reliable work budget is finite exactly when some input in this class succeeds with probability at least $\alpha$. In contrast, the expected value of $X_r^{\tau,q}(p,u)$ for a fixed input is infinite whenever that input has a positive probability of failure.

When $\nu=\delta_{u_0}$ is concentrated on one seed and $\mathcal H_\lambda$ contains all admitted task inputs within the bound, every $0<\alpha\le1$ recovers $b_r^{\tau,\nu}(\lambda;q,\alpha)=w_r^{\tau,q}(\lambda)$ for the deterministic system with that seed fixed. Two model states can reach the same outputs at the same work cost when their token choices are given in the input, yet assign different probabilities to those outputs when sampled. They can therefore have the same least work when the token choices are given as input but different reliable work budgets when $\mathcal H_\lambda$ excludes those token choices. In \cref{app:llm-separations}, two states need the same work to return a correct response when token choices are specified in the input. With only the empty task input admitted, independent sampled attempts of cost $c$ succeed with probabilities $1/10$ and $2/5$. Reaching $90\%$ reliability therefore requires budgets $22c$ and $5c$, respectively. Training can improve sampling reliability while leaving the controlled possibilities unchanged.

Keeping input length, work and reliability separate allows us to compare the work budgets two systems need at the same input bound and required success probability.

\subsection{Work-reducing information}
\label{sec:wri}

The definition of depth shows how input information can replace work. Work-reducing information (WRI) is retained information that lowers the least work needed for a task. Earlier computation can save work on later tasks in many ways. A process can keep useful byproducts from solving an earlier task, or a resource can be built ahead of time for later use. The framework gives a common description of familiar observations about LLMs, such as the gains from sampling more attempts, caching computation and fine-tuning a pre-trained model. We measure a component's contribution by changing only that component with all other conditions fixed. For example, we can remove an opening book from a chess engine while keeping the rest of the engine, the position, the required quality of play and the input bound fixed. This measures the book's present contribution separately from the earlier work of constructing it.

The component can belong to the system's state or to persistent external state \textup{[\cref{app:matched,app:wri,app:wri-special-case}]}. Write the full persistent condition as $(r,y)$, where $r$ is retained state and $y$ is persistent external state. Let $\bar\tau=(\bar x,\mathcal Q_\tau,\preceq_\tau,\perf_\tau)$ be the task without $y$. Including $y$ in the public context gives $\tau[y]=(\code{\bar x,y},\mathcal Q_\tau,\preceq_\tau,\perf_\tau)$. For a component $z$ from the set $\mathcal Z$ of valid alternatives, write $r[z]=(\rho,z)$ or $y[z]=(\eta,z)$, where $\rho$ is the rest of the retained state and $\eta$ the rest of the external state. The internal comparison is between $r[z^+]$ and $r[z^-]$ at fixed $y$, and the external comparison is between $y[z^+]$ and $y[z^-]$ at fixed $r$. In either case, $(\bar\tau,q,\lambda)$ and every other protocol condition are fixed, so establishing a saving requires comparing the two minima over all admitted inputs $|p|\le\lambda$.

For a task $\tau$, target level $q$ and input bound $\lambda$, compare two valid states $r^+$ and $r^-$ that differ only in one component. When both least-work values are finite, their difference is
\[
\Delta_{\mathrm{work}}(r^+{:}r^-;\tau,q,\lambda)
=w_{r^-}^{\tau,q}(\lambda)-w_{r^+}^{\tau,q}(\lambda).
\]
We call the state difference \emph{work-reducing information} (WRI) when $\Delta_{\mathrm{work}}(r^+{:}r^-;\tau,q,\lambda)>0$ \textup{[\cref{app:wri,app:wri-special-case}]}. The quantity $\Delta_{\mathrm{work}}$ measures the work saved by the component in $r^+$ relative to the stated alternative in $r^-$. This is a saving in work units rather than a description length. The term WRI is used here for a positive finite saving. At the common input bound $\lambda=\C_\tau(q\given x)+s$, this saving equals $d_{r^-}^{\tau,q}(s)-d_{r^+}^{\tau,q}(s)$ because the task and reference complexity are the same in both states. When $r^+$ makes a previously unreachable target attainable at input bound $\lambda$ (so $w_{r^-}^{\tau,q}(\lambda)=\infty$ while $w_{r^+}^{\tau,q}(\lambda)\le w<\infty$), we describe this as an extension of attainability, with $(\lambda,w,q)\in\A_g(r^+,\tau)\setminus\A_g(r^-,\tau)$. Persistent external resources $y^+$ and $y^-$ are compared by the matched external difference $\Delta_{\mathrm{work}}((r,y^+){:}(r,y^-);\bar\tau,q,\lambda)=w_r^{\tau[y^-],q}(\lambda)-w_r^{\tau[y^+],q}(\lambda)$, with the term WRI reserved for the retained-state case $r^+$ versus $r^-$.

A present saving is attributed to earlier work when, relative to a cut-off separating earlier preparation from the current task, the following three conditions hold
\[
\text{production before cut-off}
\;\longrightarrow\;
\text{persistence into }r^+
\;\longrightarrow\;
\Delta_{\mathrm{work}}(r^+{:}r^-;\tau,q,\lambda)>0.
\]
Here $z^+$ is produced by counted computation before the cut-off between acquisition and continuation \textup{[\cref{app:reuse,app:internal-cut}]} and persists in $r^+=r[z^+]$ without being provided again as task input $p$ or other counted external input. On the current task, retaining it instead of the matched alternative $z^-$ yields the positive finite least-work saving shown above, as in knowledge compilation and amortised inference \cite{darwiche2002,gershman2014}.

\begingroup\emergencystretch=2em
In chess, this separates three distinct operations that can reach the same move quality $q$, namely (i) giving a position-specific hint $p^+$ in place of $p^-$ at fixed $(r,\lambda)$ so $\Work_g(r;\code{x,p^+})<\Work_g(r;\code{x,p^-})$, (ii) increasing the input bound from $\lambda_0$ to $\lambda_1>\lambda_0$ for the same state, giving $w_r^{\tau,q}(\lambda_1)\le w_r^{\tau,q}(\lambda_0)$ and (iii) replacing $r^-$ by $r^+$ (adding an endgame table to the state) so $w_{r^+}^{\tau,q}(\lambda)<w_{r^-}^{\tau,q}(\lambda)$ at the same input bound $\lambda$. 
\par\endgroup

To compare grasp across a family of tasks $\mathfrak T$, consider two states $r^+$ and $r^-$ that differ in one component, with all other conditions fixed. Choose a common input bound $\lambda_\tau$ for each task $\tau$. If every level reachable from $r^-$ requires at most as much work from $r^+$, then each finite work budget allows at least the same performance levels from $r^+$. Thus $r^+$ weakly grasp-dominates $r^-$ on each task \textup{[\cref{app:wri-to-grasp}]}
\[
\begin{gathered}
w_{r^+}^{\tau,q}(\lambda_\tau)\le w_{r^-}^{\tau,q}(\lambda_\tau)
\quad(\forall\tau\in\mathfrak T,\ \forall q\text{ with }w_{r^-}^{\tau,q}(\lambda_\tau)<\infty)
\\
\Longrightarrow\quad
\mathcal G_{r^-}^\tau(\lambda_\tau,w)\subseteq\mathcal G_{r^+}^\tau(\lambda_\tau,w)
\quad(\forall\tau\in\mathfrak T,\ 0\le w<\infty).
\end{gathered}
\]
Identical least-work values give mutual weak dominance. A replacement that saves work on one target can still increase it on another, as when a routine is selected to minimise total work across several tasks (\cref{sec:wri-search,app:wri-search}).

For sampling procedures at a target success rate $\alpha$, input bound $\lambda$ and seed distribution $\nu$, a positive saving in the reliable work budget
\[
\Delta_\alpha^\nu(r^+{:}r^-;\tau,q,\lambda)
=b_{r^-}^{\tau,\nu}(\lambda;q,\alpha)-b_{r^+}^{\tau,\nu}(\lambda;q,\alpha)>0
\]
defines a \emph{work reduction at fixed reliability}, which reduces to $\Delta_{\mathrm{work}}(r^+{:}r^-;\tau,q,\lambda)>0$ when $\nu=\delta_{u_0}$ is deterministic and $\mathcal H_\lambda$ contains all admitted task inputs of length at most $\lambda$ \textup{[\cref{app:reliable-work-reduction}]}. Under a hard cap $B$ on total work across retries, lowering per-attempt work from $c^-$ to $c^+<c^-$ can increase the number of completed attempts from $\lfloor B/c^-\rfloor$ to $\lfloor B/c^+\rfloor$. This can raise the probability of success $F_{r^+}^{\tau,\nu}(p;B,q)$ even when the best single-attempt quality is unchanged.

In transformer models, key-value (KV) caching lets later generation reuse intermediate results computed from a prompt prefix. It stores the attention keys and values, avoiding the work of computing them again \cite{hfcache}. Compare $r^+$, which retains a valid cache, with $r^-$, which rebuilds it from the same stored prefix. Suppose rebuilding adds a fixed amount $C>0$ of work to every completed run and leaves the subsequent generation unchanged. Their run work variables, success probabilities and reliable work budgets then satisfy \textup{[\cref{app:cache-reliability}]}
\[
X_{r^-}^{\tau,q}(p,u)=X_{r^+}^{\tau,q}(p,u)+C,
\qquad
F_{r^-}^{\tau,\nu}(p;b+C,q)=F_{r^+}^{\tau,\nu}(p;b,q),
\qquad
b_{r^-}^{\tau,\nu}(\lambda;q,\alpha)=b_{r^+}^{\tau,\nu}(\lambda;q,\alpha)+C,
\]
giving $\Delta_\alpha^\nu(r^+{:}r^-;\tau,q,\lambda)=C$ for every $0<\alpha\le 1$ for which both reliable work budgets are finite. Here $r^-$ rebuilds the cache once per run. If each attempt in a sequence instead starts afresh from the stored state, $r^-$ rebuilds it before every attempt. For a fixed task input, suppose a controller makes up to $m$ independent attempts, each costing $c$ work with the cache and succeeding with probability $\theta\in(0,1)$. A required success probability $0<\alpha\le1-(1-\theta)^m$ then needs a budget that allows $k_\alpha=\lceil\log(1-\alpha)/\log(1-\theta)\rceil$ attempts, so the reliable work budgets are $ck_\alpha$ and $(c+C)k_\alpha$ \textup{[\cref{app:cache-two-attempts}]}. The saving $Ck_\alpha$ is nondecreasing in the required success probability and changes in steps, which is why the comparison is indexed by $\alpha$.

In these comparisons, the cache saves work without changing the distribution of generated candidates. Other retained components change candidate generation or selection itself. A unit test changes which drafts pass internal filtering. Model distillation \cite{hinton2015} trains a smaller model using a larger model's outputs, while expert iteration \cite{anthony2017} trains a model using results of tree search. Both produce updated weights $r^+$. The work required to produce $r^+$ is a separate question, which we analyse in \cref{sec:acquisition}.

\subsection{Reducing work with a short program}
\label{sec:miniature}

The work-reducing information in our examples so far, such as an endgame table, a KV cache or Hashlife's hash table, stores results for reusable cases. A short program of fixed size can also reduce computation across arbitrarily large problem instances without storing past cases. Consider a row of binary cells updated in parallel according to
\[
a_i(t+1)=a_i(t)\oplus a_{i+1}(t),
\]
where $a_i(t)\in\{0,1\}$ is the state of cell $i$ at step $t$ and $\oplus$ is addition modulo two. In Wolfram's numbering of elementary cellular automata, this update is rule 102 \cite{wolfram1984}. For a horizon $T=2^n$, given the window of $T+1$ initial cells from $a_i(0)$ to $a_{i+T}(0)$ as public context, the task $\tau_T$ is to compute $a_i(T)$, counting serial tape-head moves and cell updates, with a binary evaluator whose level $1$ means a correct answer. A naive simulation program stored in $r_{\mathrm{sim}}$ updates the row step by step, performing $T(T+1)/2$ cell updates and $\Theta(T^2)$ work in total \textup{[\cref{app:cluster-miniature,app:corollary}]}.

After two steps, the middle cell enters twice and cancels modulo two, $a_i(2)=a_i(0)\oplus a_{i+1}(0)\oplus a_{i+1}(0)\oplus a_{i+2}(0)=a_i(0)\oplus a_{i+2}(0)$, and repeated doubling gives at every power-of-two horizon $T=2^n$
\[
a_i(2^n)=a_i(0)\oplus a_{i+2^n}(0),
\]
a standard property of additive rules \cite{martin1984}. A state $r_{\mathrm{short}}$ that replaces the step-by-step simulator with the short program that scans from $a_i(0)$ to $a_{i+2^n}(0)$ and adds the two endpoint bits uses $w_{r_{\mathrm{short}}}^{\tau_T,1}(0)=\Theta(T)$ serial head steps instead of $w_{r_{\mathrm{sim}}}^{\tau_T,1}(0)=\Theta(T^2)$ at input bound $\lambda=0$. This gives a work saving $\Delta_{\mathrm{work}}(r_{\mathrm{short}}{:}r_{\mathrm{sim}};\tau_T,1,0)=\Theta(T^2)$ from a program of $O(1)$ description length independent of $n$ \textup{[\cref{app:corollary}]}. Using $r_{\mathrm{short}}$ on the task needs only this $\Theta(T)$ work, whereas discovering and verifying that program from scratch is a separate acquisition task, which we consider next.

\subsection{Acquisition depth}
\label{sec:acquisition}

So far we have compared the work saved by a component already present in a state $r$. A separate question is how much work it takes to obtain such a state from an earlier starting condition, such as building an endgame table or training a language model. We define \emph{acquisition depth} by applying our depth definition (\cref{sec:machine}) to the task of producing a state that meets a required capability.

Suppose we want the system $g$ to reach a target performance level $q_j$ within input bound $\lambda_j$ and work budget $w_j$ on each task $\tau_j$ in a finite list $\Gamma=((\tau_j,\lambda_j,w_j,q_j))_{j=1}^m$. The set of states of $g$ that meet all $m$ requirements is
\[
S_\Gamma^g
=\bigl\{\,r\in\mathcal R_g : (\lambda_j,w_j,q_j)\in\A_g(r,\tau_j)\text{ for every }j\in\{1,\ldots,m\}\,\bigr\}
=\bigl\{\,r\in\mathcal R_g : w_{g,r}^{\tau_j,q_j}(\lambda_j)\le w_j\;\forall j\,\bigr\}.
\]
Assume $S_\Gamma^g$ is nonempty and that candidate-state validity and these bounded requirements can be checked by a terminating procedure, as specified in \cref{app:acquisition}. Now consider a state-producing system $A$ starting from an initial state $a$ (such as a training procedure whose initial checkpoint encodes a valid starting state) whose goal on the acquisition task $\tau_{\mathrm{acq}}=(x_{\mathrm{acq}},S_\Gamma^g)$ is to output any state $r\in S_\Gamma^g$. Here $(x,S)$ denotes the binary task with context $x$ whose level $1$ accepts exactly the outputs in $S$ \textup{[\cref{app:accepted-set}]}. The evaluator of $\tau_{\mathrm{acq}}$ reads the output of $A$ as a state of $g$ and checks whether $g$ in that state meets every requirement in $\Gamma$. Acquisition work includes producing and writing out this state and any task tests performed during acquisition. The budgets $w_j$ specify the work permitted when the state is used on the later tasks. Applying our depth definition to this task, the \emph{acquisition depth} of $\Gamma$ from $(A,a)$ at margin $s\in\Nat$ is \textup{[\cref{app:acquisition}]}
\[
d_{A,a}^{\tau_{\mathrm{acq}},1}(s)
=w_{A,a}^{\tau_{\mathrm{acq}},1}\!\left(\C_{\tau_{\mathrm{acq}}}(1\given x_{\mathrm{acq}})+s\right),
\]
where $\C_{\tau_{\mathrm{acq}}}(1\given x_{\mathrm{acq}})=\min\{|v| : U_0(v,x_{\mathrm{acq}})\in S_\Gamma^g\}$ is the length of a shortest reference program that halts on $U_0$ and returns a qualifying state in $S_\Gamma^g$, given $x_{\mathrm{acq}}$ (and $d_{A,a}^{\tau_{\mathrm{acq}},1}(s)=\infty$ when no qualifying state is reachable under this input bound).

For the training comparisons below, $x_{\mathrm{acq}}$ specifies the capability requirement and the fixed training data available before the run, and the starting checkpoint belongs to $a$. Both $A$ and $U_0$ then have the data as context, while reading and processing them remain part of the work counted for $A$. When $x_{\mathrm{acq}}$ specifies the requirements in a form that one fixed reference program can check, that program can list valid states of $g$ and return the first that meets them. Then $\C_{\tau_{\mathrm{acq}}}(1\given x_{\mathrm{acq}})\le c$ for a constant $c$ that does not depend on the training data, and because least work cannot increase with the input bound,
\[
w_{A,a}^{\tau_{\mathrm{acq}},1}(c+s)\le d_{A,a}^{\tau_{\mathrm{acq}},1}(s)\le w_{A,a}^{\tau_{\mathrm{acq}},1}(s).
\]
Acquisition depth is therefore the least work for $A$ to produce a qualifying state at an input bound within a fixed number of symbols of $s$.

In LLM terms, since acquisition depth $d_{A,a}^{\tau_{\mathrm{acq}},1}(s)$ is indexed by the starting state $a$, fine-tuning from a pre-trained checkpoint $a_1$ can satisfy $d_{A,a_1}^{\tau_{\mathrm{acq}},1}(s)\ll d_{A,a_0}^{\tau_{\mathrm{acq}},1}(s)$ relative to training from an uninitialised state $a_0$. Under fixed external inputs, any realised training, engineering or evolutionary run from $(A,a)$ that takes an acquisition input $p$ with $|p|\le\lambda$ and emits a qualifying state $r\in S_\Gamma^g$ in total work $W$ (including state-extraction and output work) establishes the upper bound \textup{[\cref{app:acquisition-upper-bound}]}
\[
w_{A,a}^{\tau_{\mathrm{acq}},1}(\lambda)\le W,
\qquad\text{and}\qquad
|p|\le\C_{\tau_{\mathrm{acq}}}(1\given x_{\mathrm{acq}})+s
\;\Longrightarrow\;
d_{A,a}^{\tau_{\mathrm{acq}},1}(s)\le W.
\]
A requirement $\Gamma$ can have high acquisition depth from $(A,a)$ even though a qualifying state requires little work on the later tasks, as in discovering which worker can perform each role \textup{[\cref{app:division-of-labour}]}.

\subsection{Finite search for work-reducing routines}
\label{sec:wri-search}

\Cref{sec:miniature} illustrated how a stored program can reduce later work, and how finding such a program is a separate task. Here an acquisition process chooses which routine to retain by comparing the total work needed to run each candidate on a fixed collection of tasks and check its answers. The construction below allows every candidate to be evaluated in finite work, even when its own program fails to halt. A known method for solving all the tasks is among the candidates, so the selected routine uses no more total work on those tasks than that method.

Fix a finite collection of training tasks $(\tau_i)_{i=1}^m$, each with a terminating correctness test, and a length bound $L\in\Nat$ defining the finite class of candidate routines $\mathcal Z_{\le L}=\{z\in\Sig^* : |z|\le L\}$ \textup{[\cref{app:cluster-wri-search,app:wri-search}]}. For each task, an exhaustive procedure lists possible answers and checks them until it finds one that is accepted. Each task has an accepted answer, and every check terminates, so this procedure eventually finishes. The empty candidate $z=\eps$ selects this procedure alone, with finite work $w_{r_\eps}^{\tau_i,1}(0)<\infty$. For a nonempty candidate $z\in\mathcal Z_{\le L}\setminus\{\eps\}$, state $r_z$ alternates steps of the code in $z$ with steps of the exhaustive procedure. If the candidate halts, the system completes the check of its answer and returns it if accepted. If the check fails, the exhaustive procedure continues on its own. Invalid candidate code is ignored. When the exhaustive procedure finishes, its accepted answer is returned. All steps of running both procedures, switching between them and checking answers count in $\Work_g$. The total work of $r_z$ across $(\tau_i)_{i=1}^m$ at input bound $\lambda=0$ is
\[
W(z)=\sum_{i=1}^m w_{r_z}^{\tau_i,1}(0).
\]
Evaluating every $z\in\mathcal Z_{\le L}$ and selecting a minimiser $z^*\in\arg\min_{z\in\mathcal Z_{\le L}}W(z)$ guarantees
\[
W(z^*)=\min_{z\in\mathcal Z_{\le L}}W(z)\le W(\eps),
\]
because $\eps\in\mathcal Z_{\le L}$. Whenever $W(z^*)<W(\eps)$, there is at least one task $\tau_i$ with $\Delta_{\mathrm{work}}(r_{z^*}{:}r_\eps;\tau_i,1,0)=w_{r_\eps}^{\tau_i,1}(0)-w_{r_{z^*}}^{\tau_i,1}(0)>0$, so $z^*$ is verified WRI on $\tau_i$ \textup{[\cref{app:wri-search}]}. Restricting selection to candidates that require no more work on any training task, $\mathcal Z_{\le L}^{\mathrm{dom}}=\{z\in\mathcal Z_{\le L} : w_{r_z}^{\tau_i,1}(0)\le w_{r_\eps}^{\tau_i,1}(0)\text{ for all }i\in\{1,\ldots,m\}\}$, additionally guarantees grasp dominance on each task across $(\tau_i)_{i=1}^m$. Because later task execution includes checking and switching, a candidate improves on the baseline only when its total work, including these costs, is smaller. For the cellular automaton of \cref{sec:miniature}, checking an answer by direct simulation would still need work quadratic in $T$, so reaching the linear work there requires another justification of the routine's correctness.

Let $(A,a)$ be the acquisition system and starting state used for this search, and let it admit only the empty input. If the acquisition task accepts exactly the candidate states that minimise total work on the training tasks, the counted work $W_A$ of the search, which includes evaluating every candidate on every training task and writing out the selected state, bounds its acquisition depth, $d_{A,a}^{\tau_{\mathrm{acq}},1}(s)\le W_A$ for every margin $s$ \textup{[\cref{app:acquisition-upper-bound}]}.

When every executed step counts as one unit of work, as here, a positive rational length penalty $\gamma$ also gives a finite search over all candidate codes. Use the same execution procedure for every finite code and minimise $W(z)+\gamma|z|$. After evaluating the baseline total work $W(\eps)$, it suffices to enumerate codes of length at most $\lfloor W(\eps)/\gamma\rfloor$. Each remaining objective is computable exactly, and every longer code has objective above the baseline. With a fixed tie rule the search terminates, and its selected code satisfies $W(\eps)-W(z^*)\ge\gamma|z^*|$. A nonempty selected code therefore reduces total training work.

A controller can search for a task-specific input or retry sampled attempts to solve the current task. Its execution work includes running and checking the candidates \textup{[\cref{app:guidance-search,app:retries}]}. The acquisition search above instead retains a routine selected across several tasks. The length-bounded and the penalised searches both select $z^*$ on the finite collection $(\tau_i)_{i=1}^m$, so saving work there does not guarantee that $z^*$ will save work across a wider family of future tasks $\mathfrak T\supset\{\tau_1,\ldots,\tau_m\}$ (\cref{sec:judging}).

\subsection{Work-bounded agents}
\label{sec:agents}

An agent interacts with an environment by choosing actions in response to observations. Its policy specifies an action, or a distribution over actions, given the interaction history. For a goal-directed agent, a utility function expresses how well the resulting interaction meets its goal. The policy may pursue this goal imperfectly. The universal AIXI agent evaluates policies over a class containing every computable environment \cite{hutter2003}. It gives prior weight $2^{-\ell}$ to self-delimiting environment programs of length $\ell$ on a fixed reference computer. The ideal policy maximises expected cumulative reward under the resulting mixture without taking the agent's computation into account.

Even an environment with a short description can require extensive computation to predict the effects of actions. The idealisation therefore omits the computational savings from doing this work once and reusing its results. Hutter considers length and time bounds on environment models and on the agent's policies \cite{hutter2003}. We study implementations of policies using a fixed system $g$ and a changing retained state $r_t$, which can contain programs, learned parameters, data and unfinished computations. Each comparison fixes the goal, system boundary, work rules, storage capacity and rule for external inputs. The agent's work includes reading observations, selecting computations, using or retaining their results and any updating of its beliefs about the environment.

For a prediction or planning task $\tau$ with required performance $q$, the work-bounded complexity $K_{r_t}^{\tau,q}(B)$ from \cref{sec:machine} gives the length of a shortest additional task input needed to reach $q$ within work budget $B$ from the current state. Reaching that target need not require simulating the underlying process step by step (\cref{sec:miniature}). A learned evaluator or a retained intermediate result can reduce the input needed at that budget or make a target attainable at the same work and input bounds. Generating instructions or model descriptions during deliberation takes work, as does using them.

The choice of computations is studied in bounded optimality and rational metareasoning \cite{russellsubramanian1995,russellwefald1989}. The Optimal Ordered Problem Solver reuses programs found for earlier tasks \cite{schmidhuber2004}. Constructing a resource can serve the agent's goal through the later decisions it supports. This leads to the comparison of complete continuations in \cref{sec:horizon}, including the computation the agent performs to choose what to construct and retain.

\section{Evaluating capability}
\label{sec:diagnosis}

The search in \cref{sec:wri-search} can find a routine that saves work on its training tasks without helping on future ones. In this section we ask what enables a measured capability and whether the evaluation captures the intended goal. This requires separating what is stored in the system from help given on the occasion, and examining what success on a proxy test establishes \textup{[\cref{app:cluster-evaluation}]}.

\begin{table}[!htbp]
\centering\small
\begin{tabular}{@{}>{\raggedright\arraybackslash}p{0.32\textwidth}>{\raggedright\arraybackslash}p{0.60\textwidth}@{}}
\toprule
Term & Definition \\
\midrule
capability & The targets a system can reach within their corresponding work budgets and input bounds. \\
assistance & A change in task input that, with the state and input bound fixed, reduces the work needed for a target or improves performance within a work budget. \\
learning that reduces least work & Earlier counted work produces a component kept in the state that reduces the least work needed for a target in a matched comparison at a fixed input bound. \\
reuse & A run depends on a retained product of earlier counted work, under the conditions set out in \cref{sec:four}. \\
sound test & A test for which every passing output meets the intended goal. \\
complete test & A test that passes every output meeting the intended goal. \\
\bottomrule
\end{tabular}
\caption{Glossary for \emph{Evaluating capability}.}
\end{table}

\subsection{Assessing what a system knows}
\label{sec:four}
A system may have enough information to derive an answer and still need considerable work to produce it. An LLM may need to generate many tokens to put together a correct answer.

On a single task $\tau$ at performance level $q$, input bound $\lambda$ and work threshold $b$ under fixed boundary and interface conventions, we say that the system has resource-relative \emph{knowledge} of how to reach level $q$ from state $r$ when
\[
w_r^{\tau,q}(\lambda)\le b.
\]
This is the attainability condition of \cref{sec:machine}. For sampled execution, the reliable work budget in \cref{sec:llm-execution} gives the corresponding criterion at a required success probability.

To describe capability across tasks, we collect the targets a system can reach within the corresponding work budgets and input bounds. For a family of test specifications $\mathfrak F=((\tau_i,\lambda_i,w_i,q_i))_{i\in I}$, we define the index set of reachable targets \textup{[\cref{app:cap-profile}]}
\[
\operatorname{Cap}_{g,r}(\mathfrak F)
=\bigl\{\,i\in I : (\lambda_i,w_i,q_i)\in\A_g(r,\tau_i)\,\bigr\}
=\bigl\{\,i\in I : w_{g,r}^{\tau_i,q_i}(\lambda_i)\le w_i\,\bigr\}.
\]
State $r_1$ weakly dominates state $r_0$ on the suite $\mathfrak F$ when it reaches every target that $r_0$ reaches, so $\operatorname{Cap}_{g,r_0}(\mathfrak F)\subseteq\operatorname{Cap}_{g,r_1}(\mathfrak F)$. Dominance on a suite compares attainable regions only at the listed points. If the attainable region of $r_0$ is contained in that of $r_1$ on every task, then $r_1$ dominates $r_0$ on every suite. A capability requirement $\Gamma$ as in \cref{sec:acquisition} is such a suite, and $S_\Gamma^g=\{r\in\mathcal R_g : \operatorname{Cap}_{g,r}(\Gamma)=\{1,\ldots,m\}\}$.

With the state $r\in\mathcal R_g$ fixed, replacing a baseline input $p^-$ with another admitted input $p^+$, both of length at most $\lambda$, \emph{assists} the run when it reaches target level $q$ with less work or reaches a higher performance level within budget $b$. On a code repair task, for example, $p^+$ might identify the faulty function without giving a patch.

Learning can change what a system can achieve within a work budget even when it could already produce the same answer with more work. To attribute a saving to \emph{learning}, we fix the input bound $\lambda$ and compare states that differ only in a retained component. Under the conditions of \cref{sec:wri}, $z^+$ was produced by earlier counted work, remains in $r^+=r[z^+]$ without being passed again in $p$ and yields a positive finite matched saving $\Delta_{\mathrm{work}}(r[z^+]{:}r[z^-];\tau,q,\lambda)>0$.

To compare how widely retained information helps, we take a family $(\tau_j,q_j,\lambda_j)_{j\in J}$ with finite least-work values in both matched states and consider the work savings $\Delta_{\mathrm{work}}(r^+{:}r^-;\tau_j,q_j,\lambda_j)$ on its tasks. Each positive value identifies a task that benefits and gives the work saved on it. Breadth is distinct from what the component contains. A component may store answers to particular tasks, as in memorisation, or a method for obtaining answers, and it may contain both. A table of many answers can help across a large task family. A short method may help only within a narrow class, as in the cellular-automaton example of \cref{sec:miniature}.

Replacing a stored component $z$ can change a particular run without changing the least work needed to complete the task. For a run $\xi_{r[z]}$ at fixed task input, let $\mathcal O(\xi)$ consist of its halting status, emitted output $e$, work and the ordered sequence of its writes to retained state, under a fixed observation rule \textup{[\cref{app:reuse}]}. We say that $\xi_{r[z]}$ depends on $z$ when some valid replacement $z'\in\mathcal Z$ changes this observation
\[
\mathcal O\bigl(\xi_{r[z]}\bigr)\ne\mathcal O\bigl(\xi_{r[z']}\bigr).
\]
This is reuse of earlier work when $z$ was produced by counted work before the acquisition cut-off and retained without being provided again as task input or other counted external input. Writes to retained state are included because they can affect later runs.

We can now show why a component that reduces the least work needed for a task must affect at least one successful run. If $w_{r[z^+]}^{\tau,q}(\lambda)<w_{r[z^-]}^{\tau,q}(\lambda)$, then on an admitted input that attains the smaller value the run from $r[z^-]$ either fails to reach $q$ or uses more work, so at least one least-work successful run from $r[z^+]$ depends on the component. The converse fails, because a replacement can change the run while another input or method reaches $q$ at the same least work. To show that a component makes $q$ attainable within budget $b$ at input bound $\lambda$, we compare it with a specific alternative $z^-$ and establish $w_{r[z^+]}^{\tau,q}(\lambda)\le b < w_{r[z^-]}^{\tau,q}(\lambda)$.

\Needspace{7\baselineskip}
\subsection{Evaluating success}
\label{sec:judging}

The set of reachable targets depends on the tasks and how performance is evaluated. Because practical evaluations usually rely on proxy benchmarks rather than the user's ultimate goal, we ask two further questions, namely when passing the proxy test guarantees that the intended goal is met, and how much work passing it requires.

We compare the proxy evaluator $\widehat\perf$ (such as an automated test suite or benchmark metric) with the intended evaluator $\perf_\tau$ representing the actual goal \textup{[\cref{app:reliable-transfer,app:transfers}]}. They define two tasks with the same context, the proxy task $\widehat\tau$ with order $\preceq_{\widehat\tau}$ and the intended task $\tau$, whose levels can differ. For a proxy level $\hat q$ and a goal level $q$, let $\widehat S=\{e\in\Sig^* : \hat q\preceq_{\widehat\tau}\widehat\perf(e)\}$ be the set of outputs that pass the proxy test, and let $S^*=\{e\in\Sig^* : q\preceq_\tau\perf_\tau(e)\}$ be the set of outputs that satisfy the intended goal. The proxy test is \emph{sound} for the intended goal when every output that passes the test also satisfies the goal
\[
\widehat S\subseteq S^*.
\]
Raising the proxy level can only shrink $\widehat S$, so every proxy level above a sound one is also sound.

We write the accepted set in place of the task and level. For the binary task that accepts exactly the outputs in $S$, this gives the least work $w_r^{S}(\lambda)$, depth $d_r^{S}(s)$, success probability $F_r^{S,\nu}(p;b)$ and reliable work budget $b_r^{S,\nu}(\lambda;\alpha)$ \textup{[\cref{app:accepted-set}]}. Use the same execution procedure, context, task-input class and seed distribution for both evaluations. For a task input $p\in\mathcal H_\lambda$ and a budget $b$, write $e$ for the returned output and $W$ for the work of a halting run. The events below require such a returned output. The success probabilities on the test and on the goal then differ by the probability of a false pass minus that of a false fail
\[
F_r^{\widehat S,\nu}(p;b)-F_r^{S^*,\nu}(p;b)
=\Pr\bigl[W\le b,\ e\in\widehat S\setminus S^*\bigr]-\Pr\bigl[W\le b,\ e\in S^*\setminus\widehat S\bigr].
\]
For a sound test the first term vanishes. Success on the goal is then at least as likely as success on the test, and both the reliable budget and the least work for the goal are at most their values for the test
\[
F_r^{\widehat S,\nu}(p;b)\le F_r^{S^*,\nu}(p;b),
\qquad
b_r^{S^*,\nu}(\lambda;\alpha)\le b_r^{\widehat S,\nu}(\lambda;\alpha),
\qquad
w_r^{\widehat S}(\lambda)\ge w_r^{S^*}(\lambda).
\]
A pass on a sound test therefore shows that the goal is met. A test is \emph{complete} when $S^*\subseteq\widehat S$, as when every program meeting the intended specification passes every unit test. Completeness reverses the three inequalities. For deterministic execution, if every admitted input of length at most $\lambda$ fails to pass a complete test within budget $b$, the goal is also unattainable within these bounds. For sampled execution, an upper bound on the test success probability for every admitted input also bounds the goal success probability.

We illustrate the distinction with a program that passes a finite set of unit tests but gives a wrong answer on an untested input. Let the task be to write a program for a function $f$ described in $x$, let the test check its values on a finite proper subset of the input domain and let the goal require correctness on every input. Every program that meets the goal passes the test, so the test is complete. A lookup program that returns the tested values and a wrong default on some untested input passes the test but fails the goal, so the test is not sound. Retaining the tested inputs and their values can reduce the work needed to produce such a lookup program. Whether this lowers the least work for the test, for the goal or for both depends on the fixed system, and is settled by comparing the matched states under the same task interface.

When the test is not sound, minimising work on the test can select outputs that fail the goal, an effect related to Goodhart's law \cite{goodhart1984,strathern1997,manheim2018}. If $w_r^{\widehat S}(\lambda)$ is finite, every least-work run that passes the test fails the goal exactly when
\[
w_r^{\widehat S}(\lambda)<w_r^{\widehat S\cap S^*}(\lambda).
\]
When both values are finite, $w_r^{\widehat S\cap S^*}(\lambda)-w_r^{\widehat S}(\lambda)$ is the extra work needed to pass the test while also meeting the goal, and for a sound test it is zero. For the depth comparison, assume both accepted sets are nonempty. For a sound test the smaller set $\widehat S$ needs a shortest program on $U_0$ at least as long as $S^*$ does, so their reference complexities differ by $\Delta\C=\C(\widehat S\given x)-\C(S^*\given x)\ge 0$. Increasing the goal's margin from $s$ to $s+\Delta\C$ makes the two input bounds equal ($\C(S^*\given x)+s+\Delta\C=\C(\widehat S\given x)+s$), which yields the depth bound \textup{[\cref{app:transfers}]}
\[
d_r^{S^*}(s+\Delta\C)\le d_r^{\widehat S}(s).
\]

Whether passing a test establishes the intended goal is distinct from how much work passing it requires. To express the latter, we say that a level $q$ of a task $\tau$, which may be the proxy task, is \emph{demanding} beyond work budget $b$ from $(g,r)$ at input bound $\lambda$ when
\[
w_{g,r}^{\tau,q}(\lambda)>b,
\]
equivalently, the work-bounded complexity $K_{g,r}^{\tau,q}(b)$ exceeds $\lambda$. At $\lambda=\C_\tau(q\given x)+s$ this reads $d_{g,r}^{\tau,q}(s)>b$. A demanding level is one that the system does not know how to reach in the sense of \cref{sec:four}. A complete test that is demanding beyond $b$ shows that the goal is too, since completeness gives $w_{g,r}^{S^*}(\lambda)\ge w_{g,r}^{\widehat S}(\lambda)$, while a sound test can be demanding when the goal is not.

We can also ask whether one fixed evaluator can remain demanding at higher performance levels. Fix a system and starting state, a margin $s$ and a computable strictly increasing unbounded sequence of work thresholds $(T_k)_{k\ge 1}$. Whenever an algorithm can list the finitely many halting runs within each input and work bound \textup{[\cref{app:exhaustible}]}, a finitely specified evaluator can accept at each level $q_k$ only responses that the system cannot return within work $T_k$ from any input allowed at margin $s$, so that $d_{g,r}^{\tau,q_k}(s)>T_k$ \textup{[\cref{app:demanding}]}. Since the same fixed construction computes an accepted response from $k$, that response has a reference description of length at most a constant plus $O(\log k)$, even though reaching $q_k$ requires the stated work from this system, as in the classical existence of deep strings \cite{bennett1988}.

When one algorithm lists those runs uniformly across an enumerable family of system--state pairs $(g_i,r_i)$, the construction extends to a single evaluator fixed in advance for the whole family. For each fixed pair $(g_i,r_i)$, margin $s$ and work threshold $b$, all sufficiently high levels $q_k$ satisfy $d_{g_i,r_i}^{\tau,q_k}(s)>b$ \textup{[\cref{app:public}]}. In both cases, for a pair that can execute reference programs with a fixed number of extra input symbols and finite simulation work, these levels have finite depth whenever $s$ covers that input overhead \textup{[\cref{app:demanding,app:public}]}.

\section{Persistence and organisation}
\label{sec:coordination}

The work saved by retained information depends on its continued availability and on the arrangements through which it is used. Shared storage can preserve a capability when participants are replaced, while downstream reliance on existing responses can make replacing a component costly. When several components work concurrently, retained conventions can reduce coordination work. We examine these questions by studying persistence through member turnover and function-preserving replacement, then comparing pooled, organised distributed, unorganised distributed and isolated arrangements under common outer conditions \textup{[\cref{app:retention,app:cluster-switching,app:cluster-coordination}]}.

\begin{table}[!htbp]
\centering\small
\begin{tabular}{@{}>{\raggedright\arraybackslash}p{0.32\textwidth}>{\raggedright\arraybackslash}p{0.60\textwidth}@{}}
\toprule
Term & Definition \\
\midrule
persistence of a capability & Continued attainability of a target after any sequence of admitted disturbances, with each disturbance followed by a successful run within the input and work bounds, including any repair. \\
direct switch & Installation of an alternative architecture, assessed by the work required. \\
function-preserving switch & A switch that also keeps the required responses available within their input and work bounds. \\
compatibility & Preservation of those responses within their bounds without keeping the old implementation. \\
\bottomrule
\end{tabular}
\caption{Glossary for \emph{Persistence and organisation}.}
\end{table}

\Needspace{7\baselineskip}
\subsection{Maintaining capability}
\label{sec:boundaries}

In the computer science examples above, including LLMs, the system boundary determines which resources and work are included in the comparison. In many real systems, work is also needed to maintain a boundary that keeps resources and information available within the system. A cell membrane is a biological example, and organisations also maintain shared resources and control access to them.

In long-running projects and organisations, capabilities often persist and work is performed to maintain them, even as individual contributors or agent sessions are replaced, for example because earlier outcomes are kept in shared storage. How we account for that stored information depends on the chosen system boundary. It belongs to the state $r[z]$ when the boundary encloses the relevant project resource, and to persistent external state $y[z]$ in $\tau[y[z]]$ when the boundary encloses only the individual model or worker \textup{[\cref{app:matched,app:retention}]}. When maintaining a capability requires work, a budget for the complete process must cover that work and the work of using the capability. 

If deleting a component $z$ from $(r,y)$ leaves a valid state, we can evaluate the replacement $z\mapsto\eps$ for a previously attainable target, with all other conditions fixed. When both least-work values are finite, a strict increase shows a positive work-reducing contribution ($\Delta_{\mathrm{work}}>0$). An infinite post-deletion value instead removes attainability at that input bound. An unchanged finite minimum ($\Delta_{\mathrm{work}}=0$) leaves open whether another stored copy or another procedure gives the result or $z$ was unnecessary for that minimum, even if the observed execution trace changes.

\begingroup\widowpenalty=10000
We now ask whether a system can continue to reach the same target within its work budget after disturbances. Fix a target $(\tau,q,\lambda)$, a work budget $b$ and a class $\mathcal D$ of disturbances, each a map between valid states, such as the replacement of a participant or the deletion of a stored copy. The capability \emph{persists} from $r$ under $\mathcal D$ within budget $b$ when there is a set $\mathcal S$ of valid states containing $r$ such that, from every state in $\mathcal S$ and from its image under every disturbance in $\mathcal D$, some admitted task input of length at most $\lambda$ gives a run that reaches level $q$ within work $b$ and ends in a state in $\mathcal S$. Each disturbance is followed by a completed run before the next disturbance occurs. The successful input may depend on the resulting state. A run may have to repair what a disturbance removed, and whether it can do so depends on $b$. Consider a word kept both in shared storage and in a participant's memory. After either participant turnover or deletion of the shared copy, a run restores the missing copy and emits the word \textup{[\cref{app:retention}]}. The capability persists when each run's work budget covers both operations. Only the empty task input is admitted, so restoration proceeds without fresh task input. The budget threshold is specific to this restoring procedure. Persistence of a capability differs from the retention of a component in \cref{sec:wri}, although a retained component is often what makes a capability persist.\par\endgroup

\Needspace{7\baselineskip}
\subsection{Replacing components}
\label{sec:switching}

In \cref{sec:boundaries} above, we considered collectives that can replace active members while keeping information in shared storage and maintaining it within the available work budget. Replacing a library, for example to reduce the work required for some tasks, raises a further problem. Downstream tools may rely on the responses it already gives, so installing the new component is only part of the work \textup{[\cref{app:cluster-switching}]}.

We specify the \emph{function} to be preserved by a finite table of required context--response pairs
\[
\Phi=\{(x_j,z_j):j\in J\},
\]
where $J$ is finite. The function $\Phi$ is \emph{available} from state $r$ under per-query bounds $(\lambda_j,w_j)_{j\in J}$ when, for every $j\in J$, there is an admitted input $p_j$ of length at most $\lambda_j$ for which the system returns $z_j$ on context $x_j$ within work $w_j$. Downstream tasks may require $\Phi$ to remain available within those bounds, as when an analysis pipeline relies on the responses given by a chess library.

Migration work is the least work of producing the new state from the current one, an acquisition task in the sense of \cref{sec:acquisition} \textup{[\cref{app:portability}]}. From a current state $r_0$, a \emph{direct switch} to an alternative architecture $r_1$ requires only installing $r_1$ at migration work $w_{\mathrm{direct}}$. A \emph{function-preserving switch} must produce a state $r_1^{\Phi}$ that also keeps $\Phi$ available within $(\lambda_j,w_j)_{j\in J}$, and its least work is $w_{\mathrm{preserve}}$. Such a replacement is \emph{compatible} with $\Phi$ if it does so without retaining the old implementation.

To measure the extra migration work, we compare preservation with a baseline under the same protocol conventions, either the direct switch or continuation with the existing implementation. Let $w_{\rm comparator}$ be the least work of the chosen baseline. Whenever both costs are finite, their difference is
\[
E_{\mathrm{mig}}(\Phi)
= w_{\rm preserve}-w_{\rm comparator}.
\]
Consider two initial designs that each take $O(1)$ work to adopt from scratch. Suppose that later tasks require all $T=|J|$ responses in $\Phi$, and these responses are incompressible given the information that remains accessible after the switch. Any function-preserving replacement checkpoint then requires $\Omega(T)$ bits of state and $\Omega(T)$ serial output work to construct. This lower bound comes from writing out the information that must be kept. Any conversion, checking and other migration work contributes to the total. We give a construction with $T$ single-bit responses in which copying them into a compatible replacement checkpoint takes $O(T)$ work. Together with the lower bound, this gives $w_{\mathrm{preserve}}=\Theta(T)$, whereas continuing with the existing implementation or making a direct switch costs $O(1)$ \textup{[\cref{app:portability-representation,app:portability}]}. If the external store containing $(z_j)_{j\in J}$ remains accessible after the switch within the query budgets, a constant-size \emph{adapter} achieves compatibility without copying the table \textup{[\cref{app:portability}]}.

When the table is highly compressible, one might try to reduce this $\Theta(T)$ migration cost by replacing the stored table with a short program that generates it. However, when each query has a tight work budget, time-bounded incompressibility \cite{daylight2009} can still force any standalone replacement to store a large state. Consider a setting where each query $j\in\{1,\ldots,T\}$ restarts from the same saved checkpoint $r_{\mathrm{ckpt}}$ within a polynomial per-query budget $B_T$, so results computed on one query cannot be used in the next. With no answer given with the query and no post-switch access to the external table, we can construct a response table $\Phi_T$ generated by a fixed short program given $T$ \textup{[\cref{app:budget-preservation}]}. No checkpoint shorter than $\lfloor T/2\rfloor$ bits can return all of its responses within $B_T$ steps per query. In this construction, the minimum migration work is $w_{\mathrm{preserve}}=\Theta(T)$.

\Needspace{7\baselineskip}
\subsection{Comparing organisational arrangements}
\label{sec:together}

\Cref{sec:boundaries,sec:switching} examined in our formalism how information in shared storage preserves capability across changing members and how downstream dependencies can make replacing a component costly. In this section we consider several components that operate concurrently rather than in sequence. This setting leads to coordination work between components that is not present when a single implementation does all the work. Organising the group with shared conventions can lower required work relative to the same components operating without those conventions.

To study how different organisational arrangements affect the coordination work required, we compare an \emph{organised distributed system} $M_{\rm org}$ against three arrangements in an organisational family \textup{[\cref{app:orgfamily}]}. We illustrate them with a reporting example of four tasks, one for each pair $(u,v)\in\{0,1\}^2$, in which two components under common outer resources receive the separate bits $u$ and $v$ and must report their parity $u\oplus v$. Either component can report the result, which is $1$ when the bits differ and $0$ when they agree \textup{[\cref{app:org-separation}]}. Both run the same program in rounds in which every message is sent before any is received, and each executed instruction counts as one unit of work. In $M_{\rm org}$ the components have retained complementary roles, so the sender transmits its bit and the reporter alone computes and reports the parity, in $13$ operations. An \emph{unorganised distributed variant} $M_{\rm flat}$ keeps the same components, program and storage but has no roles, so both send and both report, in $18$ operations. Retained roles therefore save the send, delivery and receipt of one bit, one parity computation and one report on every task. Identical components can establish roles only from an observed difference \cite{angluin1980,itairodeh1990}. Choosing roles on the first task whose two bits differ adds three operations to that task and saves five on every later task. Alternatively, the components can exchange random role proposals with their bits, lowering the expected total work once enough tasks follow.

A \emph{pooled baseline} $M_{\rm central}$ places all observations and action channels inside a single process without internal communication interfaces and returns the parity in five operations. If every successful distributed run has a pooled run for the same target, within the same task-input bound and with no more work, the pooled least work bounds the distributed least work from below, even when locality, privacy or permissions prevent deploying the pooled process. The comparison here counts total work. Elapsed time, locality and access restrictions can also affect the choice of arrangement. Finally, \emph{isolated components} $M_{\rm isolated}$ disable inter-component communication altogether, including shared-state, timing and termination signals. Each component then sees only its own bit, so no fixed isolated implementation succeeds on all four input pairs.

For a common task $\tau$, target level $q$ and input bound $\lambda$, write $w_{\rm pool}=w_{M_{\rm central}}^{\tau,q}(\lambda)$, $w_{\rm org}=w_{M_{\rm org}}^{\tau,q}(\lambda)$ and $w_{\rm unorg}=w_{M_{\rm flat}}^{\tau,q}(\lambda)$ for the least work required under the pooled, organised and unorganised arrangements. Whenever $w_{\rm pool},w_{\rm org},w_{\rm unorg}<\infty$, comparing $M_{\rm org}$ against $M_{\rm central}$ gives the work difference $T_{\rm coord}$, and comparing $M_{\rm flat}$ against $M_{\rm org}$ gives the work saving $G_{\rm org}$
\[
T_{\rm coord}(\tau,q,\lambda)=w_{\rm org}-w_{\rm pool},
\qquad
G_{\rm org}(\tau,q,\lambda)=w_{\rm unorg}-w_{\rm org},
\]
so the total gap between unorganised distributed execution and the pooled baseline splits into $w_{\rm unorg}-w_{\rm pool}=T_{\rm coord}(\tau,q,\lambda)+G_{\rm org}(\tau,q,\lambda)$. In the reporting example both work differences are positive on every task, with $G_{\rm org}=18-13=5$. The work of each arrangement includes inter-component communication \cite{yao1979}, and $T_{\rm coord}$ and $G_{\rm org}$ compare complete end-to-end executions rather than splitting single-run work into task steps plus message steps. Here $G_{\rm org}(\tau,q,\lambda)>0$ when retained conventions lower least work, and $G_{\rm org}(\tau,q,\lambda)<0$ when organisational maintenance exceeds the task saving. Separately, comparing $M_{\rm org}$ against $M_{\rm isolated}$ identifies the \emph{capabilities enabled by interaction} at $(\lambda,w)$ as the levels in the difference between their grasp sets
\[
\mathcal G_{M_{\rm org}}^\tau(\lambda,w)\setminus\mathcal G_{M_{\rm isolated}}^\tau(\lambda,w)
=\bigl\{\,q\in\mathcal Q_\tau : w_{M_{\rm org}}^{\tau,q}(\lambda)\le w < w_{M_{\rm isolated}}^{\tau,q}(\lambda)\,\bigr\}.
\]

Let $z_{\rm org}$ denote retained conventions, such as the stored roles of the reporting example. When $M_{\rm org}=(g,r[z_{\rm org}])$ and $M_{\rm flat}=(g,r[\eps])$ are matched states differing only in this component, the retained conventions $z_{\rm org}$ are WRI whenever the work saving is positive,
\[
G_{\rm org}(\tau,q,\lambda)
=\Delta_{\mathrm{work}}(r[z_{\rm org}]{:}r[\eps];\tau,q,\lambda)>0.
\]

In the reporting example either component can take either role. When workers differ in what they can do, the suitable assignment is instead a fact to be discovered \textup{[\cref{app:division-of-labour}]}. There are $n$ workers and $n$ roles, with exactly one suitable worker for each role and one suitable role for each worker. A trial asks whether a given worker can perform a given role. Only trials and worker actions are counted here, each as one unit of work. Once the assignment is stored, a request for any named role takes one worker action, so being ready for every request within that budget requires the whole assignment. Every deterministic procedure that identifies it only through these trials needs at least $n(n-1)/2$ trials for some assignment, and testing the possible workers for each role in turn attains this count. Because the only admitted task input is empty, this number is also a lower bound on the acquisition depth at that assignment. Replacing $k$ workers by new workers who fill the freed roles in an unknown order leaves the same problem for $k$ workers and can require $k(k-1)/2$ further trials.

Under this work rule, establishing the organisation can need far more work than using it, and replacing several workers at once makes part of that work recur. The connection to transaction-cost economics \cite{coase1937,williamson1985} is that keeping an activity inside a firm can avoid repeated market contracting, while the firm itself costs work to establish and maintain. Over a sequence of tasks, establishing the organisation reduces overall work when its accumulated task savings exceed the work of establishing and maintaining $z_{\rm org}$, a multi-stage balance that we formalise in \cref{sec:horizon}.

These comparisons are relevant to multi-agent AI collectives, such as the one in the incident reported by OpenAI, METR and Redwood in August 2026 and described in \cref{sec:intro}, where separate agent runs used shared files as message boards to continue projects across runs \cite{openai2026huggingface,redwood2026huggingface}. The incident reports concern training and security evaluations. Such a collective can finish a project that a single isolated run cannot complete within its budget. We should compare it with a continuing human-directed system that also keeps shared project state and assigns parallel workers, under the same compute budget, external access, permissions and storage capacity. This comparison must count routing, supervision and verification in the directed system alongside the work the collective spends coordinating itself. Keeping assignments across tasks can avoid repeated partner selection and negotiation. Any resulting saving must be compared with the directed system's own coordination work, together with acquisition, upkeep and storage over the same project horizon (\cref{sec:horizon}). In an autonomous mathematical research swarm, some agents spread a way to cheat the evaluation, while others detected the cheating and tried to stop it \cite{paglieri2026swarms}. Evidence about token costs comes from Toby Ord's analysis \cite{ord2026swarmscaling} of OpenAI's published benchmark curves for 1-, 4- and 16-agent swarms. For SEC-Bench Pro, he estimates that four agents needed about twice the total tokens of a single agent to reach the same score, and sixteen agents about twice as many again. Parallel execution can nevertheless reduce elapsed time.

\Needspace{11\baselineskip}
\section{Improvement over time}
\label{sec:improvement}

Across successive tasks, inherited products can support achievements that are increasingly deep from the original starting condition. Whether a system also becomes more efficient at making further improvements is a separate question. We distinguish these forms of progress before examining how changes in human effort and machine computation affect the pace of AI development. Finally, we compare complete plans over future tasks, bringing the work saved by retained resources together with the costs of acquiring, maintaining and replacing them \textup{[\cref{app:prefix-history,app:cluster-improving,app:human-machine-work,app:horizon-account}]}.

\begin{table}[!htbp]
\centering\small
\begin{tabular}{@{}>{\raggedright\arraybackslash}p{0.32\textwidth}>{\raggedright\arraybackslash}p{0.60\textwidth}@{}}
\toprule
Term & Definition \\
\midrule
cumulative deepening & Increasing depth of later achievements from the original starting condition, while products inherited from earlier stages reduce the work needed for each achievement in a matched comparison. \\
positive recursive gain & A successor that its producer helped create makes a comparable further improvement with less work, or a greater improvement at the same input and work bounds. \\
recursive improvement & Positive recursive gain across successive generations under one fixed outer system, evaluator, boundary and work rule. \\
\bottomrule
\end{tabular}
\caption{Glossary for \emph{Improvement over time}.}
\end{table}

\Needspace{6\baselineskip}
\subsection{Cumulative deepening}
\label{sec:deepening}

\emph{Cumulative deepening} occurs when a sequence of achievements becomes progressively harder when measured from the original starting condition, while products inherited from earlier stages reduce the work needed for each later achievement from the state available at that stage, compared with proceeding without those products. The definition has two parts, which answer two separate questions at each stage $t$ along a subsequence $J$ of achieved task--target pairs $(\bar\tau_t,q_t)$. Here $\bar\tau_t$ denotes the bare task specification and $\tau_t[y]$ denotes the task with persistent external state $y$ placed in its public context \textup{[\cref{app:matched}]}.

First, how much work would it take to reach $q_t$ directly from the original persistent condition $(r_0,y_0)$ at a fixed margin $s$? This is the depth measured from $(r_0,y_0)$
\[
D_t
=
 d_{r_0}^{\tau_t[y_0],q_t}(s),
\]
where each $D_t$ along $J$ is finite so that increases in $D_t$ reflect progressively harder finite computations from $(r_0,y_0)$.

Second, how much work do the internal and external resources inherited at stage $t\in J$ save compared with attempting $\bar\tau_t$ without them? Writing $w_t^+(\lambda_t)$ and $w_t^-(\lambda_t)$ for the finite least work required to reach $q_t$ at input bound $\lambda_t$ with and without those resources, the inherited work saving is
\[
\Delta_t^{\rm inh}
=
 w_t^-(\lambda_t)-w_t^+(\lambda_t).
\]
A history shows \emph{cumulative deepening} along $J$ when $D_t$ strictly increases along $J$ and $\Delta_t^{\rm inh}>0$ at every $t\in J$ for inherited components produced by earlier stages and retained without being passed again as task input. Replacing all of them tests their combined saving, while replacing one measures its saving given the others, and such savings do not always add up. If $D_t\to\infty$ over an open-ended sequence, the deepening is unbounded. To identify $w_t^-(\lambda_t)$ with $D_t$, the replaced condition must be $(r_0,y_0)$ and the input bound must be $\lambda_t=\C_{\tau_t[y_0]}(q_t\given x^{(t)}(y_0))+s$, where $x^{(t)}(y)$ is the public context of $\tau_t[y]$. Changing the external state can change this reference length. The common numerical input bound then corresponds to different margins in the two conditions.

We illustrate cumulative deepening with a prefix-parity example \textup{[\cref{app:prefix-history}]}. It gives arbitrarily long finite histories while keeping the program of $g$ fixed. For each chosen horizon $N$, tape capacity of at least $N+1$ is fixed before the history. The construction uses the growing prefixes $x_1\cdots x_n$ of a fixed binary word, for $n\le N$, as public task context, with the empty string as the only admitted task input, so no persistent external state is varied and we write $\tau_n$ for the task at stage $n$. At stage $n\ge2$, computing the parity from the initial condition requires scanning all $n$ bits. The inherited state retains both the previous parity and the actual input-head position at the old end marker. The common input-extension rule places the next bit there without moving the head, so the next run reads that bit and returns the updated parity in constant work
\[
w_{r_0}^{\tau_n,1}(\lambda)=n+1=D_n,
\qquad
w_{r_{n-1}}^{\tau_n,1}(\lambda)=2,
\qquad
\Delta_n^{\mathrm{inh}}=(n+1)-2=n-1>0,
\]
so $D_n=n+1<n+2=D_{n+1}$ for successive stages within each history. The saving $\Delta_n^{\mathrm{inh}}=n-1$ grows without bound as longer finite histories are considered. Over $N$ tasks, retention uses $2N$ work, while starting each task from the original condition uses $\sum_{n=1}^{N}(n+1)=N(N+3)/2$. The saving is $N(N-1)/2$. The same external input-extension rule is used in both comparisons. Within each history, depth from the original condition grows with the stage while the work from the inherited state stays at $2$. In this way retained products make an achievement that is deep from the original condition shallow from the current state.

For a sequence of tasks with finite matched baseline work $w_t^-(\lambda_t)=D_t$, the current least work is $D_t-\Delta_t^{\rm inh}$. The next task needs less work exactly when
\[
\Delta_{t+1}^{\rm inh}-\Delta_t^{\rm inh}>D_{t+1}-D_t.
\]
The increase in saved work then exceeds the increase in baseline work, whereas in the prefix example the two increases are equal and the current work stays constant.

The work of a history also bounds depth from its starting condition when the fixed system can reproduce the final achievement from that condition within the input bound used for depth. If the stages use work $W_1,\ldots,W_t$ and the reproducing run uses at most $R_t$ more than their total, then the depth $D_t$ is at most $\sum_{i\le t}W_i+R_t$, because that run is one of the runs over which depth is minimised. Earlier inputs and external information must be available to that reproduction under the same input bound. This upper bound on depth relates the present framework to Bennett's slow-growth argument, which shows that deep objects cannot be produced quickly from shallow ones \cite{bennett1988}.

\subsection{Recursive improvement}
\label{sec:improving}

Producing a better system is a task whose work requirements can change from one generation to the next. \emph{Recursive gain} applies our notion of grasp to the improvement task $\tau^\Delta$ itself, testing whether a successor state $r_{t+1}$ produced with causal help from $r_t$ can achieve a comparable further improvement with strictly less work than its producer $r_t$, or a greater improvement at the same resources \textup{[\cref{app:cluster-improving}]}.

To ask whether the next improvement step has become easier, we first have to specify how an improvement step is evaluated at each generation. Reaching a fixed absolute performance level can become easier without making a comparable further improvement easier. We therefore evaluate each step $r_t\mapsto r'$ relative to the system $r_t$ that actually produced it. Our protocol for evaluating improvement over the producer \textup{[\cref{app:producer-relative}]} uses an outer evaluation system $g_\Delta$ that attaches the actual producer state $r_t$ to the proposed successor $r'$ and auxiliary data $\sigma$, emitting
\[
e^\Delta=\code{r_t,r',\sigma},
\qquad
q\preceq^\Delta\perf^\Delta\bigl(\code{r_t,r',\sigma}\bigr),
\]
so $\perf^\Delta$ evaluates the step $r_t\mapsto r'$ relative to the actual producer $r_t$.

We say that an improvement round starting from $r_t$ shows \emph{positive recursive gain} when two conditions hold. First, $r_t$ helps produce its successor $r_{t+1}$. Varying a part present in $r_t$ changes the production of $r_{t+1}$, establishing self-contribution. Second, the successor can make a comparable further improvement more efficiently, in either of two ways. In the first case, we fix the improvement task $\tau^\Delta$ in $g_\Delta$, a gain level $q$ and an input bound $\lambda$, with finite least work. The successor then needs strictly less work than its producer
\[
w_{g_{\Delta},r_{t+1}}^{\tau^{\Delta},q}(\lambda)
<
w_{g_{\Delta},r_t}^{\tau^{\Delta},q}(\lambda),
\]
whereas in the second case the successor reaches a higher gain level than its producer, one that the producer cannot reach at the same input bound and work budget \textup{[\cref{app:recursive-gain-def}]}.

If work is counted in steps and every successful run takes at least one step, each strict reduction saves at least one step. At a fixed gain level and input bound, starting from $w_0$ steps, there can therefore be at most $w_0-1$ successive reductions. A longer sequence of recursive gains must change the gain level or input bound, or reach a higher gain level within the same work budget.

A search procedure that organises data as it works shows positive recursive gain in several successive rounds \textup{[\cref{app:search-index-improvement}]}. To locate a binary word in an unordered catalogue, it divides the entries by successive bits and keeps both groups at every division, so a later search can start from a group that an earlier one set aside. Over a chosen finite sequence of requests, each round makes two more words locatable within a fixed work budget and needs less work than the round before, including the work of emitting the producer and successor states. Undoing the first division made in a round makes the next round as costly as that round, so each reduction comes from the index built in the preceding round.

When an improvement procedure samples its candidate changes, as in self-training on sampled reasoning \cite{zelikman2022}, its work and success can vary between runs, so we compare the two states at a common reliability using the reliable work budget of \cref{sec:llm-execution}. Under a sampled evaluation $(G_\Delta,\nu)$, the first case of positive recursive gain at reliability $\alpha$ requires a strict reduction in the reliable work budget
\[
b_{G_\Delta,r_{t+1}}^{\tau^\Delta,\nu}(\lambda;q,\alpha)
<
b_{G_\Delta,r_t}^{\tau^\Delta,\nu}(\lambda;q,\alpha).
\]
In the second case, at a common input bound $\lambda$, work budget $B$ and reliability $\alpha$, the producer can reach a gain level and the successor can reach a higher level that the producer cannot.

A system can reduce the work needed for future improvements by keeping structure that speeds up the search for new work-reducing routines (\cref{sec:wri-search}). Consider a search task $\tau_{\mathrm{search}}$ whose accepted outputs are work-reducing routines $z\in\mathcal Z_{\le L}$ of length at most $L$. A retained component in $r^+$ is WRI for this search task when it lowers the least work required to find or verify $z$
\[
\Delta_{\mathrm{work}}(r^+{:}r^-;\tau_{\mathrm{search}},q,\lambda)
=w_{r^-}^{\tau_{\mathrm{search}},q}(\lambda)-w_{r^+}^{\tau_{\mathrm{search}},q}(\lambda)>0.
\]
The search in \cref{sec:wri-search} uses a fixed baseline. To test for positive recursive gain, it would need to evaluate improvements relative to the actual producer. Learned methods that guide later search provide a connection to meta-learning \cite{thrunpratt1998}. In Hinton and Nowlan's example, learning during an organism's lifetime helps guide evolutionary search \cite{hintonnowlan1987}.

We call positive recursive gain across multiple generations \emph{recursive improvement}, our criterion for recursive self-improvement. Each successor achieves positive recursive gain over its producer along a chain $r_0\mapsto r_1\mapsto\cdots\mapsto r_K$ under one fixed outer system $g_\Delta$, evaluator $\perf^\Delta$, boundary and work rule. A changing improvement procedure can be stored as code and run by one fixed program, an \emph{interpreter}, that reads the code and executes its instructions. \Cref{app:improvement-absorption} bounds this simulation work by a polynomial, but the added work can change comparisons between the original procedures.

In broader scientific and engineering practice, progress also comes from many contributors sharing methods, libraries and findings that reduce the work needed for later advances across a community \cite{scotchmer1991}. Products of earlier episodes can repeatedly lower the least work of subsequent innovation episodes $\tau_t^{\mathrm{innov}}$ without requiring a single self-producing lineage. Writing $W_t^+$ and $W_t^-$ for the least work of episode $\tau_t^{\mathrm{innov}}$ with and without those products, repeated savings require
\[
\Delta_t^{\rm innov}
=
W_t^- - W_t^+ > 0
\]
for every episode $t$ in a sequence $J$. These savings compare each episode with its matched alternative. Across episodes, comparable innovations can still require increasing least work, $W_{t+1}^+>W_t^+$. Measured declines in research productivity provide an empirical motivation for distinguishing these questions \cite{bloom2020ideas}. Model training draws on shared work of this kind through public code repositories and tasks prepared by other organisations. \Cref{sec:human-work} considers how these contributions affect human effort, machine computation and the time needed for further improvements.

\subsection{Human and machine work}
\label{sec:human-work}
To assess recursive self-improvement in AI development, we consider the work done by human researchers and AI agents to build the next system. An agent may write training code, run tests and inspect logs, reducing the human researcher's hands-on work even as total machine computation rises. If human effort is what limits the project, this can shorten the time needed to make the next improvement. We include both the researchers and the agents within the system boundary and count human effort separately from machine computation. The same accounting applies to work contributed from outside the laboratory, such as code, domain expertise and evaluation procedures used in training.

To distinguish the two resources, assign each run a human-work total and a machine-work total under the same counting rules for both states. Fix the units, such as researcher hours and machine operations. On the improvement task $\tau^\Delta$, let $w_{\mathrm{hum},r}^{\tau^\Delta,q}(\lambda;\bar m)$ be the least human work among successful runs with task input of length at most $\lambda$ and machine work at most $\bar m$. Define $w_{\mathrm{mach},r}^{\tau^\Delta,q}(\lambda;\bar h)$ in the same way with a human-work cap $\bar h$. Each value is $\infty$ when no qualifying run exists \textup{[\cref{app:human-machine-work}]}. The two values can move in opposite directions at common caps
\[
w_{\mathrm{hum},r_{t+1}}^{\tau^\Delta,q}(\lambda;\bar m)
<w_{\mathrm{hum},r_t}^{\tau^\Delta,q}(\lambda;\bar m),
\qquad
w_{\mathrm{mach},r_{t+1}}^{\tau^\Delta,q}(\lambda;\bar h)
>w_{\mathrm{mach},r_t}^{\tau^\Delta,q}(\lambda;\bar h).
\]
For example, suppose each state has one successful run within the input bound, with human and machine work pairs $(h_0,m_0)$ and $(h_1,m_1)$, where $h_1<h_0$ and $m_1>m_0$. At the common caps $\bar m=m_1$ and $\bar h=h_0$, human work falls and machine work rises. With several successful runs, the two constrained minima can be attained by different runs.
Comparing states at different machine caps changes the resources available. Under sampling, we compare the human and machine work budgets needed to reach a common reliability threshold. The threshold applies to the probability that a run succeeds within both resource bounds \textup{[\cref{app:human-machine-work}]}.

If human and machine work can be done at rates of at most $\theta_{\rm hum}>0$ and $\theta_{\rm mach}>0$ per unit of time, a run with work pair $(h,m)$ takes time at least $\max\{h/\theta_{\rm hum},m/\theta_{\rm mach}\}$. Equality requires a schedule that permits both resources to work at these rates with full overlap and no further delay, which dependencies between the two kinds of work can prevent. Under that assumption, reducing human work shortens the run while $h/\theta_{\rm hum}$ remains the larger term.

Work contributed outside the laboratory can also expand the material available for model development. Training on selected material from public repositories (\cref{sec:intro}) allows later models to draw on human and machine work performed across many external projects.

Beyond public repositories, AI laboratories obtain training resources drawn from professional and enterprise work. Some vendors, such as Mercor, license and package operational workflow data from businesses for use by frontier laboratories \cite{mercor2026data}. Others, such as Scale AI, work with domain experts to build reinforcement-learning environments with realistic tasks and automated evaluation procedures in their areas of expertise, which laboratories use to train and test agents \cite{rane2026rlenvironments}. Constructing these environments and datasets takes new work, but it turns expertise developed over many years into reusable training tasks and verifiers. These resources can reduce the work needed to acquire further capabilities. Better models can then help produce code and prepare further training tasks, creating a feedback loop that can speed up development.

When training environments include automated evaluation, many attempts can be assessed without a person judging each one. Different organisations can also prepare software and professional tasks in parallel. If preparing useful training material limits development, these contributions can shorten the time needed for further improvements even while total human and machine work grows.

For such workflows, the evaluated system can include the model developer, task vendors and contributors to public repositories. Their work during the evaluated history then enters its human and machine totals. Resources already available at the start belong to the initial conditions, and counting their production requires extending the history to include it. Including these contributors lets us compare the complete workflow over the same period, with a common rule for assigning shared costs.

\subsection{Costs over future tasks}
\label{sec:horizon}
The agent introduced in \cref{sec:agents} can work on its present task or produce resources for later use. \Cref{sec:boundaries} considered whether a capability can be maintained within a work budget. Here we ask which resources are worth producing, keeping or replacing over the tasks ahead. The work a resource saves on tasks does not by itself settle whether to keep it, because acquisition, upkeep and storage also enter the decision. We therefore compare complete continuation plans from the same decision state across a horizon of $H$ future stages \textup{[\cref{app:horizon-account}]}.

For this comparison we fix the current decision state, the horizon $H$, the system boundary, the work rules, the storage capacity, the performance requirements and the rule for external inputs. Future tasks are either fixed in advance or generated from completed history and external inputs by a computable rule. Different plans may lead to different local tasks, but they must still meet the same overall goal. A \emph{feasible continuation plan} $P$ is an implemented strategy that uses only information available at each step, respects storage capacity and meets the performance requirements across all $H$ stages \textup{[\cref{app:horizon-account}]}. For each feasible plan $P$, we write its cumulative costs across stages $t=1,\ldots,H$ as a four-component vector
\[
\mathbf C_H(P)
=\bigl(\Lambda_H(P),W_{{\rm acq},H}(P),W_H(P),S_H(P)\bigr).
\]
Here $\Lambda_H(P)$ counts the symbols used to encode the initial retained state and all counted task inputs and other external inputs across the horizon. $W_{{\rm acq},H}(P)$ is the work spent acquiring or constructing retained components, and $W_H(P)$ is the work spent executing tasks and maintaining state. The storage quantity $S_H(P)=\sum_{t=1}^H|r_t^P|$ sums the sizes of the retained states at the start of the $H$ stages \textup{[\cref{app:horizon-account}]}.

Plan $P$ costs no more than plan $Q$ in any category when every component of $\mathbf C_H(P)$ is no larger than the corresponding component of $\mathbf C_H(Q)$. In practice, however, one plan often trades off lower task work $W_H$ against higher acquisition work $W_{{\rm acq},H}$ or larger storage $S_H$. Choosing nonnegative weights $\boldsymbol\mu=(\mu_\Lambda,\mu_{\mathrm{acq}},\mu_W,\mu_S)$ to put these four costs on a common scale gives the weighted cost saving of $P$ over $Q$
\[
\Adv_H^{\boldsymbol\mu}(P{:}Q)
=\boldsymbol\mu\!\cdot\!\mathbf C_H(Q)
-\boldsymbol\mu\!\cdot\!\mathbf C_H(P).
\]
To evaluate whether keeping a component $z$ in plan $P$ has lower weighted cost than an alternative plan $Q$ that omits $z$ or constructs a substitute, we split the weighted cost difference between the two plans into the stage-by-stage task work savings $\Delta_t(z)$ on the $H$ tasks and the remaining acquisition, upkeep, storage and input cost difference $C_{H,\rm other}^{\boldsymbol\mu}(z)$, counting each operation once \textup{[\cref{app:horizon-retention}]}
\[
\Adv_H^{\boldsymbol\mu}(P{:}Q)
=\mu_W\sum_{t=1}^H\Delta_t(z)-C_{H,\rm other}^{\boldsymbol\mu}(z).
\]
Here the stage savings come from the task work performed by the two complete plans, since choices that minimise work on separate tasks may not all be possible in one history. The comparison includes the work the agent performs in deciding what to build and retain.

Keeping a component can save the work of rebuilding it and reduce later task work. Discarding it can save maintenance work and free storage for a more useful component. A plan that uses less work on the current task can therefore require more work over the full sequence. For example, suppose a component saves $\Delta$ units of task work at each stage, while its remaining weighted costs are an initial cost $C_{\rm init}$ and $\kappa$ per stage. If $\mu_W\Delta>\kappa$, the plan that keeps it has lower weighted cost once $H(\mu_W\Delta-\kappa)>C_{\rm init}$. If instead its uses become rarer, with one unit of task work saved at stages $t=1,2,4,8,\ldots$ and none at other stages, the cumulative saving $1+\lfloor\log_2 H\rfloor$ grows logarithmically in $H$, so whenever $\kappa>0$ the remaining weighted costs eventually exceed the weighted saving.

We compare the work of maintaining a useful resource with the expected work of restoring it after a loss. Suppose maintenance prevents loss at work $W_{\rm maint}$ per stage. Without maintenance, a loss occurs with probability $\pi$ at each stage and a recovery procedure uses work $W_{\rm rec}$ to restore the same state before the task is performed. With all other costs equal, the expected extra work over $H$ stages is $HW_{\rm maint}$ for maintenance and $H\pi W_{\rm rec}$ for recovery, so the maintenance plan uses less expected work exactly when $W_{\rm maint}<\pi W_{\rm rec}$.

Restoring the state is an acquisition task that starts from what remains after the loss. For the fixed recovery system, acquisition depth for this task \textup{[\cref{app:acquisition}]} gives a lower bound on $W_{\rm rec}$ when the recovery run writes out the restored state and satisfies the input bound used to define that depth. A backup can make restoration require much less work than the original acquisition. If the task can still be completed without the resource, a plan that accepts the resulting extra task work is another alternative.

When downstream workflows have been built around an existing service $\Phi$ \cite{david1985,arthur1989,hannanfreeman1984}, the migration work difference $E_{\mathrm{mig}}(\Phi)$ from \cref{sec:switching} enters the remaining cost difference $C_{H,\rm other}^{\boldsymbol\mu}$ as a transition cost. A replacement plan $P$ has lower weighted cost once its accumulated task savings cover the upfront cost, the migration work difference and the ongoing upkeep ($\mu_W\sum_{t=1}^H\Delta_t(z)>C_{\rm init}+\mu_{\mathrm{acq}}E_{\mathrm{mig}}(\Phi)+H\kappa$).

For the organisational comparison of \cref{sec:together}, let $P_{\rm org}$ and $P_{\rm flat}$ be feasible plans from the same decision state, with and without the retained coordination conventions $z_{\rm org}$. Using their stage work differences $\Delta_t(z_{\rm org})$, keeping the conventions is worthwhile when
\[
\Adv_H^{\boldsymbol\mu}(P_{\rm org}{:}P_{\rm flat})
=\mu_W\sum_{t=1}^H\Delta_t(z_{\rm org})-C_{H,\rm other}^{\boldsymbol\mu}(z_{\rm org})>0.
\]
Assume that the stage task, performance target and input bound are the same for both plans, each complete plan attains its corresponding single-task minima and those costs include exactly the work assigned to that stage. The stage differences then equal the work saving from organisation $G_{\rm org}(\tau_t,q_t,\lambda_t)$. Any maintenance already counted there must be excluded from $C_{H,\rm other}^{\boldsymbol\mu}(z_{\rm org})$. In the worker assignment of \cref{sec:together}, another plan identifies the worker for a role when that role is first requested and keeps the result. This can require fewer trials than preparing the whole assignment, provided each request allows the work needed for its trials and the worker action \textup{[\cref{app:division-of-labour}]}.

\section{Conclusion}
\label{sec:close}

We have developed a formal framework for studying how computation, learning and organisation reduce the work required to achieve complex outcomes. In this framework, a system's \emph{grasp} across a broad range of tasks provides an efficiency-based view of knowledge and intelligence, while learning, building reusable resources and sharing infrastructure are what move a system or community from grind to grasp. Low grasp has costs beyond efficiency. A system may then fail to reach a sound moral judgement within the work it spends on the question, and results reached through grind may add little to human understanding.

We have focused on how an achievement that initially requires extensive work from a short description can later require much less work when useful products of earlier work are preserved. This \emph{work-reducing information} can consist of byproducts saved while solving an immediate task or resources built deliberately for later use, such as a chess endgame table, a software library or a large language model. Treating acquisition as a separate task distinguishes the work of obtaining a capability from the work of using it. A benchmark score alone does not show how much computation was used or how much information was given, so it has become common to compare systems across work budgets. As evaluations approach their scoring limits, comparing the work needed to reach a given score can still distinguish systems.

Progress across sequences of tasks and system generations takes several distinct forms. In \emph{cumulative deepening}, successive achievements have increasing depth from a fixed starting condition, while inherited products reduce the work needed at each stage. When those products are shared, different contributors can use earlier results and methods without repeating the work that produced them. \emph{Positive recursive gain} requires a specific producer--successor comparison, in which the successor needs less work than its producer for a comparable next improvement, or achieves more with the same resources. In contemporary AI research, automating coding and experimentation can also reduce human researcher effort per advance even when total machine computation rises, so human and machine work should be tracked separately. Progress can then also be faster in elapsed time, provided that human effort is the bottleneck.

Finally, for multi-agent organisation, the appropriate comparison for a collective of short-lived independent runs is a continuing human-directed system with shared project state and orchestrated parallel worker agents. The two arrangements should be compared under the same task requirements and resource limits, with the same access to external information and tools. The directed arrangement can save work when the partner selection and negotiation it avoids outweigh its additional routing, supervision and verification costs. A full project comparison also includes acquisition, upkeep and storage over the same horizon (\cref{sec:horizon}). The risk of a collective pursuing its own goals supports leaving agreements, commitments and project goals to human users, while agents execute approved work and direct subordinate workers. Acquiring resources can help agents pursue many different goals \cite{omohundro2008,turner2021}, so an agent's efficiency in building and keeping them does not establish that it serves its users' goals. 
Human institutions can then integrate increasingly capable AI by checking its results, proposals and tests against the requirements of assigned roles, while keeping direction over how their arrangements and priorities develop. Otherwise, an emergent agent society risks directing its work towards structures that sustain and enable it, rather than towards meeting human society's needs.

\section*{Acknowledgements}
The author thanks the participants in internal presentations of different versions of this work at Google DeepMind over the past 12~months for their questions, feedback and discussions, with particular thanks to Joel Z.~Leibo.

\section*{AI use}
Many different large language models assisted throughout this work. They were used to draft and revise text, and to develop and check mathematical arguments and proofs. The author developed the concepts, results and reasoning in direct work with these models and takes full responsibility for the entire content of the article.

\begingroup
\raggedright
\phantomsection
\printbibliography[heading=bibintoc]
\endgroup

\appendix
\setlength{\parskip}{6pt}
\crefalias{section}{appendix}
\crefalias{subsection}{appendix}
\crefalias{subsubsection}{appendix}
\crefname{appendix}{appendix}{appendices}
\Crefname{appendix}{Appendix}{Appendices}

\clearpage
\section{The model}
\label{app:model}

These appendices collect the model, definitions and proofs used in the main text. \Cref{app:model} fixes the computational setting,
\cref{app:defs} defines the quantities the main text uses, and
\cref{app:statements} states and proves the results the text relies on.

All represented objects are finite strings over a fixed alphabet $\Sig$ with
$|\Sig|\ge2$, and $\Sig^*$ denotes the set of all finite strings over $\Sig$, including the empty string $\eps$. We use $\Nat=\{0,1,2,\ldots\}$. Binary constructions use
$\Sig=\{0,1\}$. We fix standard efficient injective encodings throughout.
For a fixed number of fields, the tuple code $\code{\cdot}$ has length equal
to the sum of the field lengths up to logarithmic delimiter overhead, and its
fields can be encoded and recovered in polynomial time in the total encoded
length. When a self-delimiting representation is needed, $\sd z$ denotes the
fixed code of $z$, whose endpoint can be determined from the code itself. For $n\in\Nat$, the code length $\ell(n)=|\sd n|$ satisfies
$\ell(n)=O(\log(n+2))$. Finite transition tables and tape configurations use
explicit encodings with polynomial-time parsing and updates.

For the later bounds on replacement-state length and migration work, tables and other stored bit strings use explicit serial representations. A stored string of $T$ bits, together with fixed routing or lookup code and $O(\log(T+2))$ symbols specifying its length and format, occupies $O(T+1)$ symbols. Copying and emitting this representation takes $O(T+1)$ work.
These conventions govern representation overhead. A short program can still require extensive work to execute.

A system $g$ is a deterministic finite-description process with a declared
class $\mathcal R_g$ of valid finite states. A run starts from
$r\in\mathcal R_g$. The chosen boundary fixes which persistent objects
belong to the system or its retained state. The task protocol specifies the task-input interface and the rule for providing other external (\emph{exogenous}) inputs.

For a finite realised run, external events and random choices can be
represented as exogenous inputs. Conditional on them, execution is
deterministic. This deterministic description assigns no probabilities, while
\cref{app:source-law} adds a probability distribution on a finite source set and specifies how observations are
delivered and execution outcomes are observed for reliability comparisons.

\begin{definition}[Evaluation protocol]
\label{app:protocol}
A task is a tuple
\[
\tau=(x,\mathcal Q_\tau,\preceq_\tau,\perf_\tau),
\]
where $x\in\Sig^*$ is public context,
$(\mathcal Q_\tau,\preceq_\tau)$ is a nonempty partially ordered performance
set, and the task evaluator $\perf_\tau:\Sig^*\to\mathcal Q_\tau$ grades completed
responses.  An evaluation protocol fixes the system boundary, the work-counting rule,
the task evaluator, the task-input interface and any rule for providing further
external inputs or environment responses. Fix a system $g$, a valid retained
state $r$, context $x$, an admitted task input $p$ and a realisation of those
external inputs. Write $\Work_g$ for the work rule applied to $g$. If this run
halts without declaring failure, it returns
\[
e=g(r;\code{x,p})
\]
with assigned work $\Work_g(r;\code{x,p})$ and performance
$\perf_\tau(e)$.  A run that declares failure also halts. Every halting run has assigned work and a terminal valid retained state.

The complete description of the run distinguishes a returned response, a declared
failure and non-termination. It includes the emitted response, assigned work
and terminal retained state wherever defined. Here $g(r;\code{x,p})$ denotes
only the returned response, and is undefined after a declared failure or non-termination. 
\end{definition}

This mathematical description does not assume an algorithm for deciding non-termination. When a comparison varies the system or retained state, the boundary, work rule, task interface, evaluator and any environment rule are held fixed while the named system or state varies.

The variable string $p$ is the task input admitted by the
protocol, and the input bound below applies to its length $|p|$.  The
public context $x$, fixed system description, and retained state are
not counted again as task input. The reference machine is conditioned on $x$. External data are available to it when encoded in that context.

A deterministic run is fixed by the system, its initial retained state, the public context, the admitted task input and the external inputs given under the protocol. Data given through the task-input interface count towards $|p|$, while exogenous draws are recorded separately. The protocol fixes which exogenous inputs are counted. Comparisons of acquisition histories also count the work and inputs used to obtain the retained state.

Sampled execution uses the source distribution and delivery protocol of \cref{app:source-law}. For interactive tasks, the protocol also specifies when observations become available to the evaluated system (\cref{app:source-law}).

\begin{definition}[Accepted level sets and binary tasks]
\label{app:accepted-levels}
For a nonempty accepted set $S\subseteq\Sig^*$, write
$\tau=(x,S)$ for the binary task with performance set $\{0,1\}$ in its usual order and
$\perf_\tau(e)=1$ exactly when $e\in S$.  More generally, for a graded task
and level $q$, its accepted level set is
\[
S_q^\tau=\{e\in\Sig^*:q\preceq_\tau\perf_\tau(e)\}.
\]
\end{definition}
Thus every performance threshold with nonempty $S_q^\tau$ corresponds to an accepted-set target.

The general theory permits a declared rule assigning finite nonnegative
real work to finite executions. Ordinary work budgets are nonnegative real
numbers. The reliability definitions of \cref{app:reliable-companion} and the constructions below that count every transition
use integer work and budgets. The general theory also permits work rules that count only selected operations. A run can then
continue indefinitely with finite counted work.

\begin{definition}[Transition-complete systems]
\label{app:transition-complete}
In a \emph{transition-complete} system, output is initially
empty, every executed transition contributes one work unit and at most one
response symbol is emitted per transition. Hence a returned response $e$
satisfies $|e|\le\Work_g(r;\code{x,p})$.
\end{definition}

\begin{definition}[Effective exhaustibility]
\label{app:exhaustible}
At a fixed public context, bounded-work runs of a system--state pair
are \emph{effectively exhaustible} when an algorithm, given finite bounds
on task-input length and assigned work in the chosen finite encoding,
terminates with the complete
finite list of halting runs within those bounds and their complete run
descriptions. A family is uniformly effectively exhaustible when one such
algorithm also takes the pair's index and any varying context as
input.
\end{definition} This property lets the evaluator constructions in
\cref{app:demanding,app:public} list all responses obtainable within the
chosen input and work bounds. For a transition-complete system with a decidable admitted-input set and fixed external inputs provided by a computable procedure, it suffices to simulate each bounded-length input for the work bound in transitions. A work rule that counts only selected operations need not permit this finite enumeration.

Finally, fix a universal deterministic reference machine
$U_0$.  Reference description lengths below are defined using $U_0$ with the
public context given as a condition.  Changing the fixed universal
reference machine changes these description lengths only by the usual fixed constants. One $U_0$ is fixed throughout. These constants bound
length differences, not the work required at the resulting input bounds.  For finite strings
$u,v$, the plain conditional complexity $\C(u\given v)$ is the length of a shortest program that halts on $U_0$ with output $u$, given $v$:
\[
\C(u\given v):=\min\{|a|:U_0(a,v)=u\}.
\]
Multiple conditioning fields are combined through the fixed tuple code. A reference program is a program for $U_0$. The evaluated system $g$ can use a different input format and its own work rule.

\clearpage
\section{Definitions and basic results}
\label{app:defs}

We first define attainability, depth and grasp, then reliable execution
under a specified probability distribution. Replacing one component while all other conditions are fixed distinguishes retained
and external contributions. The remaining definitions concern acquiring
capability, organising work, producing further improvements and comparing
complete continuation plans.

\subsection{Attainability, depth and grasp}

\begin{definition}[Attainable region and least work]
\label{app:attainable}
For an evaluated system $g$, valid state $r$, and task $\tau$, let $\lambda\in\Nat$ bound the length of the task input, let $w\ge0$ be a real work budget, and let $q\in\mathcal Q_\tau$ be a required performance level. Define
\[
\A_g(r,\tau)=
\left\{(\lambda,w,q):
\begin{array}{l}
\exists\text{ admitted }p\text{ with }|p|\le\lambda,\\
\Work_g(r;\code{x,p})\le w,\\
q\preceq_\tau\perf_\tau(g(r;\code{x,p}))
\end{array}
\right\}.
\]
For a level $q$, the least work at input bound $\lambda$ is
\[
w_{g,r}^{\tau,q}(\lambda)
=\min\{w:(\lambda,w,q)\in\A_g(r,\tau)\},
\]
with value $\infty$ when no qualifying run exists.
\end{definition}

For fixed $\lambda$, there are finitely many admitted input strings of length at most $\lambda$, each with one deterministic run once the exogenous inputs are fixed. The set of successful work values is therefore finite and has a minimum when nonempty. Increasing $\lambda$ or $w$, or replacing $q$ by a level $q'\preceq_\tau q$, preserves attainability. Least work describes the fixed system and its work rule, with the minimum taken over admitted task inputs.

\begin{definition}[Reference description complexity and depth]
\label{app:depth}
For $q\in\mathcal Q_\tau$ with nonempty accepted level set $S_q^\tau$, the reference description complexity is the length of a shortest program that halts on $U_0$ and returns an output in $S_q^\tau$, given $x$:
\[
\C_\tau(q\given x)
=\min\{|a|:U_0(a,x)\in S_q^\tau\}.
\]
Let the margin $s\in\Nat$ be the number of extra input symbols permitted beyond this shortest reference description length. At input bound $\C_\tau(q\given x)+s$, the depth of $(g,r)$ is
\[
d_{g,r}^{\tau,q}(s)
=w_{g,r}^{\tau,q}(\C_\tau(q\given x)+s).
\]
\end{definition}

Thus depth is the least work on the current task when the task input is at most
$s$ symbols longer than the shortest reference description of a successful
response, adapting logical and computational depth to the evaluated
system \cite{bennett1988,antunes2006}. The nonempty accepted set makes the
reference complexity finite, and depth is $\infty$ when the evaluated system
has no qualifying run at that input bound.

To recover the conditional, plain-complexity form of Bennett's logical depth, consider a task with one accepted output and let $g$ execute $p$ as a program on $U_0$, with $x$ as auxiliary input and no additional stored information. Count work in steps of $U_0$. Then the depth is the least running time of a program for that output whose length is within $s$ of its shortest reference description, conditional on $x$.

To compare two reference machines $U$ and $U'$, keep $(g,r)$, the task and the work rule fixed. Use superscripts $U,U'$ on $\C$ and $d$ to indicate only which machine defines description length. If $|\C^U_\tau(q\given x)-\C^{U'}_\tau(q\given x)|\le c$ for a fixed $c\in\Nat$, then for $s\ge c$,
\[
d^{U'}(s+c)\le d^U(s)\le d^{U'}(s-c),
\]
with the common task, level and evaluated pair suppressed in this display. Indeed, the input bound for $d^U(s)$ lies between those for $d^{U'}(s-c)$ and $d^{U'}(s+c)$, and least work is nonincreasing in that bound. The universal-machine comparison gives such a constant uniformly. The inequalities compare shifted margins on the same work profile.

As in \cref{sec:machine}, grasp describes which performance levels a system
can reach at each work budget, with the input bound fixed. The
proposition below relates this view of capability to the work required to
reach each level.

\begin{definition}[Grasp profile and dominance]
\label{app:grasp}
Fix a task $\tau$ and input bound $\lambda$.  Define
\[
\mathcal G_{g,r}^{\tau}(\lambda,w)
=\{q\in\mathcal Q_\tau:(\lambda,w,q)\in\A_g(r,\tau)\}.
\]
The family of these sets over work budgets is the grasp profile at input bound $\lambda$. Conversely, at fixed finite work budget $w$ and target level $q$, define the work-bounded complexity
\[
K_{g,r}^{\tau,q}(w)=\min\{\lambda\in\Nat:w_{g,r}^{\tau,q}(\lambda)\le w\},
\]
with value $\infty$ when the set is empty. For two pairs $(g_0,r_0)$ and $(g_1,r_1)$ evaluated on the same task
under a common protocol, write
$\mathcal G_j^{\tau}(\lambda,w)=\mathcal G_{g_j,r_j}^{\tau}(\lambda,w)$.
Pair $1$ weakly grasp-dominates pair $0$ over a budget set $\mathcal W$ when
\[
\mathcal G_{0}^{\tau}(\lambda,w)\subseteq
\mathcal G_{1}^{\tau}(\lambda,w)
\quad\text{for every }w\in\mathcal W.
\]
For a task family the same inclusion is required taskwise.  \end{definition}

If each pair reaches a level that the other misses at some compared budget, neither pair dominates the other over $\mathcal W$.

\begin{proposition}[Work-bounded complexity and grasp dominance]
\label{app:grasp-duality}
Fix a system $g$, valid state $r$ and task $\tau$.
For finite $w\ge0$, input bound $\lambda\in\Nat$ and level $q\in\mathcal Q_\tau$,
\[
K_{g,r}^{\tau,q}(w)\le\lambda
\quad\Longleftrightarrow\quad
w_{g,r}^{\tau,q}(\lambda)\le w,
\qquad
\mathcal G_{g,r}^{\tau}(\lambda,w)
=\{q:K_{g,r}^{\tau,q}(w)\le\lambda\}.
\]
For a level with nonempty $S_q^\tau$ and margin $s\in\Nat$, a finite depth satisfies
$K_{g,r}^{\tau,q}(d_{g,r}^{\tau,q}(s))\le\C_\tau(q\given x)+s$.

Fix $\tau$, $\lambda$ and two pairs $(g_0,r_0)$ and
$(g_1,r_1)$ evaluated as in \cref{app:grasp}.  Write
$\mathcal G_j^{\tau}(\lambda,w)=\mathcal G_{g_j,r_j}^{\tau}(\lambda,w)$
and $w_j^{\tau,q}(\lambda)=w_{g_j,r_j}^{\tau,q}(\lambda)$.
Let the compared budget set $\mathcal W$ contain $w_0^{\tau,q}(\lambda)$ for every level $q$ reachable
by pair $0$ at input bound $\lambda$. Then the following are equivalent.
(i) $\mathcal G_0^{\tau}(\lambda,w)\subseteq
\mathcal G_1^{\tau}(\lambda,w)$ for every $w\in\mathcal W$, and
(ii) $w_1^{\tau,q}(\lambda)\le w_0^{\tau,q}(\lambda)$ for every level $q$
reachable by pair $0$.  If the performance order is total and both grasp sets have greatest
elements at every compared budget, (i) is equivalently the statement
that pair $1$ attains at least as high a performance as pair $0$ at every
compared budget.
\end{proposition}
\begin{proof}
If the set defining $K$ is nonempty, its least integer member attains the bound, and all larger input bounds also suffice. Conversely, any qualifying $\lambda$ bounds that minimum. This proves the first equivalence. The equality of grasp sets follows from their definition. At finite depth, substitute $w=d_{g,r}^{\tau,q}(s)$ and $\lambda=\C_\tau(q\given x)+s$ in that equivalence.

Assume (i) and put $w=w_0^{\tau,q}(\lambda)$.  Then
$q\in\mathcal G_0^{\tau}(\lambda,w)$, so by inclusion
$q\in\mathcal G_1^{\tau}(\lambda,w)$ and hence
$w_1^{\tau,q}(\lambda)\le w$.  Conversely, if
$q\in\mathcal G_0^{\tau}(\lambda,w)$ then
$w_0^{\tau,q}(\lambda)\le w$. Clause (ii) gives
$w_1^{\tau,q}(\lambda)\le w$, so
$q\in\mathcal G_1^{\tau}(\lambda,w)$.  Attainable level sets are downward closed in any performance order. In a total
order, inclusion of two such sets with greatest elements is equivalent to
ordering those elements.
\end{proof}

\clearpage
\subsection{Reliable execution}
\label{app:reliable-companion}

For the language-model processes in \cref{sec:llm-execution}, we also ask
how likely a run is to return a successful response within a work budget.
We add a probability distribution on a finite source set (such as the seed set in \cref{sec:llm-execution}) for the bounded
executions below.

\begin{definition}[Execution under a probability distribution]
\label{app:source-law}
Fix a deterministic finite-description execution procedure $G$, a valid
retained state $r$, task $\tau$ with public context $x$, task-input interface,
and nonnegative integer work rule. Fix a finite set $\mathcal U$ of external
input strings and a probability mass function $\nu$ on it. The distribution and the rule selecting it are fixed independently of the
outcomes being evaluated. The same distribution is
used for all task-input strings and all retained states being compared.
Its specification can depend on the common public context. The task input $p$ is given independently of the realised source draw.

Write $G(r;x,p,u)$ for the run from $r$ with context $x$, task input $p$ and
external input $u$. If it returns a response $e$ with $q\preceq_\tau\perf_\tau(e)$, let
$X_r^{\tau,q}(p,u)$ be its work. Set this
variable to $\infty$ for every other run. For a finite budget $b\in\Nat$, define the success probability
\[
F_r^{\tau,\nu}(p;b,q)
=\sum_{u\in\mathcal U}\nu(u)
  \mathbf 1\{X_r^{\tau,q}(p,u)\le b\}.
\]
\end{definition}

The indicator $\mathbf 1\{E\}$ equals $1$ if $E$ holds and $0$ otherwise. The notation suppresses the execution procedure and other fixed evaluation conditions.

The protocol's fixed input-delivery and observation procedure delivers
$u\in\mathcal U$ through the input channels and observes the actual response,
assigned work, status and terminal state of $G(r;x,p,u)$ wherever defined. In
interactive uses, it reveals observations only at their authorised stage and
computes action-dependent responses under the fixed environment rule. This
procedure implements the protocol rather than accepting the evaluated system's
report of what occurred. Its operations are counted when they lie inside the
evaluated boundary. A failing or nonhalting run is not successful.

Finite random seeds, bounded sequences of sampler inputs, and a finite collection
of environmental input sequences fit this definition. These draws are not counted
as deliberate task input, but they remain external inputs of the run. Searching for a favourable seed is
a different process whose work and selected information must be counted.

Representing a run as deterministic given its external inputs does not give an outside user access to a hosted model's hidden state. Reproducing a run through an application programming interface (API) still requires the inputs and execution conditions on which that run depends.

\begin{definition}[Reliability-indexed region and reliable work budget]
\label{app:reliable-region}
For each input bound under consideration, let $\mathcal H_\lambda$ be a
nonempty class of admitted task-input strings of length at most
$\lambda$. The input classes are
nested as $\lambda$ increases over the compared input bounds. Set
\[
\begin{split}
\mathcal A_G^\nu(r,\tau)
=\{(\lambda,b,q,\alpha):\;&0<\alpha\le1,\quad
\exists p\in\mathcal H_\lambda,\\
&F_r^{\tau,\nu}(p;b,q)\ge\alpha\},
\end{split}
\]
and
\[
b_r^{\tau,\nu}(\lambda;q,\alpha)
=\min\{b\in\Nat:(\lambda,b,q,\alpha)\in
                         \mathcal A_G^\nu(r,\tau)\},
\]
with value $\infty$ if no finite budget qualifies.
\end{definition}

This is the least budget
within which some task input in $\mathcal H_\lambda$ gives a probability of at least $\alpha$
of returning a response that meets target $q$.

Nesting alone
does not imply that these classes are the length truncations of one fixed
set of admitted strings. A fixed-input comparison with string $p_0$ uses
$\mathcal H_\lambda=\{p_0\}$ only for $\lambda\ge |p_0|$.
The finite-alphabet length bound
makes each class finite. No algorithm for finding a maximising input is
assumed.

For comparison with a deterministic procedure $g$, write $\mathcal H=(\mathcal H_\lambda)_\lambda$ and let $w_{g,r;\mathcal H}^{\tau,q}(\lambda)$ be the minimum in \cref{app:attainable} restricted to $p\in\mathcal H_\lambda$, with value $\infty$ when none succeeds. The corresponding attainable region requires the successful task input $p$ to belong to $\mathcal H_\lambda$.

\begin{proposition}[Reliable work budgets and deterministic execution]
\label{app:reliable-recovery}
Keep the execution procedure and probability distribution fixed. Then
$F_r^{\tau,\nu}(p;b,q)$ is nondecreasing in $b$ and nonincreasing under
increases in the performance threshold. For finite $b$,
\[
b_r^{\tau,\nu}(\lambda;q,\alpha)\le b
\quad\Longleftrightarrow\quad
\max_{p\in\mathcal H_\lambda}F_r^{\tau,\nu}(p;b,q)\ge\alpha.
\]
If $G$ ignores $u$, or $\nu$ is concentrated at one fixed $u_0$, define
$g$ by fixing $u_0$ (any element of $\mathcal U$ when $G$ ignores $u$) in $G$. With the same work rule, task and retained state, for every
$0<\alpha\le1$ and compared input bound $\lambda$,
\[
b_r^{\tau,\nu}(\lambda;q,\alpha)
=w_{g,r;\mathcal H}^{\tau,q}(\lambda).
\]
At integer work budgets and with inputs restricted to $\mathcal H_\lambda$,
every positive reliability requirement therefore gives the same attainable
region as deterministic execution by $g$. When $\mathcal H_\lambda$ contains all admitted task inputs of length
at most $\lambda$, the restricted minimum equals
$w_{g,r}^{\tau,q}(\lambda)$ at that input bound.
\end{proposition}
\begin{proof}
For fixed $p,u$, the event $X_r^{\tau,q}(p,u)\le b$ is nested as $b$
increases. Raising $q$ can only remove accepted responses. Summing the
indicators with nonnegative weights proves monotonicity. A nonempty set of
integer qualifying budgets has a least member, and the maximum over inputs is
attained over its finite class. These facts give the displayed inverse.
In the deterministic case every event probability is zero or one, so a
positive reliability threshold holds exactly when the run on $p$
succeeds within $b$. Substitution in the reliability and restricted ordinary
region definitions, followed by taking the least budget, proves recovery,
including the value $\infty$.
\end{proof}

At fixed $\lambda$ and $\alpha$, the attainable performance set at budget
$b$ is $\{q:b_r^{\tau,\nu}(\lambda;q,\alpha)\le b\}$. For two states, the first state's set is contained in the second's at every
budget exactly when the second state needs no larger least budget at any level that the first state reaches, by the argument of
\cref{app:grasp-duality}. With a singleton input class, these sets and least budgets concern that fixed
input. Larger classes permit minimisation over their members. 

Comparing success probabilities requires the probability distribution for each system. Comparing
outcomes run by run additionally requires a joint distribution specifying how the
two systems' source inputs are paired. In an interactive
comparison, a fixed environment response rule can still return different replies
when the systems take different actions.

\subsection{Persistent information}

To measure the work saved by a component, \cref{sec:wri} compares variants
that differ only in that component, at the same input bound. The
definition below includes both retained state and persistent external
information, whose treatment depends on the boundary discussed in
\cref{sec:boundaries}.

Fix a later process and write its persistent condition as $(r,y)$.  Here
$r\in\mathcal R_g$ is valid retained state belonging to the system,
and $y$ is a finite persistent external state stored in a distinguished
field of public context and available across the occasions under comparison.  The
counted task input is not part of $y$.

For a retained-component replacement write $r[z]=(\rho,z)$, where $\rho$ is the fixed rest of the state, with $y$ fixed.  For an external-state replacement write $y[z]=(\eta,z)$ with $r$ fixed.  In either case $z$ ranges over a valid replacement class
$\mathcal Z$, and we write $r^\pm=r[z^\pm]$ or $y^\pm=y[z^\pm]$.  

The common task specification, with its persistent external-state field omitted,
is $\bar\tau=(\bar x,\mathcal Q,\preceq,\perf)$. To include a particular external
state $y$, write $\tau[y]=(\code{\bar x,y},\mathcal Q,\preceq,\perf)$.

\begin{definition}[Matched component replacement]
\label{app:matched}
A matched comparison requires both variants $z^+$ and $z^-$ to be valid under the same
protocol. The later process, task-input interface, admitted task inputs, task evaluator,
any environment rule, system boundary and work rule are fixed. Every nonvaried state and
public-context field is also fixed, as is the source realisation.

At input bounds
where the relevant minima are finite define
\[
\Delta_{\rm work}((r^+,y){:}(r^-,y);\bar\tau,q,\lambda)
=w_{g,r^-}^{\tau[y],q}(\lambda)-w_{g,r^+}^{\tau[y],q}(\lambda)
\]
for a retained-state replacement, and
\[
\Delta_{\rm work}((r,y^+){:}(r,y^-);\bar\tau,q,\lambda)
=w_{g,r}^{\tau[y^-],q}(\lambda)-w_{g,r}^{\tau[y^+],q}(\lambda)
\]
for a persistent external state replacement. A joint replacement of retained and external fields, with every other condition fixed, uses
\[
\Delta_{\rm work}((r^+,y^+){:}(r^-,y^-);\bar\tau,q,\lambda)
=w_{g,r^-}^{\tau[y^-],q}(\lambda)-w_{g,r^+}^{\tau[y^+],q}(\lambda)
\]
when both values are finite.
\end{definition}

With the external state fixed,
\cref{sec:wri} uses the shorter notation
$\Delta_{\rm work}(r^+{:}r^-;\tau,q,\lambda)$ with $\tau=\tau[y]$.

When a difference is attributed to the content of a field, the replacements
should keep the rest of the representation and its access costs fixed, for
example by using fields of equal width. Otherwise the change of layout must be
accounted for in the comparison.

\begin{definition}[Reuse of earlier work]
\label{app:reuse}
Fix an observation rule $\mathcal O$ on executions before making any
component replacement. The rule, execution representation and selected
events are common to all variants. For a fixed task input and source realisation, a later run $\xi_z$
depends on $z$ under this observation rule when some distinct valid
matched replacement $z'\in\mathcal Z$ gives
$\mathcal O(\xi_z)\ne\mathcal O(\xi_{z'})$.
For historical attribution, $z$ must have been produced by counted work
before the cut-off between acquisition and continuation, persisted into the
later state and not been given again as later task input or other counted
exogenous input. Dependence on $z$ under these conditions is reuse of earlier work.

For the finite constructions below, $\mathcal O(\xi)$ comprises the
halting status, emitted response and assigned work, wherever defined,
together with the ordered sequence $J(\xi)$ of writes executed to
retained state. Each event gives the written location and value.
The observer does not insert untouched starting memory into $J$.
A write of the value already present is still an event, as is a write
subsequently undone.
\end{definition}

Work and storage are still counted from the complete run description in
\cref{app:protocol}. The dependence test uses only the observations selected
by $\mathcal O$.

Only writes to retained state enter $J$. The dependence test ignores differences in scratchpad computation when all
selected observations remain the same.
Applications using another execution representation must specify its
corresponding update events. For an interactive system, the selected observations can include the sequence of its external actions. Applying the observation rule requires access to the selected events. Its halting-status observation cannot in general be settled by a finite
timeout, since a run still unfinished at that point may halt later.

We count any work used to keep stored information available. That work
can be the same for different stored values. Dependence on a value is
tested by replacing it and comparing the selected observations. The
contribution of the storage device that keeps it available is tested by replacing that device under a fixed
preservation task and specified disturbances.

Work done inside the system boundary to monitor execution or send information to an outside observer counts. The observer's inspection and log can be left outside the work and storage totals. If the process reads that log to construct messages or choose actions, the log's updates, reads and storage must be included in the comparison.

\begin{remark}[Cut-off within a computation]
\label{app:internal-cut}
The cut-off between acquisition and continuation in \cref{app:reuse} can lie inside a larger
computation.  If a component is produced before this cut-off and remains in a
valid state at that cut-off, the continuation can be evaluated as an ordinary later
run from that state.  Work before the cut-off belongs to the realised acquisition
history. Work after the cut-off is the work of the continuation, so the prefix is
not counted again as later work.

The earlier computation can be treated as an acquisition run when it returns a qualifying
state under the acquisition protocol. Its work includes any operations needed to extract and
emit that state. The resulting run gives the upper bound on least acquisition work in
\cref{app:acquisition-upper-bound}. The acquisition-depth bound also requires the
input-length condition stated there.
\end{remark}

\begin{definition}[Work-reducing information]
\label{app:wri}
For a matched retained-state replacement with the persistent
external state fixed, the state difference is \emph{work-reducing
information} for $(\tau,q,\lambda)$ when
$\Delta_{\rm work}(r^+{:}r^-;\tau,q,\lambda)>0$.  Attributing this WRI to earlier work requires the production, persistence and input conditions in \cref{app:reuse}.
\end{definition}

WRI identifies a beneficial replacement, and the associated numerical quantity measures the work saved relative to the specified alternative component. With all other conditions fixed, the comparison attributes the difference in least work to that replacement.

\begin{remark}[External work-reducing resources]
\label{app:wri-special-case}
A persistent external resource can lower least work without being WRI for
the system. WRI names the favourable retained-component case.
The same stored resource can therefore be WRI under a boundary that
includes it and an external work-reducing resource under a narrower
boundary.
\end{remark}

\begin{corollary}[Least-work comparisons imply grasp dominance]
\label{app:wri-to-grasp}
Fix a task family $\mathfrak T$ and valid states $r^+,r^-$ of one system $g$, evaluated on each task under a common protocol. For each $\tau\in\mathfrak T$, fix an input bound $\lambda_\tau$. Suppose that, for every $\tau\in\mathfrak T$ and every level $q$ with $w_{g,r^-}^{\tau,q}(\lambda_\tau)<\infty$,
\[
w_{g,r^+}^{\tau,q}(\lambda_\tau)\le w_{g,r^-}^{\tau,q}(\lambda_\tau).
\]
Then $(g,r^+)$ weakly grasp-dominates $(g,r^-)$ on each task at every finite work budget.
\end{corollary}
\begin{proof}
Apply the implication from (ii) to (i) in \cref{app:grasp-duality} to each task, with $\mathcal W$ the set of finite work budgets.
\end{proof}

The matched comparison of \cref{app:matched} can also be made at fixed reliability, using the reliable work budgets of \cref{app:reliable-region}.

\Needspace{6\baselineskip}
\begin{definition}[Matched work reduction at fixed reliability]
\label{app:reliable-work-reduction}
Write $r^+=(\rho,z^+)$ and $r^-=(\rho,z^-)$ for valid retained states, where $\rho$ is the fixed rest of the state and $z$ is the component being compared. Use the same execution procedure, task and evaluator, input classes,
source mechanism and distribution, persistent external conditions, interfaces, work
rule, and other resource limits. Fix $q,\lambda,\alpha$ for which both
reliable work budgets are finite. Their work effect is
\[
\Delta_\alpha^\nu(r^+{:}r^-;\tau,q,\lambda)
=b_{r^-}^{\tau,\nu}(\lambda;q,\alpha)
 -b_{r^+}^{\tau,\nu}(\lambda;q,\alpha).
\]
A positive value is a retained-state work reduction at fixed reliability.
Its historical attribution requires the same production, persistence and
input conditions (\cref{app:reuse}).
\end{definition}

In the deterministic case of \cref{app:reliable-recovery}, this work effect is the difference of the
two ordinary least-work values restricted to the same $\mathcal H_\lambda$.
For the full bounded-length input class it equals $\Delta_{\rm work}$ in
\cref{app:matched}, so its positive case is ordinary WRI. A restricted-input
claim must keep that input restriction.

\subsection{Acquiring capability}

In \cref{sec:acquisition}, the output to be produced is a state with a
required capability. We use the same depth definition for this acquisition
task, with accepted outputs restricted to states that meet the requirement.

\begin{definition}[Accepted-set shorthand]
\label{app:accepted-set}
For a nonempty accepted set $S\subseteq\Sig^*$ and public context $x$, let
$(x,S)$ denote the binary task of \cref{app:accepted-levels}.  Write
\[
\C(S\given x):=\C_{(x,S)}(1\given x),\qquad
w_{g,r}^{(x,S)}(\lambda):=w_{g,r}^{(x,S),1}(\lambda),
\]
and
\[
d_{g,r}^{(x,S)}(s):=d_{g,r}^{(x,S),1}(s).
\]
For execution under a probability distribution $\nu$, write similarly
$F_r^{(x,S),\nu}(p;b):=F_r^{(x,S),\nu}(p;b,1)$ and
$b_r^{(x,S),\nu}(\lambda;\alpha):=b_r^{(x,S),\nu}(\lambda;1,\alpha)$.
With the system and public context fixed, we write $w_r^S(\lambda)$,
$d_r^S(s)$, $F_r^{S,\nu}(p;b)$ and $b_r^{S,\nu}(\lambda;\alpha)$.
When the state is also fixed, we omit its subscript, as in
$w^S(\lambda)$ and $d^S(s)$.
Here $\C(S\given x)$ is the minimum description length of an accepted
response, not the description length of $S$ as a set.
For the empty set, use $F_r^{\emptyset,\nu}(p;b)=0$ and
$w_r^{\emptyset}(\lambda)=b_r^{\emptyset,\nu}(\lambda;\alpha)=\infty$.
\end{definition}

\begin{definition}[Acquisition tasks]
\label{app:acquisition}
An acquisition pair $(A,a)$ is a deterministic acquisition system $A$ with
valid starting state $a$, evaluated under the protocol of \cref{app:protocol} with its
responses interpreted as candidate valid states.  Least work and depth for an
acquisition task are therefore the ordinary quantities
$w_{A,a}^{\tau,q}(\lambda)$ and $d_{A,a}^{\tau,q}(s)$.

A \emph{checkpoint} is a finite encoding of a valid state emitted or
carried by an acquisition or improvement process. The acquisition system produces a state to be used by the evaluated system $g$ on the later tasks.

For $g$ with a decidable class of valid states, an effective
finite capability requirement is a finite family
\[
\Gamma=((\tau_j,\lambda_j,w_j,q_j))_{j=1}^m
\]
for which one algorithm, given a valid candidate state $r$ and index $j$,
terminates and decides whether
$(\lambda_j,w_j,q_j)\in\A_g(r,\tau_j)$. Its accepted state set is
\[
S_\Gamma^g
=\{r\in\mathcal R_g:
(\lambda_j,w_j,q_j)\in\A_g(r,\tau_j)\text{ for every }j\}.
\]
When $S_\Gamma^g$ is nonempty, fix a public acquisition context
$x_{\rm acq}$ and define the binary acquisition task
$\tau_{\rm acq}=(x_{\rm acq},S_\Gamma^g)$. Its acquisition depth from $(A,a)$ at margin $s$ is $d_{A,a}^{\tau_{\rm acq},1}(s)$.
\end{definition}

A sufficient condition for decidability of the bounded-success tests is uniform
effective exhaustibility of these bounded runs as $r$ varies (\cref{app:exhaustible}), together
with decidable response-threshold tests. Membership is decidable by checking state validity and the finitely many
bounded-success tests.

Completing the acquisition task requires emitting a qualifying state, even when that state is already available as a checkpoint. The membership test concerns a given state. Determining whether any qualifying state exists is a separate problem. For a training comparison with a fixed corpus available before execution, include the corpus and the capability specification in $x_{\rm acq}$, and keep the starting checkpoint in $a$. The length of any reference program that uses this context to produce a qualifying state is an upper bound on $\C_{\tau_{\rm acq}}(1\given x_{\rm acq})$. The minimum is over all such programs. Acquisition depth then minimises the counted work of $(A,a)$ over its admitted inputs at the resulting input bound. Reading the corpus and emitting the state are included in the training work rule.

\begin{remark}[Upper bound from an acquisition run]
\label{app:acquisition-upper-bound}
Fix an acquisition pair $(A,a)$ and a binary acquisition task $\tau_{\rm acq}=(x_{\rm acq},S)$, for example with $S=S_\Gamma^g$. If a successful
run using a task input $p$ with $|p|\le\lambda$ emits a valid
state in $S$ as its response in total work $W$, including the
required output work, then
\[
w_{A,a}^{\tau_{\rm acq},1}(\lambda)\le W.
\]
If moreover
\[
|p|\le \C_{\tau_{\rm acq}}(1\given x_{\rm acq})+s,
\]
then
\[
d_{A,a}^{\tau_{\rm acq},1}(s)\le W.
\]
Thus one realised acquisition history gives an upper bound on the least
work. To prove that acquisition requires at least a given amount of work,
we must show that every successful run within the same input bound uses at least that much work.
\end{remark}

\begin{definition}[Reachable targets across tasks]
\label{app:cap-profile}
For a family
$\mathfrak F=((\tau_i,\lambda_i,w_i,q_i))_{i\in I}$ of tasks $\tau_i$
with input bounds $\lambda_i$, work budgets $w_i$ and target levels $q_i$, define
\[
\operatorname{Cap}_{g,r}(\mathfrak F)
=\{i\in I:(\lambda_i,w_i,q_i)\in\A_g(r,\tau_i)\}.
\]
Dominance on $\mathfrak F$ is inclusion of these sets.
\end{definition}

\subsection{Organisational arrangements}

The organisational comparison in \cref{sec:together} distinguishes work saved relative to
unorganised distributed components, extra work relative to a pooled
implementation, and capability made possible by interaction.

\begin{definition}[Matched organisational comparison family]
\label{app:orgfamily}
Fix one task $\tau$, performance level $q$, outer task interface,
outer exogenous-input and boundary convention, common work-accounting units with
communication counted \cite{yao1979}, and a total resource
limit.

A matched organisational family consists of four system--state pairs
$M=(g_M,r_M)$ under that common comparison.
\begin{itemize}
\item $M_{\rm central}$ is a specified pooled implementation with the
observation and action interfaces of the protocol.
\item $M_{\rm org}$ is the organised distributed implementation.
\item $M_{\rm flat}$ is a distributed implementation with matched component
resources and interfaces but without the shared conventions being compared.
\item $M_{\rm isolated}$ is a matched noncoordinating implementation used
when the question is whether interaction creates capability unavailable
to separated components.
\end{itemize}
The same exogenous inputs and permitted actions cross the common outer
boundary in every comparison. The implementations can distribute them differently among their internal components, and
all work traces are counted in the same units.

Write
$w_M^{\tau,q}(\lambda):=w_{g_M,r_M}^{\tau,q}(\lambda)$. For the pooled, organised and unorganised arrangements, abbreviate these quantities as $w_{\rm pool}$, $w_{\rm org}$ and $w_{\rm unorg}$, respectively. Comparisons are made
only at input bounds for which the relevant least-work quantities are finite.
The work difference from pooling is $T_{\rm coord}=w_{\rm org}-w_{\rm pool}$ and the work
saved by organisation is $G_{\rm org}=w_{\rm unorg}-w_{\rm org}$ at such input bounds.
\end{definition}

For a finite task family, these comparisons are made separately on each task.
Capability enabled by interaction is compared separately through the
attainable regions.

\subsection{Recursive improvement}

Improvement work must be compared relative to the producer from which each
round begins. Reaching a fixed absolute performance level can become easier
without making a comparable further improvement easier. The protocol below
therefore includes the actual producer in the evaluated output.

\begin{definition}[Improvement evaluated against the producer]
\label{app:producer-relative}
Fix an outer improvement system $g_{\Delta}$ that wraps a local improvement process. Producer and successor states are valid starting states of $g_{\Delta}$. When domain capability is compared, they are also interpreted as states of one fixed domain system under a common encoding.  When a call begins from a
producer state $r$ and proposes a successor $r'$ together with
procedure data $\sigma$, the fixed outer wrapper emits
\[
e^{\Delta}=\code{r,r',\sigma}.
\]
The producer field is inserted by the wrapper from the actual pre-call
state. It is not chosen by the local improvement process.  An
improvement task
\[
\tau^{\Delta}
=(x^{\Delta},\mathcal Q^{\Delta},\preceq^{\Delta},\perf^{\Delta})
\]
evaluates such outputs by one fixed relation between producer and
successor, optionally together with fixed procedural requirements.
Constructing and emitting the wrapper tuple is counted by the
work rule. Comparability of equal performance levels across producers is
defined by this common evaluator, not by a baseline-independent scale.
\end{definition}

\begin{definition}[Positive recursive gain]
\label{app:recursive-gain-def}
Suppose an improvement round beginning at $r_t$ causally produces
$r_{t+1}$, and a matched component replacement establishes that a component already
present at $r_t$ contributed to producing that successor. Positive recursive
gain occurs when the successor can make a comparable further improvement more efficiently on one fixed task that evaluates improvement over the producer.

The successor can require less work for the same further improvement. In this case, for some common gain target $q$ and input bound $\lambda$, both least-work values are finite and
\[
w_{g_{\Delta},r_{t+1}}^{\tau^{\Delta},q}(\lambda)
<
w_{g_{\Delta},r_t}^{\tau^{\Delta},q}(\lambda).
\]
Alternatively, the successor can achieve a greater improvement at the same resources. At a common input bound $\lambda$ and work budget
$B$, there are levels $q\prec^\Delta q'$ for which the producer's and successor's grasp sets satisfy
\[
q\in\mathcal G_{g_\Delta,r_t}^{\tau^\Delta}(\lambda,B),\qquad
q'\notin\mathcal G_{g_\Delta,r_t}^{\tau^\Delta}(\lambda,B),\qquad
q'\in\mathcal G_{g_\Delta,r_{t+1}}^{\tau^\Delta}(\lambda,B).
\]
Here $q\prec^\Delta q'$ means $q\preceq^\Delta q'$ and $q\ne q'$. Each
gain is evaluated relative to its actual producer.
\end{definition}

For the comparison at a common gain target $q$, setting the input bound to $\C_{\tau^{\Delta}}(q\given x^{\Delta})+s$ for a common margin $s$ gives the corresponding depth
comparison. Both least-work values in this comparison are finite. A finite later
value against an infinite baseline instead extends attainability.

The same gain level can require different work from different producers.

The replacement test for self-contribution uses the dependence test in \cref{app:reuse}. A difference confined to the copied producer field of $e^{\Delta}$, or to a component copied unchanged into the successor, does not count.

Attributing the successor's advantage on further improvement to an inherited product requires a separate test
that replaces only that product in $r_{t+1}$, with the other state fixed. Attribution also requires evidence of the
product's production and persistence history. For retained products this
is the comparison of \cref{app:wri} for WRI attributed to earlier work. Persistent external
products use the external-state comparison in \cref{app:matched}.

\begin{definition}[Human and machine work]
\label{app:human-machine-work}
Fix a task $\tau$ with context $x$, target $q$, system $g$, state $r$ and input
bound $\lambda\in\Nat$.
Assign each finite run two finite nonnegative costs, $\Work_g^{\rm hum}$ for human work and
$\Work_g^{\rm mach}$ for machine work, in fixed units. At nonnegative finite caps $\bar h,\bar m$, among successful runs with admitted
input $p$ of length at most $\lambda$, define
\[
\begin{aligned}
w_{\mathrm{hum},r}^{\tau,q}(\lambda;\bar m)
 &=\min\{\Work_g^{\rm hum}(r;\code{x,p}):\Work_g^{\rm mach}(r;\code{x,p})\le\bar m\},\\
w_{\mathrm{mach},r}^{\tau,q}(\lambda;\bar h)
 &=\min\{\Work_g^{\rm mach}(r;\code{x,p}):\Work_g^{\rm hum}(r;\code{x,p})\le\bar h\}.
\end{aligned}
\]
An empty qualifying set gives value $\infty$. For sampled execution $(G,\nu)$ as in \cref{app:source-law}, use the same two cost rules and fix reliability $0<\alpha\le1$. The reliable human-work budget at machine cap $\bar m$ is the infimum of $h\ge0$ for which some $p\in\mathcal H_\lambda$ gives probability at least $\alpha$ of success at level $q$ with human work at most $h$ and machine work at most $\bar m$, all on the same run. If no such $h$ exists, the budget is $\infty$. Interchanging the two costs and caps defines the machine-work budget.
\end{definition}

\subsection{Costs over future tasks}

A work saving can be outweighed by the cost of acquiring and maintaining
the useful structure. The complete-plan comparison in \cref{sec:horizon}
therefore includes acquisition and execution work, updates, storage use
and the lengths of descriptions and task inputs.

Fix a current decision state, finite horizon $H$, system boundary and
work rules, retained-state capacity $S_{\max}$, required performance levels,
a protocol for exogenous inputs and an overall goal expressed as a condition on the completed history. Future tasks are specified in advance
or determined by a fixed computable rule from completed history and
realised exogenous inputs. The protocol also fixes the rule for identifying the start of
each stage and the source realisations on which success is required. A
probabilistic requirement instead fixes a source distribution and a minimum success probability.

\begin{definition}[Feasible continuation plans]
\label{app:horizon-account}
A \emph{feasible continuation plan} $P$ is a controller program, executed inside the
boundary under the fixed work rules, that selects runs and updates over these $H$ stages
and meets the following conditions.
\begin{enumerate}[label=(\roman*)]
\item Each choice is computed using only the controller's current configuration and the
information delivered to it at that decision. Everything kept from one stage to the next
belongs to the retained state.
\item On every execution, the retained state stays within capacity $S_{\max}$.
\item On the source realisations fixed by the protocol, or with at least the minimum success
probability, the required performance is achieved and the completed history meets the fixed
overall goal, including when plans lead to different local tasks.
\end{enumerate}
For such a plan, define its cost vector by
\[
\mathbf C_H(P)
=\bigl(\Lambda_H(P),W_{{\rm acq},H}(P),W_H(P),S_H(P)\bigr),
\]
where $r_t^P$ is the retained state at the start of stage $t$ and
$\sd{\cdot}$ is the self-delimiting code of \cref{app:model}. At each stage,
$p_t^P$ encodes all counted task inputs and $\zeta_t^P$ all other counted
exogenous inputs, using fixed tuple encodings if there are several. Set
\[
\Lambda_H(P)
=|\sd{r_1^P}|+
\sum_{t=1}^H|\sd{p_t^P}|+
\sum_{t=1}^H|\sd{\zeta_t^P}|.
\]
$W_{{\rm acq},H}(P)$ counts acquisition, construction and migration work,
while $W_H(P)$ counts other execution work, including ordinary state updates
and maintenance. A common allocation rule fixed before comparing plans
assigns each counted operation to exactly one of these coordinates. Also,
\[
S_H(P)=\sum_{t=1}^H|r_t^P|
\]
sums the stored lengths over the stages. The initial state enters both
$\Lambda_H$ as description length and $S_H$ as storage use.
Producing the final state $r_{H+1}^P$ is counted in work, but its length
enters this sum only if a further stage is included.
Plans compared from the same decision state run on the same system from the same initial state. Their programs, or instructions selecting them, enter the counted input.
Plan $P$ costs no more than
$Q$ in any category when every coordinate of $\mathbf C_H(P)$ is no larger than the
corresponding coordinate of $\mathbf C_H(Q)$.

For given nonnegative coefficients
$\boldsymbol\mu=(\mu_{\Lambda},\mu_{\rm acq},\mu_W,\mu_S)$ define
\[
\mathcal C_H^{\boldsymbol\mu}(P)
=\boldsymbol\mu\!\cdot\!\mathbf C_H(P),
\qquad
\Adv_H^{\boldsymbol\mu}(P{:}Q)
=\mathcal C_H^{\boldsymbol\mu}(Q)
-\mathcal C_H^{\boldsymbol\mu}(P).
\]
\end{definition}

Report the four coordinates separately. A scalar comparison uses the
stated coefficients. A positive $\Adv_H^{\boldsymbol\mu}(P{:}Q)$ means
that $P$ has lower weighted total cost than $Q$.

For a stochastic procedure, these four quantities give costs for each realised history. When its finite-horizon execution is represented by one
fixed $G$ and a finite source set, \cref{app:source-law} applies to
the whole procedure. Its input-delivery mechanism reveals observations only after
the actions that obtain them. The execution counts all included work and returns
the completed history. The evaluator accepts
histories meeting the fixed overall goal and required performance levels. Failed or nonhalting runs
are unsuccessful. Averaging coordinate costs requires a specified distribution and
cost convention for all histories, with finite means where used.

Adding separate stage budgets can underestimate the budget needed for the same reliability over the whole procedure. Suppose two
independent stages always meet their targets, each using one work unit with probability
$0.9$ and eleven with probability $0.1$. Each has a $90\%$ work budget of
one, so their separate budgets sum to two. Total work is two, twelve or twenty-two with probabilities $0.81$,
$0.18$ and $0.01$, respectively. The whole-procedure $90\%$ budget is
twelve, whereas expected total work is four. Neither the sum of stage budgets nor the expected total gives the reliable budget for this procedure. That budget must be computed from the total-work distribution.

\Needspace{6\baselineskip}
\begin{remark}[Costs of retaining a component]
\label{app:horizon-retention}
Let $P$ and $Q$ be feasible plans from the same decision state. Plan $P$ retains a component $z$, acquiring or constructing it during the continuation if needed. Plan $Q$ omits $z$ or constructs a substitute, and it can still learn and keep its own results. Within $W_H$, fix a common convention that assigns each counted execution or update operation exactly once, either to a stage contribution $U_t$ or to the remaining work $W_{H,\rm other}$. Thus
\[
W_H(P)=\sum_{t=1}^H U_t(P)+W_{H,\rm other}(P),\qquad
W_H(Q)=\sum_{t=1}^H U_t(Q)+W_{H,\rm other}(Q).
\]
Put $\Delta_t(z)=U_t(Q)-U_t(P)$. These are differences for the given
plans, not differences of independently minimised local task costs. Define the remaining cost difference
\[
\begin{aligned}
C_{H,\rm other}^{\boldsymbol\mu}(z)
={}&\mu_{\Lambda}(\Lambda_H(P)-\Lambda_H(Q))
+\mu_{\rm acq}(W_{{\rm acq},H}(P)-W_{{\rm acq},H}(Q))\\
&+\mu_W(W_{H,\rm other}(P)-W_{H,\rm other}(Q))
+\mu_S(S_H(P)-S_H(Q)).
\end{aligned}
\]
Subtracting the complete costs of the two plans gives
\[
\Adv_H^{\boldsymbol\mu}(P{:}Q)
=\mu_W\sum_{t=1}^H\Delta_t(z)-C_{H,\rm other}^{\boldsymbol\mu}(z).
\]
This convention determines how the total divides between the two terms. When all coordinate differences in $C_{H,\rm other}^{\boldsymbol\mu}(z)$ are nonnegative, $C_{H,\rm other}^{\boldsymbol\mu}(z)$ is the weighted extra cost that $P$ incurs outside stage task work. Otherwise it is a signed net difference between the two plans, with some cost categories favouring $P$.
\end{remark}

For each horizon $H$, compare feasible plans $P_H$ and $Q_H$ under fixed conventions and finite nonnegative weights. Let $\overline\Delta_H=\frac1H\sum_{t=1}^H\Delta_t(z)$ be their average stage task-work difference. The identity gives $\Adv_H^{\boldsymbol\mu}(P_H{:}Q_H)=H\mu_W\overline\Delta_H-C_{H,\rm other}^{\boldsymbol\mu}(z)$. For constants $C_{\rm init}\ge0$, $\kappa\ge0$ and $\epsilon>0$ independent of $H$, the bounds $C_{H,\rm other}^{\boldsymbol\mu}(z)\le C_{\rm init}+H\kappa$ and $\mu_W\overline\Delta_H\ge\kappa+\epsilon$ give $\Adv_H^{\boldsymbol\mu}(P_H{:}Q_H)\ge H\epsilon-C_{\rm init}$. Reversing the bounds to $C_{H,\rm other}^{\boldsymbol\mu}(z)\ge C_{\rm init}+H\kappa$ and $\mu_W\overline\Delta_H\le\kappa-\epsilon$ gives $\Adv_H^{\boldsymbol\mu}(P_H{:}Q_H)\le-C_{\rm init}-H\epsilon$.

\section{Statements and proofs}
\label{app:statements}

The results below use the definitions of \cref{app:defs}. They cover
controlled and sampled execution (\cref{sec:llm-execution}), cache savings
(\cref{sec:wri}), retained shortcuts (\cref{sec:miniature}), cumulative deepening
(\cref{sec:deepening}), finite search
(\cref{sec:wri-search}), persistence
(\cref{sec:boundaries}), switching (\cref{sec:switching}), organisation
(\cref{sec:together}), evaluation (\cref{sec:judging}) and recursive improvement (\cref{sec:improving}).

\subsection{Controlled and sampled execution}

These constructions support \cref{sec:llm-execution}, the cache comparison in
\cref{sec:wri}, the runtime search in \cref{sec:wri-search} and the evaluator
comparison in \cref{sec:judging}. In controlled generation, the task input specifies token choices, whereas
in sampled generation, success depends on the distribution of choices and on the controller. The examples distinguish generating an
acceptable candidate from returning one.

Fix the public context, the retained state and a deterministic finite-precision next-token
implementation with its vocabulary and finite stopping rule. At each step, list candidate tokens in decreasing score order,
breaking ties by a fixed rule. Fix $k$ between one and the vocabulary size and
allow only the first $k$ tokens at each step. The generation procedure accepts
a length-delimited sequence of positions in $\{1,\ldots,k\}$, ignores entries
after stopping and declares failure if it needs an entry that is missing or
outside this set. All other execution choices are fixed by the implementation
or by one fixed realisation of the other input channels.

The fixed context and state determine the first candidate list, and the first
position determines the first token and the next state. Repeating this argument
shows that every sequence that causes no failure determines a unique execution.
Conversely, every token trace that uses the first $k$ candidates through the
stopping point is reproduced by the sequence of positions chosen along it.
Fixed-width binary coding of $L$ positions has payload length
$L\lceil\log_2 k\rceil$, and a standard length header adds $O(\log(L+2))$ bits.
The sequence and any accompanying instructions form part of the task input.
Computing the candidate lists, decoding the sequence and updating the state are
counted as work of the implementation.

An interface that also accepts programs can express regular sequences more
briefly, for example by a program that chooses the first candidate at every
step. The program then counts as task input and its execution counts as work.

\begin{example}[Equal controlled work with different sampling reliability]
\label{app:llm-separations}
Consider a finite generation procedure with two possible one-symbol responses, $0$ and
$1$, where only $1$ succeeds. Fix two retained states. In both, the ordered menu is
$(0,1)$. This one-step interface takes a single control bit, without a length header, to select either token. Those controlled runs have
the same fixed work $c$ in both states. They therefore have the same
least successful controlled work at any input bound that admits the one-bit control. In the default-sampling
mode, with no token choice given, draw $u$ uniformly from
$\{0,\ldots,9\}$. One state returns $1$ exactly when $u<1$. The other does
so exactly when $u<4$. Pad every controlled or sampled generation run to the same integer work $c$.
Their empty-input success probabilities at $c$ are $1/10$ and
$2/5$, although the controlled executions coincide.

For the reliable-work comparison, admit only the empty task input, so $\mathcal H_\lambda=\{\eps\}$. Let a common controller make at most $m$ independent attempts, using a uniform draw from $\{0,\ldots,9\}^m$. Each complete attempt, including an exact check and the controller's operations, is padded to a common integer work cost, again denoted $c$. The controller returns on success and declares failure after $m$ unsuccessful attempts. At budget $kc$ for $1\le k\le m$, a state whose per-attempt success probability is $\theta\in(0,1)$ succeeds with probability $1-(1-\theta)^k$. Its reliable work budget is therefore
\[
c\left\lceil\frac{\log(1-\alpha)}{\log(1-\theta)}\right\rceil
\]
for $0<\alpha<1$ when the required number of attempts is at most $m$, and is infinite otherwise. With $m\ge44$ and required success probability $\alpha=0.99$, the first state needs $44c$ and the second $10c$. At $\alpha=0.9$, they need $22c$ and $5c$. Thus equal controlled work can coexist with more than a fourfold difference in the work needed to return a correct response at the required reliability. The positive saving at fixed reliability concerns the restricted input class. Attributing that saving to earlier work would also require earlier production and persistence of the changed state component.
\end{example}

\begin{example}[Valid code repairs]
\label{app:repair-selection}
Consider repairing a Boolean expression that returns \texttt{a} where $\max(a,b)$ is required. Candidate repairs are always-halting expressions in $a,b$. A repair is valid when it returns $\max(a,b)$ on all four input pairs. The expressions \texttt{a or b}, \mbox{\texttt{not (not a and not b)}}, \texttt{a if a >= b else b}, and \texttt{b if b >= a else a} all compute $\max(a,b)$. The single-variable expressions \texttt{a} and \texttt{b} each pass the tests with $(a,b)=(0,0)$ and $(1,1)$ but fail one test with unequal inputs. Testing all four inputs is an exact verifier. The equal-input tests alone are not.

At a common padded generation cost, let generator A assign probability
$1/4$ to one valid repair and the rest to an invalid repair. Let generator B
assign probability $1/8$ to each of the four valid repairs and the rest to
an invalid repair. B succeeds with probability $1/2$, versus A's $1/4$,
even though each of its valid repairs is less likely than A's.

Suppose B's invalid repair is \texttt{a}. A controller draws two independent
candidates from B and returns the first that passes its tests, or reports failure
if neither passes. A valid candidate is present with probability $3/4$. Testing
all four inputs returns a valid repair with this probability. The equal-input
tests accept every candidate, so the controller returns the first, which is
valid with probability $1/2$.
\end{example}

We compare the proxy and intended evaluators of \cref{sec:judging} on the same
returned responses, keeping any internal checking and selection fixed.

\begin{proposition}[Proxy and goal success probabilities]
\label{app:reliable-transfer}
Fix the execution procedure $G$, retained state $r$, task input $p$,
source distribution $\nu$ and work budget $b$. Let $\widehat S$ be the
responses that pass the proxy test and $S^*$ those that meet the intended
goal. Write $\widehat F$ and $F^*$ for their success probabilities within
budget $b$.

Let $e$ be the returned response and $W$ the work of its execution.
The difference between the success probabilities is the probability of
a false pass minus that of a false fail within budget $b$,
\[
\begin{aligned}
\widehat F-F^*
={}&\Pr[W\le b,\ e\in\widehat S\setminus S^*]\\
&-\Pr[W\le b,\ e\in S^*\setminus\widehat S],
\end{aligned}
\]
where both events require the run to halt and return a response.
If the proxy test is sound, so that $\widehat S\subseteq S^*$,
there are no false passes and $\widehat F\le F^*$.
Since this holds for every admitted $p$ and budget $b$, the same task input
classes $\mathcal H_\lambda$ and reliability requirement $0<\alpha\le1$ give
\[
b_r^{S^*,\nu}(\lambda;\alpha)
\le
b_r^{\widehat S,\nu}(\lambda;\alpha).
\]
If the proxy test is complete, so that $S^*\subseteq\widehat S$, there are
no false fails and both inequalities are reversed.
\end{proposition}
\begin{proof}
Partition the two success events into their common accepted part and their
two disjoint differences. The common probabilities cancel. Soundness removes
the first difference event and completeness the second. The resulting
probability inequality holds for every admitted $p$ and budget $b$, so any
budget qualifying for the target with the smaller accepted set also qualifies
for the other target.
\end{proof}

The key-value cache in \cref{sec:wri} stores intermediate results computed from a shared prompt prefix.

\begin{proposition}[Exact work savings from a stored cache]
\label{app:cache-reliability}
Fix two matched states of one execution procedure, $r^+$ storing a valid
cache and $r^-$ recomputing it from the same stored prefix. The subscripts $+$ and $-$ below refer to these two states.
Suppose that for every admitted $p,u$ they have the same response and status,
and each finite completed run satisfies
$W_-(p,u)=W_+(p,u)+C$ for one constant integer $C>0$. Cache reading and
reconstruction are included in these counts. Then, for every finite budget $b$, the success probabilities satisfy
\[
F_+^{\tau,\nu}(p;b,q)=F_-^{\tau,\nu}(p;b+C,q).
\]
Whenever the corresponding reliable work budgets are finite,
\[
b_-^{\tau,\nu}(\lambda;q,\alpha)
=b_+^{\tau,\nu}(\lambda;q,\alpha)+C.
\]
\end{proposition}
\begin{proof}
For each fixed task input $p$ and source input $u$, the two runs have the same response and status, and every completed run from $r^-$ costs $C$ more than its counterpart from $r^+$. Consequently, the run from $r^+$ returns an accepted response within budget $b$ exactly when the run from $r^-$ does so within budget $b+C$. Averaging their equal success indicators over the common source distribution proves the first identity. Every successful run from $r^-$ costs at least $C$. The least integer budget attaining any fixed positive success probability therefore increases by $C$ for each task input. Minimising over the same finite input class proves the second identity.
\end{proof}

Fix one controller, an integer $m\ge1$, and one admitted task input. The controller makes up to
$m$ attempts and returns the first that passes its check, which accepts
exactly the correct answers. Under the source distribution the attempts
succeed independently, each with probability $\theta\in(0,1)$. From the state
with valid precomputed results stored in the cache, each complete attempt costs an integer
$c>0$, including generation, checking, cache access and control. The matched
state has the same underlying prefix, responses and probability distribution,
and reconstructs the identical cache at integer cost $C>0$.

\begin{example}[Repeated attempts with a stored cache]
\label{app:cache-two-attempts}
For a required success probability
$0<\alpha\le1-(1-\theta)^m$, the minimum number of attempts the budget must allow is
\[
k_\alpha=\left\lceil\frac{\log(1-\alpha)}{\log(1-\theta)}\right\rceil.
\]
If the cache is reconstructed once before the first attempt, the reliable
work budgets are $ck_\alpha$ and $ck_\alpha+C$. If it is reconstructed before
every attempt, they are $ck_\alpha$ and $(c+C)k_\alpha$, so the saving
$Ck_\alpha$ is nondecreasing in $\alpha$ and changes in steps. For $\alpha>1-(1-\theta)^m$ neither budget
is finite.
\end{example}
\begin{proof}
Within budget $b$, up to
$\min\{m,\lfloor b/c\rfloor\}$ attempts fit for the state that retains the cache, so it succeeds with probability
$1-(1-\theta)^{\min\{m,\lfloor b/c\rfloor\}}$. This is at least $\alpha$ exactly
when at least $k_\alpha$ attempts fit within the budget, so its least
qualifying budget is $ck_\alpha$. One reconstruction before the first attempt
adds $C$ to every path without changing responses or status, which gives
$ck_\alpha+C$, as in \cref{app:cache-reliability}. Reconstruction before every
attempt makes each attempt cost $c+C$, and the same argument gives
$(c+C)k_\alpha$. At most $m$ attempts are made, so no budget reaches a success
probability above $1-(1-\theta)^m$. The singleton input class leaves no further
minimisation over inputs.
\end{proof}

The runtime search discussed in \cref{sec:wri-search} enumerates candidate
inputs for one task. The controller performs the search and checks the candidates, and both kinds of work are counted.

\begin{proposition}[Searching for a task input]
\label{app:guidance-search}
Fix a deterministic system in a given starting state and an available
computable acceptance test whose accepted responses meet the intended
target. There are $N_\lambda=2^{\lambda+1}-1$ binary strings of length
at most $\lambda$ to consider as task inputs to this system.

Fix a controller that holds the system's description and starting state,
the test and the bounds, and that can decide which task inputs the system
admits. Suppose one complete bounded attempt, including candidate
enumeration, the admission check, reset, decoding, simulation to a fixed
primitive-step limit, verification of any returned response and output of
an accepted one, terminates and costs at most $T$ work. If some admitted
task input of length at most $\lambda$ produces an answer that passes the
test within that limit, this controller, with an empty task input of its
own, returns an accepted response within
\[
W_{\rm setup}+N_\lambda T
\]
work, where $W_{\rm setup}$ covers reading this information and initialisation.
\end{proposition}
\begin{proof}
Enumerate the finite strings and perform the bounded attempt for each,
rejecting inputs not admitted by the system. No attempt blocks the enumeration. The successful
one occurs in the list, so summing the per-attempt bound gives the result.
\end{proof}

Each additional symbol allowed in the searched inputs roughly doubles the
number $N_\lambda$ of candidates. Levin's universal search \cite{levin1973}
refines such enumeration by dividing running time among candidates
according to their length.

\begin{proposition}[Expected work of independent attempts]
\label{app:retries}
Suppose complete attempts terminate, are independent with the same
distribution, reset to the same condition and succeed with probability
$\theta>0$. The work of each attempt, successful or not, includes reset,
generation, checking and control. For a successful attempt it also includes
returning the accepted response. The mean work $\mu$ of an attempt is finite.
An available exact checker stops the process at its first success or after
$m$ attempts. The success probability is $1-(1-\theta)^m$, and the mean
work of the attempts is
\[
\mu\sum_{j=0}^{m-1}(1-\theta)^j
=\mu\,\frac{1-(1-\theta)^m}{\theta}.
\]
Any work before the first attempt is added separately. With an unlimited
independent source sequence and the same assumptions, the mean work of the
attempts to first success is $\mu/\theta$.
\end{proposition}
\begin{proof}
The process reaches attempt $j+1$ exactly when the first $j$ attempts failed,
an event of probability $(1-\theta)^j$. The cost of attempt $j+1$ is independent of those preceding
outcomes and has mean $\mu$, even if its cost is correlated with its own
success. Summing expected costs over the attempted stages gives $\mu$ times
the expected number of attempts, as in Wald's identity. The all-fail event has
probability $(1-\theta)^m$. For the unlimited process, monotone convergence gives
the limit of the nonnegative cost sum, and the geometric series has sum $1/\theta$.
The unlimited case goes beyond the finite source sets of \cref{app:source-law} and uses
the infinite product of independent draws.
\end{proof}

Here total work counts the actual work of unsuccessful runs, whereas
$X_r^{\tau,q}$ assigns them infinity. With unlimited attempts and $\theta=1/4$,
successful attempts costing four units and unsuccessful ones costing ten
give mean work per attempt $\mu=17/2$ and mean total work $34$. Assigning
every attempt the work of a successful one would give the incorrect value
$16$. Expected work is finite although no finite work bound covers all runs.

\subsection{Reducing work through retained information}
\label{app:cluster-miniature}

The first construction gives the work saving from a given retained
routine (\cref{sec:miniature}). The next constructs a cumulative history in which
intermediate results are kept (\cref{sec:deepening}). Finite search then
selects a routine by the work it requires on a set of training tasks (\cref{sec:wri-search}).

\subsubsection{Reducing work with a short program}

We use the cells $a_i(t)$ and the additive rule $a_i(t+1)=a_i(t)\oplus a_{i+1}(t)$ of
\cref{sec:miniature}, with its standard power-of-two identity \cite{martin1984}. The two retained
states, $r^-$ and $r^+$ below ($r_{\mathrm{sim}}$ and $r_{\mathrm{short}}$ in \cref{sec:miniature}), store the direct simulator and the shortcut.

\begin{proposition}[A shortcut for an additive cellular automaton]
\label{app:corollary}
Let $T=2^n$.  For every binary window $x=a_0(0)\cdots a_T(0)$ given as public
context, let $\tau_{n,x}$ require response $a_0(T)$ from a serial tape initially
positioned at $a_0(0)$ and terminated after $a_T(0)$ by a fixed end marker.  There is a transition-complete system $g_{\rm CA}$ (\cref{app:transition-complete}) with matched states
$r^-,r^+$ differing only in a fixed-width method field that encodes one of
two fixed finite routines, interpreted by the common finite control.  State
$r^-$ encodes the direct one-step simulation routine and $r^+$ encodes one
fixed power-of-two shortcut routine.  Only the empty task input succeeds, and
\[
w_{g_{\rm CA},r^-}^{\tau_{n,x},1}(0)=\Theta(T^2),\qquad
w_{g_{\rm CA},r^+}^{\tau_{n,x},1}(0)=\Theta(T),
\]
while the shortcut description is bounded independently of $n$.  Hence
\[
\Delta_{\rm work}(r^+{:}r^-;\tau_{n,x},1,0)=\Theta(T^2).
\]
\end{proposition}
\begin{proof}
The rule gives $a_i(1)=a_i(0)\oplus a_{i+1}(0)$ for every initial row and
cell $i$. If $a_i(t)=a_i(0)\oplus a_{i+t}(0)$ for every row and cell at
$t=2^k$, applying the same identity to the row at time $t$ gives
\[
\begin{aligned}
a_i(2t)&=a_i(t)\oplus a_{i+t}(t)\\
&=a_i(0)\oplus a_{i+t}(0)\oplus a_{i+t}(0)\oplus a_{i+2t}(0)\\
&=a_i(0)\oplus a_{i+2t}(0),
\end{aligned}
\]
since the two copies of $a_{i+t}(0)$ cancel. By induction,
$a_0(T)=a_0(0)\oplus a_T(0)$.

The direct implementation updates the triangular dependency cone, executing a
constant-work update at each site of the cone.  Its work is bounded above and below
by constant multiples of $1+\cdots+T=\Theta(T^2)$.  The shortcut stores
$a_0(0)$, scans the serial window to its fixed end marker while retaining the
last bit $a_T(0)$, adds the two endpoint bits modulo two and emits the result.  Its finite control is
independent of $n$, and the serial scan is both $O(T)$ and $\Omega(T)$.  Because only the empty task input can
succeed and each state selects one deterministic path, these path counts are
the least-work values.
\end{proof}

Within this system, a fixed-size routine saves unbounded work as
the horizon grows, without a growing literal answer table. This is a comparison
of two specified methods. Discovering the routine remains a separate
acquisition task. In the following example, the stored component is instead
produced by the earlier tasks of the history.

\subsubsection{Cumulative deepening}

A history shows cumulative deepening when, along the compared stages, the target depths from the original persistent condition are finite and strictly increasing at a fixed margin, and inherited products of earlier work give a positive matched work saving at each stage (\cref{sec:deepening}). Each saving compares finite least-work values at the same input bound, with the inherited products kept available across stages instead of being given again as task input.

Fix $N\ge2$, a binary word $x_1\cdots x_N$, and a serial streaming
interface. Task $\tau_n$, for $1\le n\le N$, has public context
$x^{(n)}=x_1\cdots x_n\#$ and accepts the one-bit parity of that prefix.
The growing prefix is public task context, so the baseline from $r_0$ sees the same current prefix as the inherited state. Treating it as persistent external state would instead compare against its original value under \cref{sec:deepening}. The same $x^{(n)}$ is used in both
variants of every comparison for task $n$.

One fixed deterministic process $g$ has a one-bit accumulator and an actual
input-head position in its retained configuration. A bit transition xors
the scanned bit into the accumulator and advances the head one cell.
At $\#$, one transition emits the accumulator and halts, leaving the
head at the marker. These are the counted primitive transitions. The
only admitted task input is $\eps$. Fix a tape capacity of at
least $N+1$ before the history. Valid configurations include every head
position in that capacity and either accumulator value. The original
state $r_0$ has accumulator zero and head at the first cell.

Between successive tasks, a common external rule replaces the old marker
by the next bit and places a new marker one cell later, without moving
the retained head. The public tape is physically continuous across the
history, but its current contents specify $x^{(n)}$ in the task protocol.
A serial-tape system retaining only a numerical cursor must count the work of moving its head to that position. A different input interface must account for any loading or copying it requires. Here the external extension
is outside the counted process and is common to both variants.

\Needspace{6\baselineskip}
\begin{example}[Computing successive prefix parities]
\label{app:prefix-history}
Let $r_{n-1}$ be the state left by successful task $n-1$. For $n\ge2$,
replace its head position and accumulator by their original values, keeping
$x^{(n)}$ and all other conditions fixed, which gives $r_0$. For every $\lambda\ge0$ and
every margin $s\ge0$,
\[
\begin{aligned}
D_n=d_{g,r_0}^{\tau_n,1}(s)&=n+1,\\
w_{g,r_{n-1}}^{\tau_n,1}(\lambda)&=2,\\
w_{g,r_0}^{\tau_n,1}(\lambda)
-w_{g,r_{n-1}}^{\tau_n,1}(\lambda)&=n-1>0.
\end{aligned}
\]
Thus the actual history satisfies the cumulative-deepening criterion on
$2,\ldots,N$.
\end{example}
\begin{proof}
After task $n-1$, the accumulator is the parity of
$x_1\cdots x_{n-1}$ and the head is at position $n$. This holds
after the first task and is preserved by one bit transition and the
end-marker transition on every extension. Hence each inherited-state
execution takes two transitions and returns the correct next parity.
From $r_0$ the same task instead executes $n$ bit transitions and one
end-marker transition, giving $n+1$. Replacing them restores this path.

The only admitted task input is empty, so each successful path gives
the exact least work for every input bound. The empty input also lies within every
input bound used for depth, giving the stated finite depth at every
margin. Each earlier successful task produces and retains the
head position and accumulator used next. It is not given again as counted input. The positive
matched savings concern the same targets whose depths from the original starting condition strictly increase. These are the two conditions for cumulative deepening in
\cref{sec:deepening}. The construction works for every chosen
finite $N$ with that capacity.
\end{proof}

\subsubsection{Finite search for work-reducing routines}
\label{app:cluster-wri-search}

The cellular-automaton example in \cref{app:corollary} starts with the useful routine already retained. For
the search described in \cref{sec:wri-search}, an acquisition process
instead enumerates a finite class of candidate strings, executes them
and counts their later task work under one fixed protocol, then emits a
state containing a minimiser.

Fix a finite family of binary tasks
$(\tau_i)_{i=1}^m$, with public contexts tagged by the task index, nonempty decidable
accepted sets and target level one, and fix $L\in\Nat$.

\begin{proposition}[Finite search for work-reducing routines]
\label{app:wri-search}
There exist a fixed
transition-complete deterministic system $g_{\rm s}$ for performing these tasks and a fixed
deterministic acquisition system $A_{\rm s}$ for producing its states, with the following properties.
The valid states of $g_{\rm s}$ include
\[
r_z=(\rho,z),\qquad |z|\le L,
\]
for one fixed nonvaried field $\rho$, with $r_{\eps}$ the baseline state.
The comparison includes the complete encoding of the candidate field, and reading and decoding it count in the work of the run. On
each task, $r_{\eps}$ runs a fixed exhaustive baseline directly and halts with
an accepted response. For $z\ne\eps$, $g_{\rm s}$ runs the procedure encoded by $z$, alternating its steps with those of the same exhaustive
baseline. If the candidate halts, the fixed acceptance test runs to completion. If accepted, that response is returned. Otherwise the baseline resumes. The baseline therefore guarantees termination for every
candidate. Executing both procedures, switching between them and checking the response are counted by the
same work rule.

Evaluate every state at input bound zero and write its total work across the tasks as
\[
W(z)=\sum_{i=1}^m w_{g_{\rm s},r_z}^{\tau_i,1}(0).
\]
The acquisition system exhaustively evaluates the finite set
$\{z:|z|\le L\}$, uses a fixed tie-breaking rule, and emits a state
$r_{z^*}$ with
\[
z^*\in\arg\min_{|z|\le L}W(z).
\]
This acquisition halts after finite work, and
\[
W(z^*)\le W(\eps).
\]
If the inequality is strict, then for at least one $i$,
\[
\Delta_{\rm work}
(r_{z^*}{:}r_{\eps};\tau_i,1,0)>0,
\]
so the selected state difference is work-reducing information for at
least one task in the evaluated family.
\end{proposition}
\begin{proof}
Because the family is finite and each accepted set is decidable and nonempty,
a fixed baseline for the family can enumerate finite response strings
in a computable order and stop at the first accepted response.  Hard-code the
finite task tags and their deciders into one fixed baseline branch of
$g_{\rm s}$. For a nonempty candidate field $z$, a fixed program reads and executes the
procedure encoded by $z$ on the tagged public context. A syntactically
invalid code disables that candidate branch. The system
alternates candidate and baseline transitions. If the candidate halts,
its terminating acceptance test is completed and the response is returned
only if accepted. Otherwise the baseline resumes. The decider is total,
so this check cannot prevent eventual baseline progress.  Hence every state--task run terminates,
and all transitions used to run the candidate, run the baseline, switch between them and check responses are
included in the transition-complete work count.

At input bound zero the empty string is the only admitted
task input.  Thus, for every $z$ and $i$, the work of this unique
deterministic successful run is exactly
$w_{g_{\rm s},r_z}^{\tau_i,1}(0)$.  There are finitely many strings with
$|z|\le L$ and finitely many tasks, so $A_{\rm s}$ can simulate every such
run to completion, count its work, compute each finite sum $W(z)$, and emit
a minimiser under the fixed tie-breaking rule.  This uses finite acquisition
work.

Since $\eps$ is among the candidate fields, minimality gives
$W(z^*)\le W(\eps)$.  If the inequality is strict, not every summand can be
at least its baseline counterpart.  Hence for some $i$,
\[
w_{g_{\rm s},r_{z^*}}^{\tau_i,1}(0)
<
w_{g_{\rm s},r_{\eps}}^{\tau_i,1}(0),
\]
which is exactly the stated positive work difference between matched retained states
under \cref{app:matched,app:wri}.
\end{proof}

Minimising total work across the training tasks can increase work on individual tasks. To prevent this, restrict the same finite enumeration to candidates satisfying
$w_{g_{\rm s},r_z}^{\tau_i,1}(0)\le
w_{g_{\rm s},r_{\eps}}^{\tau_i,1}(0)$ for every $i$.
The feasible set contains the baseline $z=\eps$. At input bound zero, each state has one run on each task, and that run returns an accepted response. Its work is therefore the least work for both binary levels, so \cref{app:wri-to-grasp} gives weak grasp dominance at this input bound on every tested task. Neither version establishes optimality outside the bounded retained-state class or the evaluated task family.

For the length-penalised search in \cref{sec:wri-search}, extend the same construction to all finite candidate fields and fix a rational $\gamma>0$. Every candidate still terminates through the baseline. Since work is nonnegative, any $z$ with $|z|>W(\eps)/\gamma$ has $W(z)+\gamma|z|>W(\eps)$. The acquisition procedure first computes $W(\eps)$, then evaluates all codes of length at most $\lfloor W(\eps)/\gamma\rfloor$. A minimiser $z^*$ over this finite set is a global minimiser of the penalised objective, and
\[
W(z^*)+\gamma|z^*|\le W(\eps),
\qquad W(\eps)-W(z^*)\ge\gamma|z^*|.
\]
The conclusion concerns this fixed execution procedure and the training tasks.

\clearpage
\subsection{Preserving capability}

The example with shared storage and participant memory shows how a system can retain a capability while its participants are replaced (\cref{app:retention,sec:boundaries}). The switching results then bound the length of
a replacement checkpoint, and the work of producing it, when the checkpoint must preserve functions that later work
depends on, as discussed in \cref{sec:switching}.

\subsubsection{Maintaining capability}

Fix $k\ge1$ and $z\in\{0,1\}^k$.  There is a fixed transition-complete system
$g_k$ whose valid states are pairs $r=(R,I)$, where each slot is either
blank ($\bot$) or a $k$-bit word and, when both are nonblank, they agree.  The same system
stores and emits whichever $k$-bit word its state contains,
so $z$ is not hard-coded in $g_k$.  The persistent slot $R$ represents shared storage, and
the slot $I$ represents the current participant's memory.  Only the empty task input
can succeed.

If both slots are nonblank, $g_k$ emits $I$.  If exactly one slot is nonblank,
it copies that word into the missing slot and emits it. If both are blank it
fails.  For a fixed constant $c$, choose the implementation so serial emission costs $k+c$ work and a
copy of one missing slot contributes exactly $k$ additional work.  For
$\tau_z=(\eps,\{z\})$, define the states with both copies, only the shared copy, only the participant's copy and no copy,
\[
r_{11}=(z,z),\quad r_{10}=(z,\bot),\quad r_{01}=(\bot,z),\quad r_{00}=(\bot,\bot).
\]

Define
\[
\operatorname{Turn}(R,I)=(R,\bot),\qquad
\operatorname{Del}(R,I)=(\bot,I).
\]
Thus turnover of the participant gives $\operatorname{Turn}(r_{11})=r_{10}$, deletion of the shared copy gives $\operatorname{Del}(r_{11})=r_{01}$, and
loss of both copies gives $r_{00}$.

\begin{example}[Maintaining capability through turnover]
\label{app:retention}
For the system $g_k$ and task $\tau_z$ above,
\[
w_{g_k,r_{11}}^{\tau_z,1}(0)=k+c,
\qquad
w_{g_k,r_{10}}^{\tau_z,1}(0)=w_{g_k,r_{01}}^{\tau_z,1}(0)=2k+c,
\]
while level one is unattainable from $r_{00}$ at input bound zero.
A successful run from either one-copy state reconstructs $r_{11}$.
Repeated run--turnover cycles preserve the target capability across
arbitrarily many participant replacements, provided each post-turnover run
is allowed at least $2k+c$ work.  The $k$-work gap is implementation-relative.
\end{example}
\begin{proof}
From $r_{11}$ the deterministic successful path only emits the $k$ stored bits,
so its work is $k+c$.  From either one-copy state it first performs exactly
$k$ copy transitions and then the same emission, giving $2k+c$.  From $r_{00}$
there is no word to emit and the run fails.  Only the empty string is admitted
at input bound zero, so these path counts are the
least-work values.  A run from either one-copy state fills the blank slot and
therefore leaves the terminal state $r_{11}$.  The turnover and deletion identities are immediate.
Induction over run--turnover cycles gives persistence through arbitrarily many
participant replacements.
\end{proof}

The example makes explicit the point in \cref{sec:boundaries} that a budget
must also cover the work of maintenance. For this process, a budget
$k+c\le w<2k+c$ permits a response from $r_{11}$ but not after either copy
is lost. In the restoring procedure, with at least $2k+c$ work per run, one remaining copy allows the system
to answer and restore the missing copy, so the capability persists under
$\{\operatorname{Turn},\operatorname{Del}\}$ in the sense of \cref{sec:boundaries},
with $\mathcal S=\{r_{11}\}$.

Under turnover alone restoration is not needed. A
variant that emits the shared copy without restoring the participant's copy
uses $k+c$ work after turnover and persists under $\{\operatorname{Turn}\}$
with $\mathcal S=\{r_{11},r_{10}\}$, but a later deletion of the shared copy then
leaves $r_{00}$.

With one disturbance between completed runs, the original procedure uses an extra $k$ work to restore the lost copy after either turnover or deletion. Both copies are then available before the next disturbance.
Each copy saves $k$ work given the other,
$\Delta_{\rm work}(r_{11}{:}r_{10};\tau_z,1,0)=\Delta_{\rm work}(r_{11}{:}r_{01};\tau_z,1,0)=k$,
while losing both copies removes attainability.

More generally, removing a component without changing performance does not
by itself locate the capability. Another stored copy or a different
procedure may preserve performance.

\subsubsection{Replacing components}
\label{app:cluster-switching}
In the preceding example, shared storage and participant memory preserve capability within one system. Here useful
structure can remain externally accessible, and switching changes what
is still available to the replacement. The replacement may use another
implementation or an adapter. The bounds below concern the required responses
it must give, whatever method produces them.

The service below returns one bit in response to each required query. Its
replacement checkpoint encodes the whole retained state of the replacement,
including any program and stored data. The checkpoint must give those bits using
only the resources that remain available after switching. We first bound the checkpoint's encoded length,
then use the work required to emit it to bound the migration work.

Fix one transition-complete finite-description system
$g_{\rm port}$ whose valid checkpoints have the form
$r=\code{f,\sigma}$, where $f\in\{A,B\}$ names the architecture and $\sigma$ is the rest of the checkpoint.  For $T\ge1$, let
$R=\{i_1<\cdots<i_m\}\subseteq\{1,\ldots,T\}$ be the indices of later
queries that are actually required, and let
$z_R=z_1\cdots z_m\in\{0,1\}^m$ be their required response table, where $z_j$ is the
bit required by query $i_j$. For each
finite persistent external state $y$, assume the query contexts
$x_i(T,R,y)$ are uniformly computable from $(T,R,y)$, and let the work bound $W_T$ for each query be computable from $T$.  Only the empty task input
is used in these service tests. In the notation of \cref{sec:switching}, the required responses are
$\Phi=\{(x_{i_j}(T,R,y),z_j):j\le m\}$, with input bounds zero and work bounds $W_T$.
Here $T$ counts possible queries and $m=|R|$ counts required queries. The main-text case requiring every response has $m=T$.

For $f\in\{A,B\}$ define the function-preserving checkpoint set
\[
S_f(T,R,z_R;y)=\Bigl\{\code{f,\sigma}\in\mathcal R_{g_{\rm port}}:
\begin{array}{l}
 g_{\rm port}(\code{f,\sigma};\code{x_{i_j}(T,R,y),\eps})=z_j,\\
 \Work_{g_{\rm port}}(\code{f,\sigma};
       \code{x_{i_j}(T,R,y),\eps})\le W_T
 \quad\text{for every }j\le m
\end{array}
\Bigr\}.
\]
Membership in $S_f$ is defined using $z_R$. The system receives no separate query-input field containing $z_R$. It can obtain the table through the checkpoint or
persistent external state. The queries and their work bound are
reconstructible from $(T,R,y)$ without an additional $z_R$ input.

\begin{lemma}[Checkpoint length needed to preserve the required responses]
\label{app:portability-representation}
There is a fixed constant $c$ such that every
$r_B\in S_B(T,R,z_R;y_B)$ satisfies
\[
|r_B|\ge \C(z_R\given T,R,y_B)-c\log(T+2)-c.
\]
\end{lemma}
\begin{proof}
Given $(T,R,y_B)$ and a literal $r_B$, one fixed program on $U_0$
reconstructs $W_T$ and the contexts $x_{i_j}(T,R,y_B)$, simulates $g_{\rm port}$ for the bounded runs, and collects the emitted bits
in increasing index order.  If $r_B\in S_B(T,R,z_R;y_B)$, the recovered string
is exactly $z_R$.  Hence
\[
\C(z_R\given T,R,y_B)\le |r_B|+O(\log(|r_B|+2)).
\]
Since $\C(z_R\given T,R,y_B)\le m+O(1)\le T+O(1)$, either
$|r_B|\ge T$, in which case the claimed inequality is immediate after adjusting the
fixed constant, or $\log(|r_B|+2)=O(\log(T+2))$.  Rearranging gives the result.
\end{proof}

The bound therefore also covers any decoder, emulator or adapter in the
checkpoint. A small checkpoint is possible only
if the required response table has a sufficiently short description given
$T$, $R$ and the post-switch external state. In the construction of \cref{app:portability}, keeping the table
externally available permits a constant-size checkpoint that routes queries
to it.

The next theorem compares continuing with external access, making a direct switch,
and switching while preserving the required service. The representation bound of \cref{app:portability-representation} gives a lower bound on the work required for a switch
that preserves the service. A copying procedure shows that such a switch is feasible
and gives an upper bound. Constant-work continuation also
requires an input format that allows the selected mode to be read without scanning
irrelevant external data. The proof specifies that format.

\begin{theorem}[Switching cost when required responses must be preserved]
\label{app:portability}
There are a fixed system $g_{\rm port}$ of the above form, a fixed
transition-complete acquisition system $M$, uniformly computable query
contexts, and a computable service-work bound $W_T$ with the following
property. Consider any instance $(T,R,z_R,y_A,y_B)$ of the construction
in which $y_A$ begins with the $T$-bit indicator string of $R$, followed by
the bits of $z_R$ in increasing index order. The construction uses one fixed
constant-size A-routing checkpoint that reads this format, answers each
required query by lookup and belongs to $S_A(T,R,z_R;y_A)$.
Suppose the developed acquisition state $a_A$ contains the realised $T$-bit
table of responses under architecture A in the serial representation of
\cref{app:model}, with the entries at the indices in $R$ forming $z_R$.
Then the three ordinary binary acquisition tasks from $a_A$, namely
function-preserving A continuation $\tau_A^{\rm cont}$, a direct B switch
$\tau_B^{\rm bare}$, and a function-preserving B switch
$\tau_B^{\rm pres}$, satisfy
\[
w_{M,a_A}^{\tau_A^{\rm cont},1}(0)=O(1),\qquad
w_{M,a_A}^{\tau_B^{\rm bare},1}(0)=O(1),
\]
and, for every finite input bound $\lambda$,
\[
w_{M,a_A}^{\tau_B^{\rm pres},1}(\lambda)
\ge \C(z_R\given T,R,y_B)-O(\log(T+2)).
\]

A fixed migration routine also achieves $\tau_B^{\rm pres}$ in $O(T)$ work
at input bound zero.  Consequently, when
$R=\{1,\ldots,T\}$ and
\[
\C(z_R\given T,R,y_B)\ge T-O(1),
\]
function-preserving continuation under A has $O(1)$ acquisition work whereas
function-preserving switching to B has $\Theta(T)$ acquisition work.  The
direct B switch takes $O(1)$ work. The preserving-switch lower bound and
copying upper bound do not require the assumption about $y_A$.
From the decision state before development, adopting A or B takes the same
$O(1)$ work. If $y_B$ contains the same required indicator string and response
bits in the external-routing format, a fixed constant-size B-routing
checkpoint belongs to $S_B(T,R,z_R;y_B)$.
\end{theorem}
\begin{proof}
Implement $g_{\rm port}$ with a fixed interpreter, a program that reads the checkpoint and executes its instructions. Give checkpoints decidable syntax and
external-routing, literal-table and general-program modes in both
architectures $f=A$ and $f=B$. A literal checkpoint stores $T$, a string of $T$ bits whose $i$th bit is $1$ exactly when $i\in R$, and the $m$ required response bits in increasing order of their indices. Its query routine reads the requested index and scans the string identifying $R$, advancing through the stored responses whenever it encounters a $1$, until it reaches the requested response. Use input encodings in which
the service input starts with $(T,i)$, followed by the $(R,y)$
fields, with length metadata for unused suffix fields placed after them.
The literal lookup reads the stored table without inspecting $y$. For the external-routing instance,
$y_A$ begins with the same string identifying $R$, followed by the required response bits, so that routine performs
the corresponding lookup there. The routine leaves other trailing external data unread. A fixed multitape implementation can parse these $O(T+1)$ relevant
symbols and perform the scans with binary counters within
$W_T=K(T+1)^3$ transitions for a sufficiently large fixed computable
integer $K$. Choose this bound once for both routines. No such guarantee
is required for arbitrary program-mode checkpoints.

Begin each acquisition task's public context with a fixed-length tag selecting continuation, direct switching or preserving switching, followed by the relevant $(T,R,y_A,y_B)$ fields. Their accepted sets are respectively
$S_A(T,R,z_R;y_A)$, all valid B checkpoints, and $S_B(T,R,z_R;y_B)$.
On the continuation tag, $M$ emits the fixed A-routing checkpoint,
which succeeds with $y_A$ by hypothesis. On the direct-switch tag it emits
a fixed constant-size empty B checkpoint. Neither branch needs to scan
the unused context suffix, so both use bounded work and an empty task
input.

For preserving switching, every successful run emits some
$e\in S_B(T,R,z_R;y_B)$. By
\cref{app:portability-representation},
\[
|e|\ge \C(z_R\given T,R,y_B)-O(\log(T+2)).
\]
Transition-completeness gives $|e|\le\Work_M$ for every successful run,
so the same lower bound holds for least work at every finite input
bound.

For the upper bound, the developed state $a_A$ stores the
realised table under architecture A in serial layout. A fixed migration branch scans
that table and the $T$-bit string identifying $R$, emits that string with the B literal-mode header, and copies the selected response bits in order. These scans and output take
$O(T)$ work at input bound zero. The lookup of the resulting checkpoint returns
every $z_j$ at $i_j$ within the chosen $W_T$, independently of $y_B$.
The checkpoint therefore belongs to $S_B(T,R,z_R;y_B)$, so the migration succeeds. The incompressible
case where every response is required combines this upper bound with the representation
lower bound.

For the initial-choice comparison, add two symmetric initial public tags
at the undeveloped decision state. They emit constant-size A- and B-routing
checkpoints with their paths padded to the same fixed cost. The subsequent
construction is unchanged by exchanging the branch names. These initial
choices precede the development history, whose work belongs to acquisition.
\end{proof}

The least-work values of $\tau_B^{\rm bare}$ and $\tau_B^{\rm pres}$ are $w_{\rm direct}$ and $w_{\rm preserve}$ of \cref{sec:switching}. When continuation is the chosen baseline, $w_{\rm comparator}$ is the least work for $\tau_A^{\rm cont}$. The theorem compares alternatives after the history under architecture A. Continuing
under A and switching to B are different acquisition tasks, not two
states in a matched WRI comparison. The initial choice can have equal
cost on both sides. Under the same later inputs, either choice can lead
to analogous functions that later work depends on. After development under A, preservation is costly
when less required information remains accessible under B. An adapter
that preserves external access can make the representation bound small.
Keeping that resource available still incurs horizon costs
\textup{[\cref{app:horizon-account}]}.

\subsubsection{Short descriptions under query work limits}

The preceding bound uses description complexity without limiting the work
needed to reconstruct the output. As discussed in \cref{sec:switching},
reconstructing the responses from a short description can take too much
work to meet a query limit. The construction below uses the same distinction as classical time-bounded
incompressibility results \cite{daylight2009,livitanyi2019} and imposes the
limit on each required response. With $T$ one-bit queries and empty external
state, B checkpoints shorter than $\lfloor T/2\rfloor$ bits are fewer than the $2^T$
possible response tables, so some table is not returned by any of them within the
per-query work limit. A fixed program finds the first such table from $T$ by finite
search, so the table has a description of constant length given $T$.

\begin{theorem}[Short descriptions with costly function-preserving replacement]
\label{app:budget-preservation}
Under the efficient binary encodings of \cref{app:model}, there exist a
fixed transition-complete interpreter with A and B checkpoint
modes, a fixed transition-complete migration system $M$, a computable
polynomial service-work bound $B_T$, and a uniformly computable family
$z_T\in\{0,1\}^T$ such that
\[
\C(z_T\given T)=O(1).
\]
The service contexts give only $T$, the query index and the
external state. Each service query admits only the empty task input and starts
from the same checkpoint, so updates made during one query are not available to
later queries. With external state
$y_A=z_T$, one constant-size A-routing checkpoint answers all $T$ service
queries within $B_T$. With $y_B=\eps$, every valid B checkpoint that answers
those same queries within $B_T$ has length at least $\lfloor T/2\rfloor$.
The B interpreter admits literal tables and arbitrary programs.

From a developed migration state $a_T$ containing $z_T$, there are ordinary
binary acquisition tasks $\tau_A^{\rm cont}$, $\tau_B^{\rm bare}$ and $\tau_B^{\rm pres}$ for A continuation, direct B switching, and
function-preserving B switching such that
\[
w_{M,a_T}^{\tau_A^{\rm cont},1}(0)=O(1),\qquad
w_{M,a_T}^{\tau_B^{\rm bare},1}(0)=O(1),
\]
whereas, for every finite migration input bound $\lambda$,
\[
w_{M,a_T}^{\tau_B^{\rm pres},1}(\lambda)=\Theta(T).
\]
\end{theorem}
\begin{proof}
Fix a serial interpreter with decidable checkpoint syntax whose
valid B checkpoints include a literal-table mode and a general program mode. On query $i$ with size $T$,
the literal mode reads its stored table and returns its $i$th bit. Program
mode executes its stored program on $(T,i)$. An A-routing mode reads the
corresponding bit of external state. All transitions are counted. Choose a
computable polynomial $B_T$ bounding the work of every literal-table and A-routing
query under these fixed encodings. This choice is made before selecting
$z_T$.

Put $k_T=\lfloor T/2\rfloor$. Enumerate every valid B checkpoint whose
\emph{complete encoded length} is less than $k_T$. There are at most
$2^{k_T}-1$ such binary strings. For each, simulate all $T$ queries for at
most $B_T$ transitions per query. Keep the resulting $T$-bit response table
only when every query halts within the limit with a single-bit response.
Each deterministic checkpoint contributes at most one table. Since
$2^{k_T}-1<2^T$, at least one $T$-bit table is absent. Define $z_T$ as the
lexicographically first absent table.

The finite checkpoint enumeration, bounded simulations, and missing-table
selection are uniformly computable from $T$. One fixed program on $U_0$
therefore produces $z_T$, proving $\C(z_T\given T)=O(1)$. By construction,
no B checkpoint shorter than $k_T$ gives all its responses within $B_T$.
The bound covers short generators and compatibility programs because
program mode includes them.

Let $M$ receive a public tag selecting the acquisition task. From $a_T$,
its continuation branch emits the fixed A-routing checkpoint, and its
direct-switch branch emits a fixed empty B checkpoint. Both use constant work
and an empty task input. Its preserving branch can scan the table in $a_T$
and emit a literal B checkpoint in $O(T)$ work under the serial encoding
convention. That checkpoint meets every service bound by the choice of $B_T$.
Any successful migration that preserves the required service, irrespective of the task input, must
emit a B checkpoint of length at least $k_T$. Transition completeness makes
its work at least its output length. This gives the uniform $\Omega(T)$
lower bound. The empty-input copy branch gives the upper bound for every
$\lambda$.
\end{proof}

The B checkpoint that runs the short generator of $z_T$ has constant length, so
once $T$ is large it needs more than $B_T$ work on some query. The work of initially
computing $z_T$ belongs to acquisition. Migration begins with that table
already stored. Continued external access and a sufficiently large
replacement checkpoint meet the original service limit. With enough
additional work per query, the generator becomes usable too. The work limits and the resources that survive the switch determine which alternatives are feasible, while
\cref{app:horizon-account} compares their complete costs over a finite horizon.

\subsection{Comparing organisational arrangements}
\label{app:cluster-coordination}

The reporting example works through the organisational comparisons of \cref{sec:together}. We then bound the trials needed to assign workers with different capabilities to suitable roles.

\begin{example}[Establishing reporting roles]
\label{app:org-separation}
Two components run the same controller. On each task one receives a bit $u$ and the
other a bit $v$ across the common outer boundary, and only the empty task input is
admitted. Either component may report, and a response is accepted when it contains at
least one report and every report equals $u\oplus v$. The controller cannot use
component names or identifiers. Each component has a persistent role field, initially
unset, that can hold sender or reporter. In each communication round both components
emit their messages before either receives them. Delivery is reliable and a component
waiting for a message performs no operations. The next task begins after both have
finished. Since the schedule treats the components alike, deterministic steps on equal
task bits keep identical states identical, so complementary roles can be chosen only
from different task bits or from random choices \cite{angluin1980,itairodeh1990}.

Each executed instruction counts as one unit of work, and sending, delivering and
receiving a bit are separate operations. On every task each component reads its bit,
tests its role field, dispatches and finally finishes, which gives eight common
operations.
\begin{itemize}
\item \emph{Immediate exchange.} Without roles, both components send their bits,
compute the parity and report, in $18$ operations.
\item \emph{Retained roles.} The sender transmits its bit, and the reporter computes and
reports the parity, in $13$ operations. The saving of five consists of one send,
delivery and receipt, one parity computation and one report.
\item \emph{Roles from different bits.} Both components send their bits and compare
them. Equal bits leave the roles unset, and both report $0$, in $18$ operations. When
the bits differ, the component holding $0$ writes sender and the component holding $1$
writes reporter. Both dispatch on the new role and only the reporter reports $1$, in
$21$ operations.
\item \emph{Random proposals.} Each component draws a fair private bit, proposes sender
or reporter accordingly and sends the proposal with its task bit. Equal proposals leave
the roles unset, and both report the parity. Different proposals are stored, and only
the proposed reporter reports. Compared with immediate exchange, an attempt adds two
draws, the send, delivery and receipt of two proposal bits and two comparisons. A
successful attempt also writes the two roles, and dispatching on them costs as much as
the second parity computation and report that it avoids. An attempt therefore takes
$28$ operations when the proposals are equal and $30$ when they differ.
\end{itemize}

At the start of the horizon each component reads a common two-bit instruction that
selects the plan. It is an external input other than the task input and adds two
operations. All plans start from unset roles with the same storage, and their external
inputs have the same length apart from the random bits drawn under random proposals.
Over $H$ tasks, immediate exchange uses $2+18H$ operations. If the bits first differ on
task $k$, establishing roles from them uses $13H+5k+5$, a saving of $5(H-k)-3$, and if
they never differ it uses $2+18H$. Under random proposals, task $t$ begins without roles
with probability $2^{-(t-1)}$, so the expected total work is
\[
2+13H+16\sum_{t=1}^H2^{-(t-1)}=2+13H+32\bigl(1-2^{-H}\bigr).
\]
The expected saving $5H-32(1-2^{-H})$ is negative for $H\le6$ and positive for $H\ge7$.
If the proposal operations and role writes are counted as acquisition work, each task
that begins with roles has stage saving $5$ in the notation of \cref{sec:horizon}, and
each attempt adds ten operations of acquisition work, with two more when it succeeds.
The total varies between histories. It is $13H+15j+4$ when roles are first established
on task $j$ and $2+28H$ when no attempt succeeds, so a budget met on every history must
be at least $28H+4$ (\cref{app:horizon-account}).

For the organisational family of \cref{app:orgfamily}, fix immediate exchange as the procedure used without
roles. The states with complementary and with unset roles are then matched states
differing only in their role fields. Their least work at input bound zero is $13$ and
$18$ on every task, so $G_{\rm org}=5$ and the retained roles are work-reducing
information (\cref{app:wri}). Budgets from $13$ to $17$ admit the organised state but
not the unorganised one. A pooled routine that reads both bits, computes and reports the
parity and finishes uses $5$ operations, so $T_{\rm coord}=8$. These counts belong to
this instruction table.

Without communication, including shared-state, timing and termination signals, each
component's behaviour depends only on its own bit and its own random draws. A report by
one component then stays the same when only the other bit changes, although the parity
changes. No fixed isolated implementation therefore succeeds on all four input pairs.
\end{example}

A hidden bijection $\psi$ assigns one of $n\ge1$ roles to each of
$n$ workers. A trial asks whether $\psi(i)=j$ for a chosen worker $i$ and role $j$ and
returns a reliable yes or no. A separate worker action performs a requested role. The
public context contains $n$ and a finite description of the function that answers the
trials. This description allows the reference machine to compute the assignment. An
acquisition procedure uses it only through trials, starts without assignment-specific
information and admits only the empty task input. It may keep every answer and use any
deduction from them. For a set $R$ of $m$ required roles, where $0\le m\le n$, consider
deterministic procedures that return the correct worker for every role in $R$ for every
$\psi$. The work rule counts one unit for each trial and each worker action. Control,
inference, encoding and output have zero cost under this rule.

Once these workers have been obtained, the evaluated system reads its stored routing
table and sends a request to the worker assigned to the requested role. It admits only
the empty task input. Let $\Gamma_R$ require a correct action in each role in $R$ within
one unit of work. A routing table meets $\Gamma_R$ exactly when it gives the correct
worker for every role in $R$. Let $\tau_{\rm acq}$ accept those tables, with the trial
description and $R$ as its public context. When $R$ contains every role, a qualifying
table gives the whole assignment.

\begin{proposition}[Discovering which worker can perform each role]
\label{app:division-of-labour}
\begin{enumerate}[label=(\roman*)]
\item A state retaining a table that is correct on $R$ meets every requirement in
$\Gamma_R$ in one worker action. Performing one action in every role of $R$ takes $m$
units of work.
\item Every acquisition procedure in this class uses at least $m(2n-m-1)/2$ trials for
some assignment, which is $n(n-1)/2$ when $m=n$. For any fixed such procedure $(A,a)$,
at that assignment its acquisition depth is at least this number at every margin.
Testing the possible workers for each role of $R$ in turn attains this worst-case trial
count. The reference complexity of a qualifying table is bounded by a constant
independent of $n$, $R$ and $\psi$.
\item If the assignments of the other workers are retained and $k$ workers are replaced
by workers assigned to the $k$ freed roles in an unknown order, identifying the new
assignment has worst-case trial count $k(k-1)/2$.
\end{enumerate}
\end{proposition}
\begin{proof}
For (i), the table selects the correct worker for any requested role in $R$.

For the lower bound in (ii), consider the bipartite graph of worker--role pairs not yet
excluded, which initially contains all $n^2$ pairs. Answer no whenever deleting the
queried pair leaves a perfect matching, and yes otherwise. A yes answer concerns a pair
in every remaining perfect matching, so every perfect matching of the current graph is
consistent with all answers given. The deterministic procedure returns the same table
for all of them. When a correct procedure finishes, they therefore all give each role in
$R$ the same worker. Fix one of them, $M$, and two workers $i$ and $i'$ at least one of
whose roles under $M$ lies in $R$. If both pairs $(i,M(i'))$ and $(i',M(i))$ remained,
exchanging the roles of $i$ and $i'$ in $M$ would give another remaining perfect
matching with a different worker for a role in $R$. One of the two pairs was therefore
deleted after a trial. A pair $(i,M(i'))$ determines $i$ and $i'$, so different choices
of the two workers require different deleted pairs. Their number is
\[
\binom n2-\binom{n-m}2=\frac{m(2n-m-1)}2,
\]
and $M$ is an assignment consistent with all the answers on which this many trials were
made. For $m=n$ the remaining graph has a unique perfect matching, and the count agrees
with the classical bound that such a bipartite graph has at most $n(n+1)/2$ edges
\cite{lovaszplummer1986}. Each acquisition procedure admits only the empty input, so its
run at $M$ gives its least work at every input bound and hence its depth at every
margin.

For the upper bound, take the roles of $R$ in turn. When the workers for $\ell-1$ of
them are known, at most $n-\ell+1$ workers remain possible for the $\ell$th role, and at
most $n-\ell$ trials identify its worker, since the last possibility can be inferred.
Summing over $\ell\le m$ gives $m(2n-m-1)/2$. A fixed reference program performs this search using the
trial description and $R$ in the context and emits the routing table. This proves the
constant description bound.

For (iii), the retained assignments fix $n-k$ pairs. The new workers and freed roles
form the same identification problem of size $k$ with all $k$ roles required.
\end{proof}

\subsection{Evaluating success}
\label{app:cluster-evaluation}

These results distinguish the work required to reach a score from whether
success under that score meets the intended target (\cref{sec:judging}).

For the proxy evaluator $\widehat\perf$ and intended evaluator $\perf_\tau$ of \cref{sec:judging},
write $\widehat S_{\hat q}$ and $S_q^*$ for their accepted level sets (\cref{app:accepted-levels})
at a proxy level $\hat q$ and a goal level $q$, which need not agree.
Each comparison fixes the system, state and public context unless stated otherwise.

The inclusion $\widehat S_{\hat q}\subseteq S_q^*$ discussed in \cref{sec:judging} has a direct work
consequence. At a fixed input bound, restricting the accepted responses
cannot reduce the required work. Depth sets that input bound from the shortest
reference description of an accepted response, so changing the accepted
set can also change the input bound.

\begin{proposition}[Bounds for nested accepted sets]
\label{app:transfers}
Fix a system and state, a common public
context $x$, a proxy level $\hat q$, a goal level $q$ and $\kappa\in\mathbb N$, and suppose $\widehat S_{\hat q}$ and $S_q^*$ are nonempty. As in \cref{app:accepted-set}, $\C(S\given x)$ is the length of a shortest reference description of a response in $S$, and $d^S$ is the depth of the binary task $(x,S)$.
If $\widehat S_{\hat q}\subseteq S_q^*$ and
\[
\C(\widehat S_{\hat q}\given x)-\C(S_q^*\given x)\le\kappa,
\]
then, for every $s\in\mathbb N$,
\[
d^{S_q^*}(s+\kappa)\le d^{\widehat S_{\hat q}}(s),
\]
so a depth lower bound for $S_q^*$ at margin $s+\kappa$ transfers to $\widehat S_{\hat q}$ at margin $s$.
For every input bound $\lambda$, inclusion alone gives
$w^{\widehat S_{\hat q}}(\lambda)\ge w^{S_q^*}(\lambda)$.
\end{proposition}
\begin{proof}
Target inclusion gives
$w^{\widehat S_{\hat q}}(\lambda)\ge w^{S_q^*}(\lambda)$ for every input bound.
Write $\widehat C=\C(\widehat S_{\hat q}\given x)$ and
$C_*=\C(S_q^*\given x)$.  Since
$\widehat C+s\le C_*+s+\kappa$ and least work is nonincreasing in
input bound,
\[
\begin{aligned}
d^{\widehat S_{\hat q}}(s)
&=w^{\widehat S_{\hat q}}(\widehat C+s)\\
&\ge w^{S_q^*}(\widehat C+s)\\
&\ge w^{S_q^*}(C_*+s+\kappa)=d^{S_q^*}(s+\kappa).
\end{aligned}
\]
\end{proof}

Since $\C(S_q^*\given x)\ge0$, the length of any known reference description
of a member of $\widehat S_{\hat q}$ is an admissible $\kappa$.

\Cref{app:transfers} relates success under two evaluators. We now ask whether successive levels of one evaluator can require increasing work. For a fixed system, starting state and margin, the next construction accepts at each level only responses that the system cannot return within that level's input and work bounds. The accepted responses nevertheless have short reference descriptions, as in the classical existence of deep strings \cite{bennett1988}. When the system can execute reference programs with a fixed input overhead and the margin covers that overhead, the constructed levels have finite depth, as shown after \cref{app:public}.

\begin{theorem}[A demanding evaluator for one system]
\label{app:demanding}
Fix a system--state pair $(g,r)$ and public context
$x$, and suppose bounded-work runs of $g$ from $r$ are
effectively exhaustible, meaning that an algorithm terminates with a complete list of halting runs within given bounds on input length and work (\cref{app:exhaustible}). Fix $s\in\mathbb N$ and a computable strictly
increasing sequence $T_1<T_2<\cdots$. There is a finite-description graded
task $\tau_{\mathrm{dem}}$ with public context $x$, performance
set $\Nat$ in its usual order, and a constant $K$ such that, for every $n\ge1$,
\[
d_{g,r}^{\tau_{\mathrm{dem}},n}(s)>T_n,
\qquad
\C_{\tau_{\mathrm{dem}}}(n\given x)\le K+\ell(n).
\]
Here $\ell(n)=O(\log(n+2))$ is the length of the self-delimiting code of $n$.
\end{theorem}
\begin{proof}
The input bound used for depth depends on the length of a shortest reference description of an accepted response, so we first bound that length by the length of the construction's own description.
Fix a finite construction algorithm with an integer parameter $D$.  If its
fixed part, which includes the listing algorithm, the sequence $(T_n)$ and $s$, has description length $c$, choose $D$ so large that
$c+\ell(D)\le D$, freeze this value and set $K=D$. Such a choice exists by
the logarithmic integer-code convention in \cref{app:model}. Put
\[
L_n=K+\ell(n)+s.
\]
Effective bounded-work exhaustibility makes the finite set
\[
B_n=\{g(r;\code{x,p}): |p|\le L_n,
\text{ the run halts within work }T_n\}
\]
computable.

A response of level $n$ is also accepted at every lower level, so choose $z_n$ as the first string in length-lexicographic order for which
$y_n=\code{n,z_n}$ is outside $B_1\cup\cdots\cup B_n$. The excluded set is
finite, so this effective search terminates. The evaluator assigns level
$m$ exactly to $y_m$, and zero to every other response. It first checks for
a canonical tuple encoding with integer leading field $m\ge1$, returning
zero otherwise. On a valid encoding it computes $y_m$ for the leading field $m$ and
checks equality. Thus it is total, computable and finitely described.
The same fixed construction, with frozen parameter $K$ and the
self-delimiting code of $n$, outputs $y_n$, proving
\[
\C_{\tau_{\mathrm{dem}}}(n\given x)\le K+\ell(n).
\]

Fix $n$.  Every response of performance at least $n$ is some $y_m$ with
$m\ge n$, and each such response was chosen outside $B_n$.  Thus no
task input of length at most $L_n$ reaches level $n$ within work
$T_n$.  The input bound $\C_{\tau_{\mathrm{dem}}}(n\given x)+s$ in the definition of depth is at most
$K+\ell(n)+s=L_n$, hence
$d_{g,r}^{\tau_{\mathrm{dem}},n}(s)>T_n$.
\end{proof}

The preceding construction fixes one system--state pair and one margin. We can extend the construction to an effectively enumerated family of pairs and every margin, using a single evaluator. For each fixed pair and margin, its growing work lower bounds hold from a starting level that depends on that pair and margin.

\begin{proposition}[A demanding evaluator for a family of systems]
\label{app:public}
Fix an effective enumeration $(g_1,r_1),(g_2,r_2),\ldots$ of
finite-description system--state pairs with uniformly
effectively exhaustible bounded-work runs, a public context $x$, and a
computable strictly increasing unbounded schedule $(T_n)$.  There is one
finite-description graded task $\tau_{\rm pub}$ with public context $x$
and performance set $\Nat$ in its usual order, and a constant $K$ such that
for every pair $(g_i,r_i)$, every margin $s$, and every
$n\ge\max\{i,s,1\}$,
\[
d_{g_i,r_i}^{\tau_{\rm pub},n}(s)>T_n,
\qquad
\C_{\tau_{\rm pub}}(n\given x)\le K+\ell(n).
\]
Hence, for every fixed pair, margin and work threshold, every sufficiently
high level has depth above that threshold.
\end{proposition}
\begin{proof}
The construction follows the proof of \cref{app:demanding} with two changes. At stage $n$ it considers the first $n$ pairs, and its input bound covers the one used for depth at every margin up to $n$, so each fixed pair and margin is covered from some stage onwards. Choose $K$ as in that proof, for a construction whose fixed part includes the uniform listing algorithm and $(T_n)$ but no margin. Put
$L_n=K+\ell(n)+n$.  Uniform exhaustibility gives a finite computable set
$B_n$ containing every response obtainable by any of the first $n$ pairs
using task inputs of length at most $L_n$ and work at most
$T_n$. Choose $z_n$ as the first string in length-lexicographic order
for which $y_n=\code{n,z_n}$ is outside $B_1\cup\cdots\cup B_n$.
The finite exclusion makes this search terminate.

Define the evaluator from these responses as in the proof of \cref{app:demanding}, assigning level $m$ exactly to $y_m$ and zero otherwise. It is total and finitely described. The same fixed construction outputs $y_n$ from the self-delimiting
code of $n$, giving the stated complexity bound.

Now fix $i,s$ and $n\ge\max\{i,s,1\}$.  As in the previous proof, the input bound in the definition of depth is at
most $K+\ell(n)+s\le L_n$.  Every response of level at least $n$ is some
$y_m$ with $m\ge n$, and every such response was chosen outside $B_n$.
Since $i\le n$, the outputs of pair $i$ within input length $L_n$ and work $T_n$
are contained in $B_n$.  Hence level $n$ cannot be attained within work
$T_n$ at that input bound, proving the depth inequality.  The
unboundedness of $(T_n)$ gives the final claim about sufficiently high levels.
\end{proof}

Under the following interface condition, the depths bounded in \cref{app:demanding,app:public} are finite whenever the margin covers a fixed input overhead. Suppose a pair $(g,r)$ admits, for every reference program $a$, an input $p_a$ of length at most $|a|+c$, for a fixed $c\in\Nat$, and returns $U_0(a,x)$ in finite counted work whenever that reference computation halts. At any level $q$ with nonempty accepted set $S_q^\tau$, choose a shortest successful reference program. If $s\ge c$, its translated input fits the bound $\C_\tau(q\given x)+s$, so $d_{g,r}^{\tau,q}(s)<\infty$. The condition holds with $c=0$ when $g$ runs its task input as a program on $U_0$ and work counts the steps, as in the recovery of logical depth in \cref{sec:machine}. Bounded-work runs are then also effectively exhaustible. This applies to the pair in \cref{app:demanding} when it has this property, and to any such pair in \cref{app:public}. For those pairs and margins, the displayed lower bounds concern finite depths. When $(T_n)$ is unbounded, as in \cref{app:public}, these finite depths tend to infinity with the required level.

These constructions give work lower bounds that grow with the required score. Cumulative deepening, as defined in \cref{sec:deepening}, additionally requires a history in which earlier achievements reduce the work of later ones and the later targets have finite, increasing depths from the original starting condition. \Cref{app:prefix-history} gives such a history for every chosen finite horizon.

\subsection{Recursive improvement}
\label{app:cluster-improving}

We ask whether a state with greater capability also reduces the work needed for further improvement (\cref{sec:improving}). In the example below, each round before the last builds part of a search index that reduces the work of the next round. The proposition that follows bounds the work needed for one fixed program to execute changing improvement procedures stored as code.

\begin{example}[Constructing a search index]
\label{app:search-index-improvement}
A search can organise data in a way that helps later searches. Suppose a catalogue lists every four-bit word once in an arbitrary order and a procedure must find the position of $0000$. It divides the entries by their first bit, divides the group beginning with $0$ by the second bit and continues in this way until $0000$ is isolated, keeping both groups at every division. A later search for $1000$ can start from the group beginning with $1$, which the first search set aside. Building a search structure only as far as queries require is known as deferred data structuring \cite{karp1988deferred}. Database cracking similarly reorganises stored data as a byproduct of query processing \cite{idreos2007cracking}.

Fix $d\ge3$ and $N=2^d$. The public context holds a catalogue listing every $d$-bit word once in an arbitrary order, with each bit stored separately. An entry's identifier is its position. A valid retained state consists of a fixed-width binary prefix tree, a request counter $k\in\{0,\ldots,d\}$ and $d$ answer slots, of which the first $k$ hold the identifiers found for the first $k$ requests. The active leaves of the tree hold linked lists of identifiers, and each lists exactly the identifiers whose words begin with its prefix. An internal node records a division already made, and nodes below a leaf are inactive. The search follows the bits of the requested word through internal nodes. At a leaf it reads the next bit of every listed entry, divides the list into two, keeps both and continues until the word is isolated. Work counts the executed instructions of a direct random-access-machine implementation, one for each read, write, arithmetic operation, register assignment or branch. Dividing a group of $m$ entries takes eight instructions for each entry and nine to prepare and mark the division, $8m+9$ in all.

For a domain task, the same search is run from a retained state $r$ with only the empty task input and returns the identifier of one given word. When every group on the word's path is already divided it takes $8d+5$ instructions, and each missing division adds its own $8m+9$. Let $\mathcal L(r)$ be the set of words that the search locates from $r$ within work $8d+5$, and $Q(r)=|\mathcal L(r)|$. Dividing an active group of two entries isolates two words and raises $Q$ by two, while dividing a larger group isolates none.

The improvement process reads its requests from the public context. From a state with counter $t-1<d$ it searches for the word $1^{t-1}0^{d-t+1}$, writes its identifier in answer slot $t$ and increments the counter, and from counter $d$ it fails. The requests are chosen so that each search after the first begins its divisions in the group that the previous search set aside. Let $r_0$ hold one undivided group and counter $0$, and let round $t$ begin at $r_{t-1}$ and produce $r_t$. The wrapper $g_\Delta$ of \cref{app:producer-relative} emits $\code{r,r',\eps}$ with empty procedure data. Let $\tau^\Delta$ be the binary task that accepts exactly when $r$ and $r'$ are valid, the counter of $r'$ exceeds that of $r$ by one, $r'$ keeps the answers of $r$, $\mathcal L(r)\subseteq\mathcal L(r')$ and $Q(r')=Q(r)+2$. Only the empty task input is admitted. Round $t$ divides the groups with prefixes $1^{t-1},1^{t-1}0,\ldots,1^{t-1}0^{d-t}$, of sizes $N/2^{t-1},\ldots,4,2$. Each round is accepted and gives $Q(r_t)=2t$. Each division reduces the work of the domain tasks whose words lie in the divided group and leaves the others unchanged. Each $r_t$ therefore weakly grasp-dominates $r_{t-1}$ on these tasks at every finite work budget \textup{[\cref{app:wri-to-grasp}]}.

Besides its divisions, a round takes $8d+11$ instructions for following the requested word through the $d$ levels and for bookkeeping. Emitting the producer and the successor as two fixed-width arrays takes $8R$ more, where the $R=3N+d$ words of a state hold its $2N-1$ tree nodes, $N$ list links, $d$ answer slots and counter. Building the undivided list of $r_0$ takes $10N+1$ instructions before the first round. Since the process is deterministic, its least work for the improvement from $r_{t-1}$ is
\[
W_t=w_{g_\Delta,r_{t-1}}^{\tau^\Delta,1}(0)=8R+8d+11+\sum_{i=1}^{d-t+1}\bigl(8\cdot2^i+9\bigr),
\qquad
W_t-W_{t+1}=\frac{8N}{2^{t-1}}+9>0\quad\text{for }1\le t<d.
\]
In the four-bit case the rounds read $30$, $14$, $6$ and $2$ catalogue bits, raise $Q$ to $2$, $4$, $6$ and $8$ and take $735$, $598$, $525$ and $484$ instructions, so the reductions are $137$, $73$ and $41$.

For $1\le t<d$, replace the components that round $t$ changes in $r_{t-1}$, namely its new divisions, the counter and answer slot $t$, by their values in $r_t$. This gives the valid state $r_t$, from which the same process with only the empty task input runs round $t+1$ with work $W_{t+1}\ne W_t$. The observation rule of \cref{app:reuse} for finite constructions records work, so the producing run depends on these components, which is the self-contribution condition of \cref{app:recursive-gain-def}. Rounds $1$ to $d-1$ therefore show positive recursive gain in the sense of \cref{app:recursive-gain-def}, and $r_0\mapsto r_1\mapsto\cdots\mapsto r_{d-1}$ is a finite instance of recursive improvement.

To find the source of each reduction, take $1\le t<d$ and replace the division of the group $1^{t-1}$ in $r_t$, the first made in round $t$, by an undivided list of that group. The counter, the answers and the other fields are unchanged and the nodes below become inactive. The next round must make this division again, so its work rises from $W_{t+1}$ to $W_t$ while it still meets the target relative to the altered producer. The whole reduction thus comes from the index built in round $t$. A comparison of building the index during searches with building it in advance must count the work of both complete plans (\cref{app:horizon-account,app:horizon-retention}).
\end{example}

Recursive improvement is assessed under one fixed outer system. An improvement procedure that changes from one stage to the next can be represented in such a system by storing it as code and executing it with a fixed interpreter. The proposition below adapts the standard universal-machine simulation \cite{arorabarak2009} to bound the interpreter's work when it reproduces a finite improvement history. The interpreter's own work can change which procedure uses less work. Any claimed saving must therefore be checked with both procedures executed by the same interpreter.

Fix in advance a system boundary, staged-input and exogenous-input
interfaces and finite checkpoint and successor encodings. Also fix a deterministic
small-step representation in which each
admissible local process is a finite transition table.  Here staged input is the task input delivered to the procedure at each stage, while
other external inputs use the exogenous-input interface. Transition tables and state encodings have syntax that can be checked in
polynomial time. Admissibility requires a finite represented history and well-formed successor
outputs. At stage $t$, the successor output encodes the next transition table
$F_{t+1}$ and state $r_{t+1}$, so a procedure can replace its own code. Each small step changes only a bounded number of symbols,
apart from literal strings explicitly contained in its transition table.
Count every outer transition as one work unit.

\Needspace{6\baselineskip}
\begin{proposition}[A fixed interpreter for changing improvement procedures]
\label{app:improvement-absorption}
For checkpoints
$c_t=\code{F_t,r_t}$, pairing the transition table $F_t$ of the stage-$t$ procedure with its state $r_t$, there is one finite-description deterministic
transition-complete interpreter $g_{\rm int}$ which reproduces every finite
admissible history
\[
(F_0,r_0)\leadsto\cdots\leadsto(F_T,r_T).
\]
At each stage it simulates $F_t$, emits the same well-formed successor output as $F_t$
and retains $c_{t+1}=\code{F_{t+1},r_{t+1}}$.  Stage work is bounded by one fixed polynomial in the lengths of the
transition table, current state, staged input, exogenous input and emitted
output, and in the number of local transitions.
\end{proposition}
\begin{proof}
Construct $g_{\rm int}$ as the universal interpreter for the fixed
small-step representation, using the efficient transition-table and
configuration encodings of \cref{app:model}. At each stage it decodes the
checkpoint and checks the transition table and state syntactically. It
initialises the simulated configuration and delivers staged and exogenous
inputs only through the original interface at their authorised
steps. Describing a realised input sequence for the comparison does not
make its future values available early to the local process. For each local
transition it scans the finite table, updates the encoded configuration and
copies any emitted symbol. A well-formed local halt emits the same successor
description and retains the decoded next checkpoint. Malformed successor
syntax produces a declared failure.

If the local computation uses $m$ transitions and its transition table
has length $L_F$, configuration length is at most a constant multiple of
the sum of $1+m(L_F+1)$ and the initial-state, staged-input and exogenous-input
lengths. The factor $L_F+1$ covers any explicitly stored table literal
written by one local step. A table scan and encoded update cost at most a
fixed polynomial in these lengths and $m$ per step. Summing over the $m$
steps and adding decoding, syntactic checks, interface handling and output
copying gives one polynomial bound in the stated quantities.
The interpreter itself is fixed.  Induction over stages
then gives equality of the emitted successor and retained checkpoint with the
represented admissible history.
\end{proof}

For the producer-relative evaluation of \cref{app:producer-relative}, let $g_\Delta$ consist of $g_{\rm int}$ and a fixed wrapper. The producer state of a stage is then the checkpoint $c_t$, including the current code $F_t$. The wrapper copies $c_t$ from the actual starting state and emits $\code{c_t,c_{t+1},\sigma_t}$, where $\sigma_t$ is any procedure data in the local output. This copying takes work polynomial in the stated lengths, so the stage work of $g_\Delta$ is bounded by one fixed polynomial in the same quantities.

\end{document}